\documentclass[lettersize,journal]{IEEEtran}
\usepackage{amsmath,amsfonts}
\usepackage{algorithmic}
\usepackage{algorithm}
\usepackage{array}
\usepackage{subcaption}
\usepackage{textcomp}
\usepackage{stfloats}
\usepackage{url}
\usepackage{verbatim}
\usepackage{graphicx}
\usepackage{cite}
\usepackage{mathtools}
\usepackage{amsthm}
\usepackage{accents}
\usepackage{enumitem}
\usepackage{tcolorbox}
\usepackage[algo2e,algoruled,boxed,lined,norelsize]{algorithm2e} 
\usepackage{tablefootnote}
\usepackage{booktabs}
\usepackage{threeparttable}
\graphicspath{{figures/}}

\newtheorem{remark}{Remark}
\newtheorem{definition}{Definition}
\newtheorem{lemma}{Lemma}
\newtheorem{theorem}{Theorem}

\DeclarePairedDelimiter{\norm}{\lVert}{\rVert}

\newcommand{\ubar}[1]{\underaccent{\bar}{#1}}

\newcommand*{\tran}{^{\mkern-1.5mu\mathsf{T}}}

\DeclareMathOperator{\E}{\mathbb{E}}
\DeclareMathOperator{\vol}{\text{\textbf{Vol}}}

\SetKwInOut{Initial}{Initial}
\SetKwInOut{Parameter}{Parameters}
\SetKw{Continue}{continue}
\allowdisplaybreaks

\begin{document}

\title{Learning from Human Directional Corrections}

\author{
	Wanxin Jin,
	Todd D. Murphey,
	Zehui Lu,
	Shaoshuai Mou
\thanks{
	Wanxin Jin is with the General Robotics, Automation, Sensing and Perception (GRASP) Laboratory, University of Pennsylvania, PA 19104, USA.   Wanxin Jin is the corresponding author. Email: wanxinjin@gmail.com.
	
	Todd D. Murphey is with  the Department of Mechanical Engineering, Northwestern University, Evanston, IL 60208,	USA. This material is partially based upon work supported by the National Science Foundation under award 1837515.  Email: t-murphey@northwestern.edu.
	
	Zehui Lu and  Shaoshuai Mou are with the School of Aeronautics and Astronautics, Purdue University, West Lafayette, IN 47906, USA. Dr. Mou’s research is supported in part by grants from the Research in Applications for Learning Machines (REALM) Consortium of Northrop Grumman Corporation, Rolls-Royce Corporation, and the NASA University Leadership Initiative (ULI). Email: zehuilu789@gmail.com, mous@purdue.edu.

This work involved human subjects or animals in its research. Approval of all ethical and experimental procedures and protocols was granted by Purdue’s Institutional Review Board under Application No. IRB-2021-1539, and performed in the line with Title 45 of the Code of Federal Regulations, Part 46 (45CFR 46).
}
}

\markboth{\normalsize
	\textit{ T\MakeLowercase{his is a preprint}. T\MakeLowercase{he published version can be accessed at} IEEE T\MakeLowercase{ransactions on} R\MakeLowercase{obotics}.
	}
}%
{Shell \MakeLowercase{\textit{et al.}}: A Sample Article Using IEEEtran.cls for IEEE Journals}


\maketitle

\begin{abstract}
 This paper proposes a novel approach that enables a robot to learn an objective function incrementally from human directional corrections.   Existing methods learn from human magnitude corrections; since a human needs to carefully choose the magnitude of each correction, those methods can easily lead to over-corrections and learning inefficiency. The proposed method only requires human directional corrections ---  corrections that only indicate the direction of an input change without indicating its magnitude. We only assume that each correction,  regardless of its magnitude,  points in a direction that improves the robot's current motion relative to an unknown objective function.   The allowable corrections satisfying this assumption account for half of the input space, as opposed to the magnitude corrections which have to lie in a shrinking level set.  For each directional correction, the proposed method updates the estimate of the objective function based on a cutting plane method, which has a geometric interpretation.   We have established theoretical results to show the convergence of the learning process.    The proposed method has been tested in numerical examples,  a user study on two human-robot games, and a real-world quadrotor experiment.  The results confirm the convergence of the proposed method and further show that the method is significantly more effective (higher success rate), efficient/effortless (less human corrections needed), and potentially more accessible (fewer early wasted trials) than the state-of-the-art robot learning frameworks.
\end{abstract}


\section{Introduction}

\IEEEPARstart{F}{or} tasks where robots work in proximity to human users, a robot is required to not only guarantee the accomplishment of the task but also complete it in a way that a human user prefers. Different users may have different preferences about how the robot should perform the task.  Such customized requirements usually lead to a considerable workload of robot programming,   which requires human users to have high expertise and skills.

To circumvent the above limitations of robot programming, learning from demonstrations (LfD) empowers non-expert users to program robots by providing demonstrations.   In  LfD  \cite{ravichandar2020recent},   a  user first provides  a robot with   demonstrations in a \emph{one-time} manner,  then the robot  learns  a control policy or   objective function  \emph{offline} from  the    demonstrations. While achieving the notable success in various applications \cite{kuderer2015learning,englert2017inverse,jin2020learning},  the \emph{one-time}  and \emph{offline} nature of LfD could  lead to some  challenges.   For example, when the demonstration data is insufficient to infer an objective function due to the low informativeness  \cite{jin2018inverse}  or deviation
from optima \cite{jain2015learning},    the demonstrations have to be re-collected and the robot has to be re-trained. This is particularly inconvenient for robots with high degrees of freedom.

A  recent line of work  \cite{jain2015learning,bajcsy2017learning,zhang2019learning, losey2018including} addresses the above challenges of LfD by developing a   scheme that enables a robot to
learn an objective function from \emph{user's feedback or corrections}.  Fig. \ref{figure_intro} is an example of those learning schemes. At each iteration, a human user does not need to provide optimal demonstrations to the robot, but merely a correction which is an \emph{incremental improvement} to the robot's current motion.  The robot then leverages the correction to update its objective function. Compared to LfD, this incremental learning scheme reduces the workload of a  user, especially for those (such as novices) who cannot provide the optimal demonstrations in a one-time manner \cite{jain2013learning}. Despite its promise,   the state-of-the-art methods  \cite{jain2015learning,bajcsy2017learning,zhang2019learning, losey2018including}  still face some challenges as below.

\begin{figure}[t]
	\centering
	\begin{subfigure}[b]{1\linewidth}
		\includegraphics[width=\linewidth]{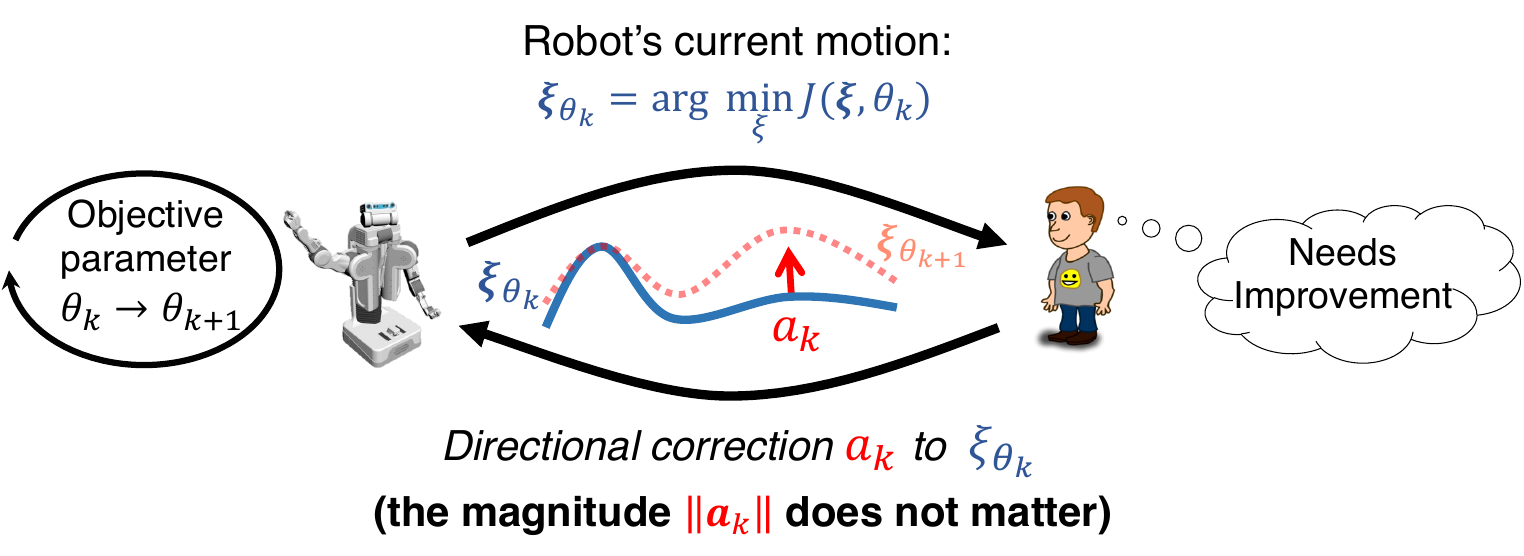}
	\end{subfigure}
	\caption{An illustration of learning from directional corrections proposed in this paper. The robot's current motion $\boldsymbol{\xi}_{\boldsymbol{\theta}_k}$ optimizes an   objective function  parameterized by 	$\boldsymbol{\theta}_k$. A human user applies a directional correction ${\boldsymbol{a}_k}$ to $\boldsymbol{\xi}_{\boldsymbol{\theta}_k}$. Note  that the magnitude $\norm{\boldsymbol{a}_k}$ of this correction does not matter. After receiving $\boldsymbol{a}_k$, the robot updates its objective function parameter to $\boldsymbol{\theta}_{k+1}$. This motion-correction-update process repeats until the convergence of the objective function parameter.
	}
	\label{figure_intro}
	\vspace{-5pt}
\end{figure}

First, providing a valid correction---the one that improves the robot motion---can be nontrivial. To obtain a valid correction,  the human user must carefully choose the \emph{magnitude} of a correction. As a robot gets closer to the `desired motion' (i.e., the motion in the human's perspective),  the allowable range of the magnitude of valid corrections gets smaller (see Section \ref{section_problem} for the detailed explanation).  Thus, the human can easily over-correct a robot near the desired motion, i.e., giving too large corrections that adversely drive the robot away from the desired motion.
Such difficulty makes robot learning inefficient (we will experimentally show this in Sections \ref{experiment_compare} and \ref{section_games}).

Second, most of  existing methods \cite{bajcsy2017learning,zhang2019learning, losey2018including} lack   \emph{theoretical guarantee} for the learning performance. While \cite{jain2015learning} shows that learning from human magnitude corrections can attain a bounded regret \cite{shivaswamy2012online}, this regret bound still cannot explicitly tell whether or not the learned objective function converges to the true one induced by all corrections.

To address the above challenges, this paper proposes a new learning method that enables a robot to learn an objective function from human \emph{directional} corrections with \emph{theoretical convergence guarantee}. We highlight the following new features of the proposed method,  as shown by Fig. \ref{figure_intro}.

(I) The proposed learning scheme only requires the user's \emph{directional corrections}. A directional correction is a correction that only concerns directional information and does not necessarily need to be magnitude-specific.  For instance, in teaching a mobile robot, a directional correction can be simply as `go left' or `go right' without dictating how far the robot should go. Furthermore, unlike existing work \cite{jain2015learning,bajcsy2017learning,zhang2019learning, losey2018including}, the proposed approach can \emph{directly} handle the sparse directional corrections without a pre-processing step of creating a human `intended' trajectory, which could introduce artifacts and unfavorably affect the learning performance.

(II) We emphasize the theoretical foundations of the proposed method. First, the learning process is based on  cutting-plane methods and thus has an intuitive geometric interpretation; and second,  we  establish the theoretical guarantees to show the convergence of our method towards finding the true objective function induced by human directional corrections. 

We have conducted extensive experiments to test the proposed method, including various numerical examples, a user study on two human-robot games, and a real-world quadrotor experiment. The results show that the proposed method is significantly more effective (higher success rate),  efficient/effortless (less human corrections needed), and potentially more accessible (fewer early wasted trials) than the state-of-the-art robot learning frameworks. 

\subsection{Related Work}

\subsubsection{One-time Learning from Demonstrations}
To learn an objective function from demonstrations, routine methods are  inverse optimal control \cite{moylan1973nonlinear,puydupin2012convex,jin2019inverse} or inverse reinforcement learning  \cite{ng2000algorithms,ziebart2008maximum,ratliff2006maximum}, where an objective function is learned from optimal demonstrations and  subsequently used  for control or planning. Successful LfD applications include autonomous driving \cite{kuderer2015learning}, robot manipulation \cite{englert2017inverse},  and motion planning \cite{jin2020learning}.
Despite the notable success \cite{jin2020pontryagin,jin2020inverse,moro2018learning},   LfD could be inconvenient in practice. First,  demonstrations in LfD are usually collected in a \emph{one-time} manner and the robot learning is    \emph{offline}.   In the case where demonstration data is insufficient to recover an objective function, such as low data informativeness as discussed in \cite{jin2018inverse}, or significantly deviating from the optima, the data has to be re-collected and the training has to be re-run.  Second, existing LfD  \cite{puydupin2012convex,jin2019inverse,ng2000algorithms,ziebart2008maximum,ratliff2006maximum} normally requires the optimality of demonstrations, which are challenging to collect for robots with high degree-of-freedoms.  For example,   when providing demonstrations for a humanoid robot, a  user has to account for the robot's motion of all degrees of freedom in a  spatially and temporally consistent manner
\cite{jain2015learning}.

\subsubsection{Incrementally Learning from  Corrections/feedback\label{key}}
Compared to one-time LfD,  learning from corrections or feedback enables a user to \emph{incrementally} improve robot motion,  making it more suitable for non-expert users who cannot provide optimal demonstrations in a one-time fashion \cite{jain2013learning}. 
The key assumption is that robot motion after correction achieves a higher reward (or a lower cost) than before correction. Under this assumption, \cite{jain2015learning} develops a co-active learning method to update the robot objective function using user feedback.  The user feedback includes a human either selecting the top-ranked robot trajectory among all candidates or physically demonstrating a preferred trajectory to the robot. By defining a  regret, 
which quantifies the average misalignment between the value of robot motion and that of the desired motion under the true objective function, \cite{jain2015learning} shows the convergence of the regret. But since regret only considers values of an objective function, it cannot directly tell and guarantee that the learned objective function itself is converging towards the true one. As we will show in Section \ref{experiment_compare}, the objective function can converge to a local solution instead of a true one.

Recently, \cite{bajcsy2017learning,zhang2019learning, losey2018including} handle learning from corrections through the perspective of the partially observable Markov decision process (POMDP). Here,  human corrections are viewed as observations about the unknown objective function parameter. By approximating the observation model and applying the maximum a posteriori estimation, they obtain a learning formula similar to the co-active learning  \cite{jain2015learning}. 
Along with this perspective, some variants have been recently developed to particularly account for the uncertainties of human corrections. For example, \cite{bobu2018learning} simultaneously estimates a rationality coefficient to characterize the rationality confidence of human corrections. \cite{losey2018including} fits the inverse reinforcement learning into a Kalman filter framework to quantitatively capture the uncertainty of the estimation.

Both the above co-active learning and POMDP-based learning require a human to carefully choose the \emph{magnitude} correction that improves robot motion. As we will detail in Section \ref{section_problem}, choosing a valid magnitude correction is difficult, especially when a     robot gets closer to the desired motion,  because the allowable magnitude of a correction has to lie in a shrinking level set. Near the desired motion, a human could easily over-correct the robot, i.e., applying too large corrections that adversely drive the robot away from the desired motion,  making the algorithm diverge or even fail. We will experimentally show this difficulty in Section \ref{experiment_compare}.
Moreover, due to the sparsity of human corrections,  all the above methods require a dedicated step to obtain a  \emph{human `intended' trajectory} for each correction. In the co-active learning, a robot needs to switch to a screening mode to obtain a human preferred trajectory, and in the POMDP-based method, a human intended trajectory is created by trajectory deformation \cite{dragan2015movement}.  These pre-processing steps may not only introduce more hyperparameters but also add some artifacts that could undermine the learning performance  \cite{zhang2019learning}

Very recently, \cite{losey2019learning} directly learns a desired trajectory from human physical interactions, then a robot tracks the learned trajectory using  linear–quadratic regulator (LQR) controllers. They minimize a trajectory distance loss, and the learning update is based on the assumption that the direction of a human correction is exactly aligned with the deepest gradient descent of the distance loss. While that work is related to our work in the sense that they both use the direction of a human correction, our work, however, focuses on learning an objective function instead of learning the desired trajectory directly. Importantly, we do not restrict the direction of a correction to be exactly aligned with the deepest gradient descent, but any correction as long as it has a direction that decreases the cost of robot motion. The above two distinctions lead to a new method of our work, which has a strong theoretical guarantee for learning convergence.

Finally, we want to mention that most of the existing methods \cite{bajcsy2017learning,zhang2019learning, losey2018including,losey2019learning}  lack theoretical guarantee about the learning performance. The only work that attempts to do so is \cite{jain2015learning}, where a  regret bound is shown. As discussed previously, such regret is an indicator of the misalignment of the values of the objective function, and cannot directly tell or guarantee the convergence of the objective function itself. In this paper, we will directly characterize the convergence of the learned objective function towards the true function that is induced by human corrections.

\subsection{Contributions}

To give readers a high-level picture of the proposed method, we show an overview in Fig. \ref{figure_algorithm}. We claim the following contributions of the proposed method.

\begin{figure}[h]
	\centering
	\begin{subfigure}[b]{1\linewidth}
		\includegraphics[width=\linewidth]{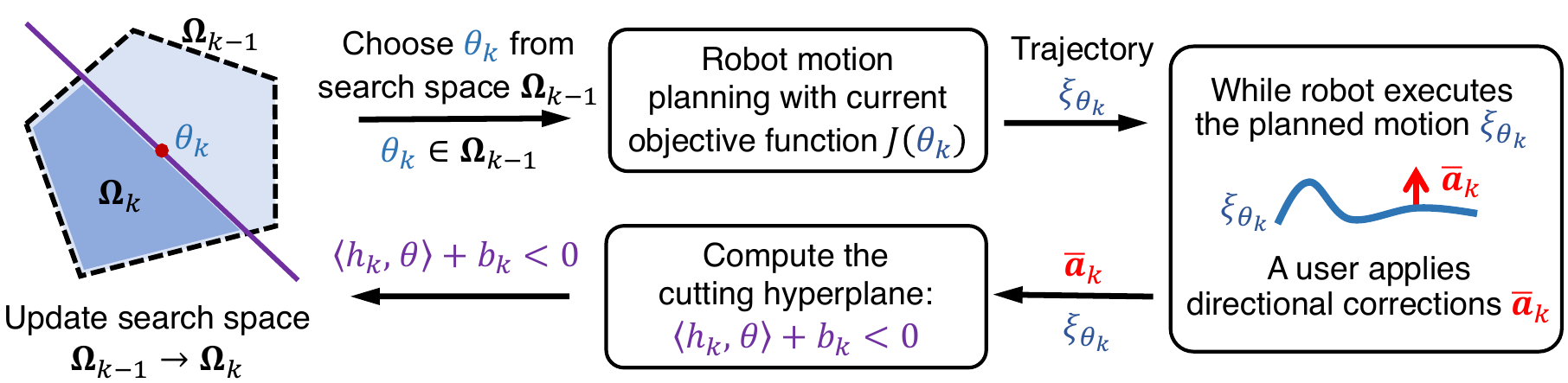}
	\end{subfigure}
	\caption{An  overview of the proposed algorithm.
	}
	\label{figure_algorithm}
\end{figure}

\begin{itemize}
	\item[(I)] The proposed method learns an objective function from human \emph{directional corrections}. It only requires that a correction, regardless of its magnitude, points in a direction of improving robot motion. As we will show in  Sections \ref{section_problem} and \ref{experiment_compare}, the allowable corrections that satisfy this requirement always account for half of the input space, making it more flexible for users to choose corrections from. Also, unlike existing methods, the proposed method directly learns from sparse directional corrections without a \emph{pre-processing} step of creating an `intended' trajectory.
	\item[(II)] The proposed learning algorithm has strong theoretical foundations. First, the proposed learning method is based on the cutting plane technique, which has straightforward geometric interpretations, as shown in Fig. \ref{figure_algorithm}. Second, we have established the theoretical results to show the convergence of the method towards finding the true objective function induced by all corrections. 
\end{itemize}

Numerical examples, a user study on two human-robot games, and a real-world quadrotor experiment confirm the efficacy and convergence of the proposed algorithm. The user study and real-world experiment further show that the proposed method is significantly more effective (higher success rate), efficient/effortless (fewer corrections needed), and potentially more accessible (fewer early wasted trials)  than the state-of-the-art methods.

In what follows,  Section II describes the problem formulation. Section III outlines the main algorithm and presents its geometric interpretation. Section IV provides the theoretical results of the algorithm and its detailed implementation. Various numerical examples and comparisons are given in Section V. Section VI presents a user study and Section VII provides a real-world experiment.  Conclusions are drawn in Section VIII. The appendix includes the proof, discussion, and possible extension of this work.

\section{Problem Formulation}\label{section_problem}

Consider a robot with the following dynamics and the initial condition:
\begin{equation}\label{equ_dynamics}
\boldsymbol{x}_{t+1}=\boldsymbol{f}(\boldsymbol{x}_{t},\boldsymbol{u}_{t}), \quad \text{with} \quad\boldsymbol{x}_0,
\end{equation}
where $\boldsymbol{x}_t\in\mathbb{R}^{n}$ is the robot state, $\boldsymbol{u}_t\in\mathbb{R}^{m}$ is the  control input,   $\boldsymbol{f}:\mathbb{R}^n\times\mathbb{R}^m\mapsto\mathbb{R}^n$ is  differentiable, and $t=1,2,\cdots$ is the time step. As commonly used in objective learning methods \cite{ng2000algorithms,ziebart2008maximum,ratliff2006maximum,jain2013learning,jain2015learning,bajcsy2017learning,zhang2019learning, losey2018including,jin2018inverse,jin2019inverse},  suppose that the robot  cost function has the following parameterized form:
\begin{equation}\label{equ_objective}
J(\boldsymbol{u}_{0:T},\boldsymbol{\theta})=\sum\nolimits_{t=0}^{T}\boldsymbol{\theta}\tran \boldsymbol{\phi}(\boldsymbol{x}_{t},\boldsymbol{u}_{t})+h(\boldsymbol{x}_{T+1}),
\end{equation} 
where $\boldsymbol{\phi}:\mathbb{R}^n\times \mathbb{R}^m\mapsto\mathbb{R}^{r}$  is a vector of the  \emph{pre-defined} features (basis functions);  $\boldsymbol{\theta}\in\mathbb{R}^{r}$ is a  vector of  weights, which are  \emph{tunable};  and  $h(\boldsymbol{x}_{T+1})$ is the  final cost  on the robot final  state $\boldsymbol{x}_{T+1}$, such as penalizing the distance between  $\boldsymbol{x}_{T+1}$ and a given goal state $\boldsymbol{x}^{\text{goal}}$, ${h}(\boldsymbol{x}_{T+1})=\norm{\boldsymbol{x}^{\text{goal}}-\boldsymbol{x}_{T+1}}^2$. Note that depending on  specific applications, the final cost term ${h}(\boldsymbol{x}_{T+1})$ can  be either absent, i.e., ${h}(\boldsymbol{x}_{T+1})=0$, or  also parameterized in a weighted-feature form with the  weights to be learned jointly with $\boldsymbol{\theta}$. The method developed in this paper can sufficiently handle either cases with little modifications.
For a fixed choice of $\boldsymbol{\theta}$, the  robot plans a sequence of  inputs $\boldsymbol{u}_{0:T}$ over  time horizon $T$ by (locally) optimizing the cost function (\ref{equ_objective})  subject to  (\ref{equ_dynamics}),  producing a  trajectory
\begin{equation}\label{equ_trajectory}
\boldsymbol{\xi}_{\boldsymbol{\theta}}=\left\{\boldsymbol{x}_{0:T\text{+}1}^{\boldsymbol{\theta}}, \boldsymbol{u}_{0:T}^{\boldsymbol{\theta}}\right\}.
\end{equation}
In the following text,  we  occasionally write the cost function (\ref{equ_objective}) as   $J(\boldsymbol{\theta})$ for  ease of  readability.

For a specific task,
suppose that a human user expects  the robot to minimize an \emph{implicit} cost function   $J(\boldsymbol{\theta}^*)$ in the same form of  (\ref{equ_objective}) but with  \emph{unknown}  $\boldsymbol{\theta}^*$. We  call  $\boldsymbol{\theta}^*$  the \emph{true  weight vector}. In general,  a  human user
may neither explicitly write down the value of $\boldsymbol{\theta}^*$ nor demonstrate the corresponding \emph{desired trajectory} $\boldsymbol{\xi}_{\boldsymbol{\theta}^*}$,  but
the human user  can tell  whether the robot motion trajectory is \emph{desired} or not. A robot trajectory is desired if it minimizes the implicit  $J(\boldsymbol{\theta}^*)$; otherwise, it is not desired.   In order for the robot to  learn  $J(\boldsymbol{\theta}^*)$ (and thus generate the desired trajectory $\boldsymbol{\xi}_{\boldsymbol{\theta}^*}$),  the human user is only able to  make  corrections  to the robot  motion. After receiving each human correction,   the robot updates its  guess  $\boldsymbol{\theta}$ towards the true $\boldsymbol{\theta}^*$. This procedure has been shown in Fig. \ref{figure_intro}.

Specifically, the process of a robot learning from human directional  corrections is iterative, as shown in Fig. \ref{figure_algorithm}. Each  iteration basically includes three steps: robot planning (\& execution), human correction, and robot update. Let $k=1,2, 3,\cdots$, be the iteration index and let  $\boldsymbol{\theta}_k$ denote the robot's guess of the weight vector at iteration $k$.  At  $k=1$, the robot is initialized with an arbitrary guess $\boldsymbol{\theta}_1$.  At iteration $k$, the robot first performs  motion \emph{planning}, i.e. solving  $\boldsymbol{\xi}_{\boldsymbol{\theta}_k}$  by
minimizing
the cost function $J(\boldsymbol{\theta}_k)$ in (\ref{equ_objective}) subject to  its dynamics  in  (\ref{equ_dynamics}).  During the robot executing  $\boldsymbol{\xi}_{\boldsymbol{\theta}_k}$, the human user observes $\boldsymbol{\xi}_{\boldsymbol{\theta}_k}$ and applies a \emph{correction}, denoted by  $\boldsymbol{a}_{t_k}\in\mathbb{R}^m$, to the robot in its input space. Here, ${t_k}\in\{0,1,\cdots,T\}$ is called \emph{correction time},  indicating at which time step  the correction $\boldsymbol{a}_{t_k}$ is made.  After  receiving  $\boldsymbol{a}_{t_k}$, the robot \emph{updates} its  guess $\boldsymbol{\theta}_k$  to $\boldsymbol{\theta}_{k+1}$ based on an update rule  developed later.

One distinguishing feature of our method is that we assume that each human  correction $\boldsymbol{a}_{t_k}$  satisfies
\begin{equation}\label{equ_assumption}
\left\langle -\nabla J(\boldsymbol{u}_{0:T}^{\boldsymbol{\theta}_k},\boldsymbol{\theta}^*),\,\,\boldsymbol{{a}}_k \right\rangle>0, \quad  k=1,2,3,\cdots.
\end{equation}
Here,
\begin{equation}\label{equ_aug_correction_vec}
\boldsymbol{{a}}_{k}= \begin{bmatrix}
\boldsymbol{0}\tran & \cdots ,& \boldsymbol{a}_{t_k}\tran, & \cdots,\boldsymbol{0}\tran
\end{bmatrix}\tran\in\mathbb{R}^{m{(T+1)}},
\end{equation} 
with    $\boldsymbol{a}_{t_k}$ filled at the  $t_k$th entry and   $\boldsymbol{0}\in\mathbb{R}^m$ elsewhere. Here, $\left\langle \cdot,\cdot\right\rangle$ is the dot product, and $-\nabla J(\boldsymbol{u}_{0:T}^{\boldsymbol{\theta}_k},\boldsymbol{\theta}^*)$ is the gradient-descent of $J(\boldsymbol{\theta}^*)$ with respect to   $\boldsymbol{u}_{0:T}$ evaluated at the robot current  trajectory $\boldsymbol{\xi}_{\boldsymbol{\theta}_k}=\{\boldsymbol{x}_{0:T\text{+}1}^{\boldsymbol{\theta}_k}, \boldsymbol{u}_{0:T}^{\boldsymbol{\theta}_k}\}$. Note that  (\ref{equ_assumption}) does not require a specific value to the magnitude of $\boldsymbol{a}_{t_k}$ but requires its direction roughly around the gradient-descent direction of $J(\boldsymbol{\theta}^*)$. It means that the correction $\boldsymbol{{a}}_{k}$ aims to guide the robot motion $\boldsymbol{\xi}_{\boldsymbol{\theta}_k}$ towards  a lower cost under $J(\boldsymbol{\theta}^*)$ unless the trajectory is desired. Thus, we call $\boldsymbol{a}_{t_k}$ a \emph{directional correction}.

\smallskip

In practice, one always needs to account for human's noisy or imperfect  corrections. Thus, we modify (\ref{equ_assumption}) as follows
\begin{subequations}\label{equ_assumption_modify}
	\begin{equation}\label{equ_assumption_modify.e}
	\begin{aligned}
	&\E_{\boldsymbol{a}_k}\left\langle -\nabla J(\boldsymbol{u}_{0:T}^{\boldsymbol{\theta}_k},\boldsymbol{\theta}^*),\,\,\boldsymbol{{a}}_k \right\rangle\\[2pt]
	=&\left\langle -\nabla J(\boldsymbol{u}_{0:T}^{\boldsymbol{\theta}_k},\boldsymbol{\theta}^*),\,\, \E_{\boldsymbol{a}_k} (\boldsymbol{{a}}_k)
	\right\rangle\\[2pt]
	=&\left\langle -\nabla J(\boldsymbol{u}_{0:T}^{\boldsymbol{\theta}_k},\boldsymbol{\theta}^*),\,\, \boldsymbol{\bar{a}}_k
	\right\rangle >0, \quad  k=1,2,...,
	\end{aligned}
	\end{equation}
	where
	\begin{equation}\label{equ_assumption_modify.abar}
	\boldsymbol{\bar{a}}_k=\E_{\boldsymbol{a}_k} (\boldsymbol{{a}}_k).
	\end{equation}
\end{subequations}
Here, $\E(\cdot)$ is the expectation  with respect to the human correction data $\boldsymbol{a}_k$.   (\ref{equ_assumption_modify}) says that the original assumption (\ref{equ_assumption}) can be met in expectation; in other words, the robot can receive and average human's \emph{multiple} directional corrections before one update of the cost function is made. There are two ways to implement multiple corrections. The most convenient way is that a human user applies multiple directional corrections at different time steps but all within the same execution of the robot motion $\boldsymbol{\xi}_{\boldsymbol{\theta}_k}$, i.e., the robot executes the current plan $\boldsymbol{\xi}_{\boldsymbol{\theta}_k}$ just once. The second way is that the multiple corrections are applied in the multiple repetitions of the execution of  $\boldsymbol{\xi}_{\boldsymbol{\theta}_k}$, i.e., the robot executes the same plan $\boldsymbol{\xi}_{\boldsymbol{\theta}_k}$ for multiple times.
Whichever implementation,  by taking the expectation of human's multiple directional corrections $\boldsymbol{\bar{a}}_k=\E_{\boldsymbol{a}_k} (\boldsymbol{{a}}_k)$, (\ref{equ_assumption_modify})  permits  noisy and imperfect human directional corrections. The similar strategies have also been  adopted in \cite{jain2015learning}.

The \textbf{problem of interest} in this paper is that given  human (averaged)  directional corrections $\boldsymbol{\bar{a}}_{k}$ satisfying the assumption   (\ref{equ_assumption_modify}), we aim to develop a rule to update  the robot weight vector  guess $\boldsymbol{\theta}_k$ such that  $\boldsymbol{\theta}_k$ converges to~$\boldsymbol{\theta}^*$ as $k=1,2,3,...$.

\begin{remark}
	We assume  that  human  corrections $\boldsymbol{a}_{t_k}\in\mathbb{R}^m$  are  in the robot input space,  which means that $\boldsymbol{a}_{t_k}$ can be directly added to  robot  input $\boldsymbol{u}_{t_k}$. This can be readily fulfilled in some applications such as autonomous vehicles, where a  user directly changes the steering angle of a vehicle. For other cases where the corrections are not readily in the robot input space, it could be met via some specific interfaces (or computation), which map correction signals to the robot input space. For example, when a human interacts with the end effector of a robot manipulator, one can convert the task-space contact force to the joint torques via Jacobian. Then, $\boldsymbol{a}_{t_k}$ is the mapped correction. The reason why we do not consider the corrections in the state space	is that 1)  corrections in input spaces may be easier to implement, as shown in our later experiments, and 2) corrections in state spaces can be infeasible for the under-actuated robot systems  \cite{tedrake2009underactuated}.
	
\end{remark}

\begin{figure}[h]
	\centering
	\begin{subfigure}[b]{0.45\linewidth}
		\includegraphics[width=\linewidth]{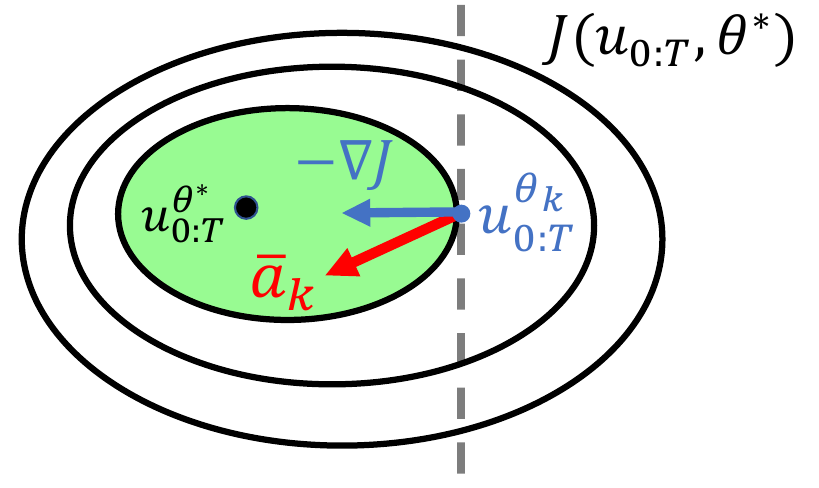}
		\caption{Allowable region (green)  of magnitude corrections.}
		\label{figure_correctionvs.1}
	\end{subfigure}
	\hspace{15pt}
	\begin{subfigure}[b]{.45\linewidth}
		\includegraphics[width=\linewidth]{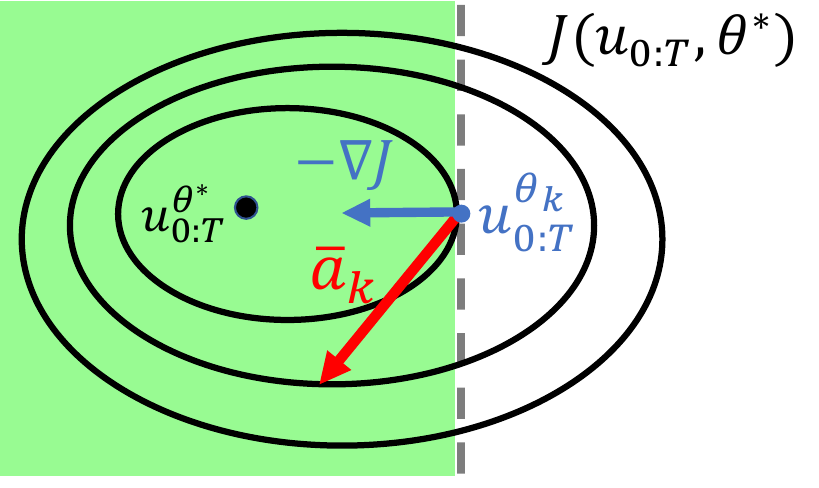}
		\caption{Allowable region  (green) of directional corrections.}
		\label{figure_correctionvs.2}
	\end{subfigure}
	\caption{Magnitude corrections v.s.  directional corrections. Contours (circles) and  the  desired  trajectory $\boldsymbol{u}_{0:T}^{\boldsymbol{\theta}^*}$ (black  dot) of the true cost function $J(\boldsymbol{\theta}^*)$ are plotted.
		(a) Green region (a sub-level set) indicates all allowable magnitude corrections $\boldsymbol{\bar{a}}_k$  that satisfy
		$J(\boldsymbol{u}_{0:T}^{\boldsymbol{{\theta}}_k}+\boldsymbol{\bar{a}}_k,\boldsymbol{\theta}^*){<} J(\boldsymbol{u}_{0:T}^{\boldsymbol{{\theta}}_k},\boldsymbol{\theta}^*)$. (b) Green  region (half of the input space) indicates all allowable directional corrections  $\boldsymbol{\bar{a}}_k$  that satisfy $\langle-\nabla J(\boldsymbol{u}_{0:T}^{\boldsymbol{\theta}_k},\boldsymbol{\theta}^*), \boldsymbol{\bar{a}}_k \rangle{>}0$.
	}
	\label{figure_correctionvs}
\end{figure}

\begin{remark}
		Assumption  (\ref{equ_assumption}) on  directional correction $\boldsymbol{a}_{{t}_k}$ is less restrictive than  ones in  \cite{bajcsy2017learning,zhang2019learning, losey2018including,jain2015learning} using   magnitude corrections,  which assume  the cost of the  corrected robot  trajectory $\boldsymbol{u}_{0:T}^{\boldsymbol{{\theta}}_k}+\boldsymbol{\bar{a}}_k$ is   lower  than that of  the robot original motion  $\boldsymbol{u}_{0:T}^{\boldsymbol{\theta}_k}$, i.e.,
	$\small J(\boldsymbol{u}_{0:T}^{\boldsymbol{{\theta}}_k}+\boldsymbol{\bar{a}}_k,\boldsymbol{\theta}^*){<} J(\boldsymbol{u}_{0:T}^{\boldsymbol{{\theta}}_k},\boldsymbol{\theta}^*)$. As shown in Fig. \ref{figure_correctionvs}, their assumptions usually lead to the restriction of  correction magnitude. Specifically, to satisfy  $\small J(\boldsymbol{u}_{0:T}^{\boldsymbol{{\theta}}_k}+\boldsymbol{\bar{a}}_k,\boldsymbol{\theta}^*){<} J(\boldsymbol{u}_{0:T}^{\boldsymbol{{\theta}}_k},\boldsymbol{\theta}^*)$, $\norm{\boldsymbol{\bar{a}}_k}$ has to be chosen from the  $\small J(\boldsymbol{u}_{0:T}^{\boldsymbol{{\theta}}_k},\boldsymbol{\theta}^*)$-sublevel set of $J(\boldsymbol{\theta}^*)$, as shown by the green region in Fig. \ref{figure_correctionvs.1}.  Furthermore, this  region will shrink as $\boldsymbol{u}_{0:T}^{\boldsymbol{{\theta}}_k}$ gets close to the  desired trajectory $\boldsymbol{u}_{0:T}^{\boldsymbol{\theta}^*}$ (black dot), thus making  $\norm{\boldsymbol{\bar{a}}_k}$ more difficult to choose.  In contrast,   the  directional corrections satisfying (\ref{equ_assumption}) always account for  half of the  input space, as shown by the green region in Fig. \ref{figure_correctionvs.2}.  A human  user can choose any correction as long as its direction lies in the half space of the   gradient descent of $J(\boldsymbol{\theta}^*)$.  Thus, (\ref{equ_assumption}) would be more likely to be satisfied  for non-expert  users. We will  experimentally show this advantage later in  Section \ref{experiment_compare} and  Section \ref{section_games}.
\end{remark}

\section{Main Algorithm Outline and its Geometric Interpretation}
In this section, we will outline the proposed main algorithm and present its geometric interpretation.

\subsection{Hyperplane for Each Directional  Correction}

Before developing the main algorithm, we first show that the condition  (\ref{equ_assumption_modify}) equals a  linear inequality imposed on the true  $\boldsymbol{\theta}^*$. This has been formally stated by the following lemma.

\begin{lemma}\label{lemma_correctionEquivalent}
	Suppose that  the current guess of the  weight vector  is $\boldsymbol{\theta}_k$, and the robot  trajectory  $\boldsymbol{\xi}_{\boldsymbol{\theta}_k}{=}\{\boldsymbol{x}_{0:T\text{+}1}^{\boldsymbol{\theta}_k}, \boldsymbol{u}_{0:T}^{\boldsymbol{\theta}_k}\}$ is a result of (locally) minimizing  $J(\boldsymbol{\theta}_k)$ in (\ref{equ_objective}) subject to  dynamics   (\ref{equ_dynamics}). For   $\boldsymbol{\xi}_{\boldsymbol{\theta}_k}$,  given a human (averaged) directional correction $\boldsymbol{\bar{a}}_{t_k}$
	satisfying (\ref{equ_assumption_modify}), one has 
	\begin{equation}\label{equ_assumption2}
	\left\langle\boldsymbol{h}_k,\boldsymbol{\theta}^* \right\rangle +b_k<0, \quad k=1,2,3\cdots,
	\end{equation}
	with 
	\begin{subequations}\label{equ_hk}
		\begin{align}
		\boldsymbol{h}_k&=\boldsymbol{H}_1\tran(\boldsymbol{x}_{0:{T\text{+}1}}^{\boldsymbol{\theta}_k}, \boldsymbol{u}_{0:T}^{\boldsymbol{\theta}_k})\boldsymbol{\bar{a}}_k\,\, \in\mathbb{R}^{r},\\[3pt]
		b_k&=\boldsymbol{\bar{a}}_k\tran\boldsymbol{H}_2(\boldsymbol{x}_{0:{T\text{+}1}}^{\boldsymbol{\theta}_k}, \boldsymbol{u}_{0:T}^{\boldsymbol{\theta}_k})
		\nabla h(\boldsymbol{x}_{T+1}^{\boldsymbol{\theta}_k}) 
		\,\, \in \mathbb{R}.
		\end{align}
	\end{subequations}
	Here, $\boldsymbol{\bar a}_k$ is defined in (\ref{equ_assumption_modify.abar}); $\nabla h(\boldsymbol{x}_{T+1}^{\boldsymbol{\theta}_k})$ is the gradient of the final cost $h(\boldsymbol{x}_{T+1})$ in (\ref{equ_objective}) evaluated at $\boldsymbol{x}_{T+1}^{\boldsymbol{\theta}_k}$; $\boldsymbol{H}_1(\boldsymbol{x}^{\boldsymbol{\theta}_k}_{0:T\text{+}1},\boldsymbol{u}^{\boldsymbol{\theta}_k}_{0:T})$ and $\boldsymbol{H}_2(\boldsymbol{x}^{\boldsymbol{\theta}_k}_{0:T\text{+}1},\boldsymbol{u}^{\boldsymbol{\theta}_k}_{0:T})$ are  computed as:
	\begin{subequations}\label{equ_matH}
		\begin{align}
		\boldsymbol{H}_1(\boldsymbol{x}^{\boldsymbol{\theta}_k}_{0:T\text{+}1},\boldsymbol{u}^{\boldsymbol{\theta}_k}_{0:T})&{=}\begin{bmatrix}
		\boldsymbol{F}_u\boldsymbol{F}_x^{{-}1}\boldsymbol{\Phi}_x{+}\boldsymbol{\Phi}_u\\
		\frac{\partial  \boldsymbol{\phi}\tran}{\partial \boldsymbol{u}^{\boldsymbol{\theta}_k}_{T}}\\
		\end{bmatrix}\in\mathbb{R}^{m(T\text{+}1)\times r},\\
		\boldsymbol{H}_2(\boldsymbol{x}^{\boldsymbol{\theta}_k}_{0:T\text{+}1},\boldsymbol{u}^{\boldsymbol{\theta}_k}_{0:T})&{=}\begin{bmatrix}
		\boldsymbol{F}_u\boldsymbol{F}_x^{-1}\boldsymbol{V}\\
		\frac{\partial  \boldsymbol{f}\tran}{\partial \boldsymbol{u}^{\boldsymbol{\theta}_k}_{T}}
		\end{bmatrix}\in\mathbb{R}^{m(T\text{+}1)\times n},
		\end{align} 
	\end{subequations}
	with 
	\begin{subequations}\label{equ_rm}
		\begin{align}
		\boldsymbol{F}_x&
		{=}\begin{bmatrix}
		{I} & \frac{-\partial\boldsymbol{f}\tran}{\partial\boldsymbol{x}^{\boldsymbol{\theta}_k}_1} &\cdots &\boldsymbol{0}&\boldsymbol{0}\\
		\boldsymbol{0} & I &\cdots &\boldsymbol{0} &\boldsymbol{0}\\
		\vdots &  \vdots & \ddots  & \vdots &\vdots \\
		\boldsymbol{0} &  \boldsymbol{0} & \cdots  & I &\frac{{-}\partial\boldsymbol{f}\tran}{\partial\boldsymbol{x}^{\boldsymbol{\theta}_k}_{T\text{-}1}} \\
		\boldsymbol{0} &  \boldsymbol{0} & \cdots  &  &I \\
		\end{bmatrix},\,\,
		\boldsymbol{\Phi}_x
		{=}\begin{bmatrix}
		\frac{\partial \boldsymbol{\phi}\tran}{\partial \boldsymbol{x}^{\boldsymbol{\theta}_k}_1} \\[4pt]
		\frac{\partial \boldsymbol{\phi}\tran}{\partial \boldsymbol{x}^{\boldsymbol{\theta}_k}_2} \\[4pt]
		\vdots \\[4pt]
		\frac{\partial \boldsymbol{\phi}\tran}{\partial \boldsymbol{x}^{\boldsymbol{\theta}_k}_T} \\[4pt]
		\end{bmatrix},
		\\
		\boldsymbol{F}_u&{=}\begin{bmatrix}
		\frac{\partial\boldsymbol{f}\tran}{\partial\boldsymbol{u}^{\boldsymbol{\theta}_k}_0} & \boldsymbol{0} &\cdots & \boldsymbol{0}\\
		\boldsymbol{0} & \frac{\partial\boldsymbol{f}\tran}{\partial\boldsymbol{u}^{\boldsymbol{\theta}_k}_1} &\cdots & \boldsymbol{0} \\
		\vdots &  \vdots &\ddots  & \vdots \\
		\boldsymbol{0} &  \boldsymbol{0} & \cdots  & \frac{\partial\boldsymbol{f}\tran}{\partial\boldsymbol{u}^{\boldsymbol{\theta}_k}_{T\text{-}1}}  \\
		\end{bmatrix},\,\,
		\boldsymbol{\Phi}_u{=}\small
		\begin{bmatrix}
		\frac{\partial \boldsymbol{\phi}\tran}{\partial \boldsymbol{u}^{\boldsymbol{\theta}_k}_0} \\[4pt]
		\frac{\partial \boldsymbol{\phi}\tran}{\partial \boldsymbol{u}^{\boldsymbol{\theta}_k}_1} \\[4pt]
		\vdots \\[4pt]
		\frac{\partial \boldsymbol{\phi}\tran}{\partial \boldsymbol{u}^{\boldsymbol{\theta}_k}_{T\text{-}1}} 
		\end{bmatrix},\\
		\boldsymbol{V}&{=}
		\begin{bmatrix}
		\boldsymbol{0} & \boldsymbol{0}& \cdots & \boldsymbol{0}&  \frac{\partial\boldsymbol{f}}{\partial \boldsymbol{x}^{\boldsymbol{\theta}_k}_T}
		\end{bmatrix}\tran.
		\end{align} 
	\end{subequations}
	In above, the dimensions are $\boldsymbol{F}_x\in\mathbb{R}^{nT\times nT}$, $\boldsymbol{F}_u\in\mathbb{R}^{mT\times nT}$,
	$\boldsymbol{\Phi}_x\in\mathbb{R}^{nT\times r}$,
	$\boldsymbol{\Phi}_u\in\mathbb{R}^{mT\times r}$,
	$\boldsymbol{V}\in\mathbb{R}^{nT\times n}$. $I$ is $n\times n$ identity. For a general differentiable function $\boldsymbol{g}(\boldsymbol{x})$ and a fixed value $\boldsymbol{x}^*$, we denote $\frac{\partial \boldsymbol{g}}{\partial \boldsymbol{x}^{^*}}$ as the Jacobian  of $\boldsymbol{g}(\boldsymbol{x})$ evaluated at $\boldsymbol{x}^*$.
\end{lemma} 

\begin{proof}
	We first derive the explicit form of  $\nabla J(\boldsymbol{u}^{\boldsymbol{\theta}_k}_{0:T}, \boldsymbol{\theta}^*)$ in (\ref{equ_assumption_modify}), and then show that (\ref{equ_assumption_modify}) can be re-written as (\ref{equ_assumption2}).

	Consider the robot current trajectory  $\boldsymbol{\xi}_{\boldsymbol{\theta}_k}{=}\{\boldsymbol{x}_{0:T\text{+}1}^{\boldsymbol{\theta}_k}, \boldsymbol{u}_{0:T}^{\boldsymbol{\theta}_k}\}$, which  satisfies the robot dynamics constraint  (\ref{equ_dynamics}). For any $t=0,1,..., T$, define the infinitesimal  perturbation pair $({\delta}\boldsymbol{\boldsymbol{x}}_{t}, {\delta}\boldsymbol{\boldsymbol{u}}_{t})$ at the state-input pair $(\boldsymbol{x}^{\boldsymbol{\theta}_k}_t,\boldsymbol{u}^{\boldsymbol{\theta}_k}_t)$. By linearizing the dynamics (\ref{equ_dynamics})  around $(\boldsymbol{x}^{\boldsymbol{\theta}_k}_t,\boldsymbol{u}^{\boldsymbol{\theta}_k}_t)$, we have
		\begin{equation}\label{equ_linear_dynamics}
	\delta\boldsymbol{x}_{t+1}=\small\frac{\partial \boldsymbol{f}}{\partial\boldsymbol{x}^{\boldsymbol{\theta}_k}_t} \delta\boldsymbol{x}_{t}+\small\frac{\partial \boldsymbol{f}}{\partial\boldsymbol{u}^{\boldsymbol{\theta}_k}_t} \delta\boldsymbol{u}_{t},
	\end{equation}
		where 
	$\frac{\partial \boldsymbol{f}}{\partial\boldsymbol{x}^{\boldsymbol{\theta}_k}_t}$ and $\frac{\partial \boldsymbol{f}}{\partial\boldsymbol{u}^{\boldsymbol{\theta}_k}_t}$ are the Jacobians of  $\boldsymbol{f}$ with respect to $\boldsymbol{x}_t$ and $\boldsymbol{u}_t$, respectively, evaluated at $(\boldsymbol{x}^{\boldsymbol{\theta}_k}_t,\boldsymbol{u}^{\boldsymbol{\theta}_k}_t)$. By stacking (\ref{equ_linear_dynamics}) for all $t=0,1,..., T$ and  also noting $\delta\boldsymbol{x}_0=\boldsymbol{0}$ (because $\boldsymbol{x}^{\boldsymbol{\theta}_k}_0=\boldsymbol{x}_0$ is given),  we have the following compact  form:
	\begin{equation}\label{equ_linear_dynamics_mat}
	-\boldsymbol{A}\tran\delta\boldsymbol{x}_{1:T+1}+\boldsymbol{B}\tran\delta\boldsymbol{u}_{0:T}=\boldsymbol{0}, 
	\end{equation}
	with $\delta\boldsymbol{x}_{1:T+1}{=}[\delta\boldsymbol{x}_{1}\tran,\cdots, \delta\boldsymbol{x}_{T+1}\tran]\tran$, $\delta\boldsymbol{u}_{1:T}{=}[\delta\boldsymbol{u}_{1}\tran,\cdots, \delta\boldsymbol{u}_{T}\tran]\tran$,
	\begin{align}
	\boldsymbol{A}=\begin{bmatrix}
	\boldsymbol{F}_x & -\boldsymbol{V}\\
	\boldsymbol{0}  & I
	\end{bmatrix}
	,\quad\text{and}\quad
	\boldsymbol{B}=\begin{bmatrix}
	\boldsymbol{F}_u & \boldsymbol{0}\\
	\boldsymbol{0}  & \small\frac{\partial\boldsymbol{f}\tran}{\partial\boldsymbol{u}^{\boldsymbol{\theta}_k}_{T}}
	\end{bmatrix},
	\end{align}
	with  $\boldsymbol{F}_x$ and $\boldsymbol{F}_u$ defined in (\ref{equ_rm}). Given $\delta\boldsymbol{x}_{1:T+1}$ and $\delta\boldsymbol{u}_{1:T}$, the increment of   $J(\boldsymbol{u}^{\boldsymbol{\theta}_k}_{0:T},\boldsymbol{\theta}^*)$, denoted as $\delta J(\boldsymbol{\theta}^*)$,  is
	\begin{equation}\label{equ_linear_objective}
	\delta J(\boldsymbol{\theta}^*)=\boldsymbol{C}\tran\delta\boldsymbol{x}_{1:T+1}+\boldsymbol{D}\tran\delta\boldsymbol{u}_{0:T},
	\end{equation}
	with 
	\begin{equation}
	\boldsymbol{C}=\begin{bmatrix}
	\boldsymbol{\Phi}_x\boldsymbol{\theta}^*\\
	\frac{\partial  {h}\tran}{\partial \boldsymbol{x}^{\boldsymbol{\theta}_k}_{T+1}}
	\end{bmatrix}\quad\text{and}\quad
	\boldsymbol{D}=\begin{bmatrix}
	\boldsymbol{\Phi}_u\boldsymbol{\theta}^*\\
	\frac{\partial  \boldsymbol{\phi}\tran}{\partial \boldsymbol{u}^{\boldsymbol{\theta}_k}_{T}}\boldsymbol{\theta}^*
	\end{bmatrix},
	\end{equation}
	with 
	$\boldsymbol{\Phi}_x$ and $\boldsymbol{\Phi}_u$  defined in (\ref{equ_rm}). Since $\boldsymbol{A}$ is always invertible (due to full rank), we solve for $\delta\boldsymbol{x}_{1:T+1}$ from (\ref{equ_linear_dynamics_mat}) and  submit it  to (\ref{equ_linear_objective}), yielding 
	\begin{align}
	\delta J(\boldsymbol{\theta}^*)&=\boldsymbol{C}\tran\delta\boldsymbol{x}_{1:T+1}+\boldsymbol{D}\tran\delta\boldsymbol{u}_{0:T},\nonumber\\
	&=\Big(\boldsymbol{C}\tran(\boldsymbol{A}^{-1})\tran\boldsymbol{B}\tran+\boldsymbol{D}\tran\Big)\delta\boldsymbol{u}_{0:T}.
	\end{align}
	Thus, we have 
	\begin{equation}\label{equ_gradient_appendix_1}
	\nabla J(\boldsymbol{u}^{\boldsymbol{\theta}_k}_{0:T}, \boldsymbol{\theta}^*)=\boldsymbol{B}\boldsymbol{A}^{-1}\boldsymbol{C}+\boldsymbol{D}.
	\end{equation}
	The above (\ref{equ_gradient_appendix_1}) can be further written as
	\begin{align}\label{equ_gradient_appendix_2}
	\nabla &J(\boldsymbol{u}^{\boldsymbol{\theta}_k}_{0:T}, \boldsymbol{\theta}^*)=\boldsymbol{B}\boldsymbol{A}^{-1}\boldsymbol{C}+\boldsymbol{D}\nonumber\\
	&=\begin{bmatrix}
	\boldsymbol{F}_u & \boldsymbol{0}\\
	\boldsymbol{0} & \small\frac{\partial\boldsymbol{f}\tran}{\partial\boldsymbol{u}^{\boldsymbol{\theta}_k}_{T}}
	\end{bmatrix}\begin{bmatrix}
	\boldsymbol{F}_x & -\boldsymbol{V}\\
	\boldsymbol{0}  & I
	\end{bmatrix}^{-1}\begin{bmatrix}
	\boldsymbol{\Phi}_x\boldsymbol{\theta}^*\\
	\frac{\partial  {h}\tran}{\partial \boldsymbol{x}^{\boldsymbol{\theta}_k}_{T\text{+}1}}
	\end{bmatrix}+\begin{bmatrix}
	\boldsymbol{\Phi}_u\boldsymbol{\theta}^*\\
	\frac{\partial  \boldsymbol{\phi}\tran}{\partial \boldsymbol{u}^{\boldsymbol{\theta}_k}_{T}}\boldsymbol{\theta}^*
	\end{bmatrix}\nonumber\\
	&=\begin{bmatrix}
	\boldsymbol{F}_u\boldsymbol{F}_x^{-1}\boldsymbol{\Phi}_x\boldsymbol{\theta}^*{+}\boldsymbol{\Phi}_u\boldsymbol{\theta}^*{+}\boldsymbol{F}_u\boldsymbol{F}_x^{-1}\boldsymbol{V}\frac{\partial h\tran}{\partial \boldsymbol{x}^{\boldsymbol{\theta}_k}_{T+1}}\\
	\frac{\partial  \boldsymbol{\phi}\tran}{\partial \boldsymbol{u}^{\boldsymbol{\theta}_k}_{T}}\boldsymbol{\theta}^*{+}\frac{\partial\boldsymbol{f}\tran}{\partial \boldsymbol{u}^{\boldsymbol{\theta}_k}_T}\frac{\partial h\tran}{\partial \boldsymbol{x}^{\boldsymbol{\theta}_k}_{T+1}}
	\end{bmatrix}, 
	\end{align}
	where we have used  Schur complement   to compute the inverse of the block matrix $\boldsymbol{A}$. By the definitions in  (\ref{equ_matH}), (\ref{equ_gradient_appendix_2}) becomes
	\begin{multline}\label{equ_gradient}
	{\nabla J(\boldsymbol{u}^{\boldsymbol{\theta}_k}_{0:T}, \boldsymbol{\theta}^*)}=\boldsymbol{H}_1(\boldsymbol{x}^{\boldsymbol{\theta}_k}_{0:T\text{+}1},\boldsymbol{u}^{\boldsymbol{\theta}_k}_{0:T})\boldsymbol{\theta}^*\\+\boldsymbol{H}_2(\boldsymbol{x}^{\boldsymbol{\theta}_k}_{0:T\text{+}1},\boldsymbol{u}^{\boldsymbol{\theta}_k}_{0:T})
	\nabla h( \boldsymbol{x}^{\boldsymbol{\theta}_k}_{T+1}).
	\end{multline}
	
	\medskip
	
	Substituting (\ref{equ_gradient}) into (\ref{equ_assumption_modify})  and also considering the definitions  in (\ref{equ_hk}), we obtain
	\begin{equation}
	\left\langle -\nabla J(\boldsymbol{u}_{0:T}^{\boldsymbol{\theta}_k},\boldsymbol{\theta}^*),\,\,{\boldsymbol{\bar{a}}_k} \right\rangle=-\left\langle\boldsymbol{h}_k,\boldsymbol{\theta}^* \right\rangle -b_k>0,
	\end{equation}
	which leads to (\ref{equ_assumption2}). This completes the proof. 
\end{proof}
We have the following remarks on the intuition of Lemma~\ref{lemma_correctionEquivalent}.

\begin{remark}
	The key result of   Lemma \ref{lemma_correctionEquivalent} is the linear inequality (\ref{equ_assumption2}), which is associated with a hyperplane $\left\langle\boldsymbol{h}_k,\boldsymbol{\theta} \right\rangle +b_k=0$. The   parameters  $\boldsymbol{h}_k$ and $b_k$ of this hyperplane, given in (\ref{equ_hk}), are both known and determined by      human directional correction $\boldsymbol{\bar{a}}_{k}$ and robot  current trajectory   $\boldsymbol{\xi}_{\boldsymbol{\theta}_k}$. In other words,  Lemma \ref{lemma_correctionEquivalent} states that given robot  motion $\boldsymbol{\xi}_{\boldsymbol{\theta}_k}$ and the human  directional correction  $\boldsymbol{\bar{a}}_{k}$, one can write  a hyperplane  $\left\langle\boldsymbol{h}_k,\boldsymbol{\theta} \right\rangle +b_k=0$ and a linear inequality  (\ref{equ_assumption2}) for the unknown  $\boldsymbol{\theta}^*$.
\end{remark}

\begin{remark}
	Obtaining  $\boldsymbol{h}_k$ and~$b_k$ requires computing the matrices	$\small\boldsymbol{H}_1(\boldsymbol{x}^{\boldsymbol{\theta}_k}_{0:T\text{+}1},\boldsymbol{u}^{\boldsymbol{\theta}_k}_{0:T})$ and $\small\boldsymbol{H}_2(\boldsymbol{x}^{\boldsymbol{\theta}_k}_{0:T\text{+}1},\boldsymbol{u}^{\boldsymbol{\theta}_k}_{0:T})$ in (\ref{equ_matH}). In our previous work  \cite{jin2018inverse}, $\small\boldsymbol{H}_1(\boldsymbol{x}^{\boldsymbol{\theta}_k}_{0:T\text{+}1},\boldsymbol{u}^{\boldsymbol{\theta}_k}_{0:T})$ and $\small\boldsymbol{H}_2(\boldsymbol{x}^{\boldsymbol{\theta}_k}_{0:T\text{+}1},\boldsymbol{u}^{\boldsymbol{\theta}_k}_{0:T})$  are called the `recovery matrix'. The Lemma 1 of \cite{jin2018inverse} has shown that the `recovery matrix'  can be iteratively computed by  integrating each  point  $(\boldsymbol{x}_t^{\boldsymbol{\theta}_k}, \boldsymbol{u}_t^{\boldsymbol{\theta}_k})$, $t=0,1,...,T$. This iterative property of the recovery matrix facilitates the computation of  $\boldsymbol{H}_1$ and $\boldsymbol{H}_2$ by avoiding  the inverse of $\boldsymbol{F}_x$ in (\ref{equ_matH}),  significantly reducing  the computational burden. Please refer to Lemma 1 of \cite{jin2018inverse} for the iteration formula. 
\end{remark}

\subsection{Outline of the Proposed  Algorithm}

With Lemma \ref{lemma_correctionEquivalent}, we are ready to develop the main algorithm to learn $\boldsymbol{\theta}^*$.  At each  iteration $k$, we maintain a \emph{weight search space}, denoted by  $\boldsymbol{\Omega}_k\subset\mathbb{R}^{r}$,   such that $\boldsymbol{\theta}^*\in\boldsymbol{\Omega}_k$ and $\boldsymbol{\theta}_k\in\boldsymbol{\Omega}_k$ for all $k=1,2,3,...$. This $\boldsymbol{\Omega}_k$ can be thought of as the possible location of $\boldsymbol{\theta}^*$, and $\boldsymbol{\theta}_k$ as
the current  guess to $\boldsymbol{\theta}^*$.
Rather than a direct rule to guide $\boldsymbol{\theta}_{k}$ towards $\boldsymbol{\theta}^*$, we will develop a rule to update $\boldsymbol{\Omega}_k$ to $\boldsymbol{\Omega}_{k+1}$ such that a useful  measure of the size of  $\boldsymbol{\Omega}_k$ will converge to 0. With this high-level idea, we outline the proposed main algorithm below.

\begin{tcolorbox}[title=\textbf{Main Algorithm (Outline)},left=0.5mm, right=0.7mm]
 \emph{Initialize the weight search space $\boldsymbol{\Omega}_0$ to be 
	\begin{equation}\label{equ_linftyball}
	\boldsymbol{\Omega}_0=\{ \boldsymbol{\theta}\in\mathbb{R}^r \,\,| \,\, 
	-\ubar{c}_i\leq  \boldsymbol{\theta}[i]\leq \bar{c}_i, \,\, i=1,2,...,r
	\},
	\end{equation} where  constants $\ubar{c}_i\geq 0$ and $\bar{c}_i\geq 0$   denote the lower bound and upper bounds of the $i$th entry $\boldsymbol{\theta}[i]$ in $\boldsymbol{\theta}$, respectively. Here, $\ubar{c}_i$ and $\bar{c}_i$ can be chosen large enough such that  $\boldsymbol{\theta}^*\in\boldsymbol{\Omega}_0$. At   iteration $k=1,2,...,$ the learning  proceeds  with the following three steps:} 
\smallskip
\begin{itemize}[leftmargin=35pt,font=\itshape]
	\setlength\itemsep{0.6em}
	\item[\textbf{Step 1:}] \emph{Choose a weight vector  guess $\boldsymbol{\theta}_k$ from the current weight search space $\boldsymbol{\Omega}_{k-1}$, i.e., $\boldsymbol{\theta}_k\in \boldsymbol{\Omega}_{k-1}$(we will discuss how to choose  $\boldsymbol{\theta}_k\in \boldsymbol{\Omega}_{k-1}$ in Section \ref{section_results}).}
	
	\item[\textbf{Step 2:}] \emph{The robot restarts and plans its motion trajectory $\boldsymbol{\xi}_{\boldsymbol{\theta}_k}$ by solving a trajectory optimization with   the cost function $J(\boldsymbol{\theta}_k)$ in (\ref{equ_objective}) and the dynamics in (\ref{equ_dynamics}). While the robot is  executing the plan  $\boldsymbol{\xi}_{\boldsymbol{\theta}_k}$, a human applies (averaged)  directional correction {$\boldsymbol{\bar{a}}_{k}$} in (\ref{equ_assumption_modify.abar}). Then, a hyperplane $\left\langle\boldsymbol{h}_k,\boldsymbol{\theta} \right\rangle+b_k=0$ is computed in (\ref{equ_assumption2})-(\ref{equ_hk}).}
	
	\item[\textbf{Step 3:}] \emph{Update the weight search space $\boldsymbol{\Omega}_{k-1}$ to $\boldsymbol{\Omega}_{k}$ via
		\begin{equation}\label{equ_updateOmega}
		\boldsymbol{\Omega}_{k}=\boldsymbol{\Omega}_{k-1}\cap\left\{\boldsymbol{\theta}\in\mathbb{R}^r\,|\, \left\langle\boldsymbol{h}_k,\boldsymbol{\theta} \right\rangle+b_k<0\right\}.
		\end{equation}}
\end{itemize}

\end{tcolorbox}

 We provide some remarks to the above outline of the proposed algorithm. For the initialization  (\ref{equ_linftyball}), we allow different entries of $\boldsymbol{\theta}$ to have their own lower and upper bounds, which may come from prior knowledge. Simply but not necessarily, one could also initialize 
\begin{equation}
\boldsymbol{\Omega}_0=\{\boldsymbol{\theta}\in\mathbb{R}^r \,\,| \,\, \norm{\boldsymbol{\theta}}_{\infty}\leq  R \},
\end{equation} where 
\begin{equation}\label{equ_radius_R}
R=\max \{ \ubar{c}_i, \ \bar{c}_i, \,\, i=1,\cdots,r \}.
\end{equation}
In Step 1, a weight vector guess $\boldsymbol{\theta}_k$ is chosen from $\boldsymbol{\Omega}_{k-1}$. We will (in (\ref{eq_prop2}) in Lemma \ref{lemma3})  show that the true  $\boldsymbol{\theta}^*$ \emph{always} lies in $\boldsymbol{\Omega}_k$ for any  $k=1,2,3,...$. Thus, we will expect $\boldsymbol{\theta}_k $ to be closer to $\boldsymbol{\theta}^*$ if  a size measure of $\boldsymbol{\Omega}_k$  gets smaller. In fact,  the weight search space $\boldsymbol{\Omega}_k$ is \emph{always non-increasing} as $\boldsymbol{\Omega}_k\subseteq\boldsymbol{\Omega}_{k-1}$ due to (\ref{equ_updateOmega}) in Step 3. In the next Section \ref{section_results}, we will  further focus on how to  pick $\boldsymbol{\theta}_k\in \boldsymbol{\Omega}_{k-1}$ such that the size  of $\boldsymbol{\Omega}_k$ is  \emph{strictly reduced} as $k$ increases. In Step 2, the robot trajectory $\boldsymbol{\xi}_{\boldsymbol{\theta}_k}$  is planned by solving a trajectory optimization with the cost function $J(\boldsymbol{\theta}_k)$ in (\ref{equ_objective}) and  the dynamics constraint   (\ref{equ_dynamics}). This can be done by any available trajectory optimization methods, e.g., \cite{li2004iterative} and existing solvers \cite{Andersson2019}. After a human applies a directional correction $\boldsymbol{\bar{a}}_{t_k}$ to the executed $\boldsymbol{\xi}_{\boldsymbol{\theta}_k}$, a hyperplane $\left\langle\boldsymbol{h}_k,\boldsymbol{\theta} \right\rangle+b_k=0$  can be computed via (\ref{equ_assumption2})-(\ref{equ_hk}), which will be used to   update the search space by (\ref{equ_updateOmega}) in Step 3. The detailed implementation of the above main algorithm  will be described in the next section.

Given the above proposed main algorithm, we next present the following important lemma.

\begin{lemma}\label{lemma3}
	Under the proposed main algorithm, one has
	\begin{equation}\label{eq_prop1}
	\left\langle\boldsymbol{h}_k, \boldsymbol{\theta}_k\right\rangle+b_k=0, \quad \forall \,\, k=1,2,3,...
	\end{equation}
	and 
	\begin{equation}\label{eq_prop2}
	\boldsymbol{\theta}^*\in \boldsymbol{\Omega}_k, \quad \forall\,\, k=1,2,3,...
	\end{equation}
	
\end{lemma}
\noindent
A proof of Lemma \ref{lemma3} is given in Appendix \ref{appendix_2}. Lemma \ref{lemma3} has the following   intuitive geometric explanations. First,  (\ref{eq_prop1}) says that the current guess $\boldsymbol{\theta}_k$ is always located on the current hyperplane $ \left\langle\boldsymbol{h}_k,\boldsymbol{\theta} \right\rangle+b_k=0$. Second, (\ref{eq_prop2}) says that although the proposed algorithm directly updates the weight search space  $\boldsymbol{\Omega}_k$ via (\ref{equ_updateOmega}), the true (but unknown) weight vector $\boldsymbol{\theta}^*$  is always contained in the current  $\boldsymbol{\Omega}_k$. Thus, intuitively, the smaller the search space $\boldsymbol{\Omega}_k$ is, the closer $  \boldsymbol{\theta}^* $ 
is to  $ \boldsymbol{\theta}_k $.

\subsection{Geometric Interpretation}\label{section_iterpretation}

In this part, we will provide an interpretation of the proposed main algorithm through a geometric perspective. For simplicity of illustrations, we assume $\boldsymbol{\theta}\in\mathbb{R}^2$ here. 

\begin{figure}[h]
	\hspace{-2pt}
	\begin{subfigure}[b]{0.45\linewidth}
		\includegraphics[width=\linewidth]{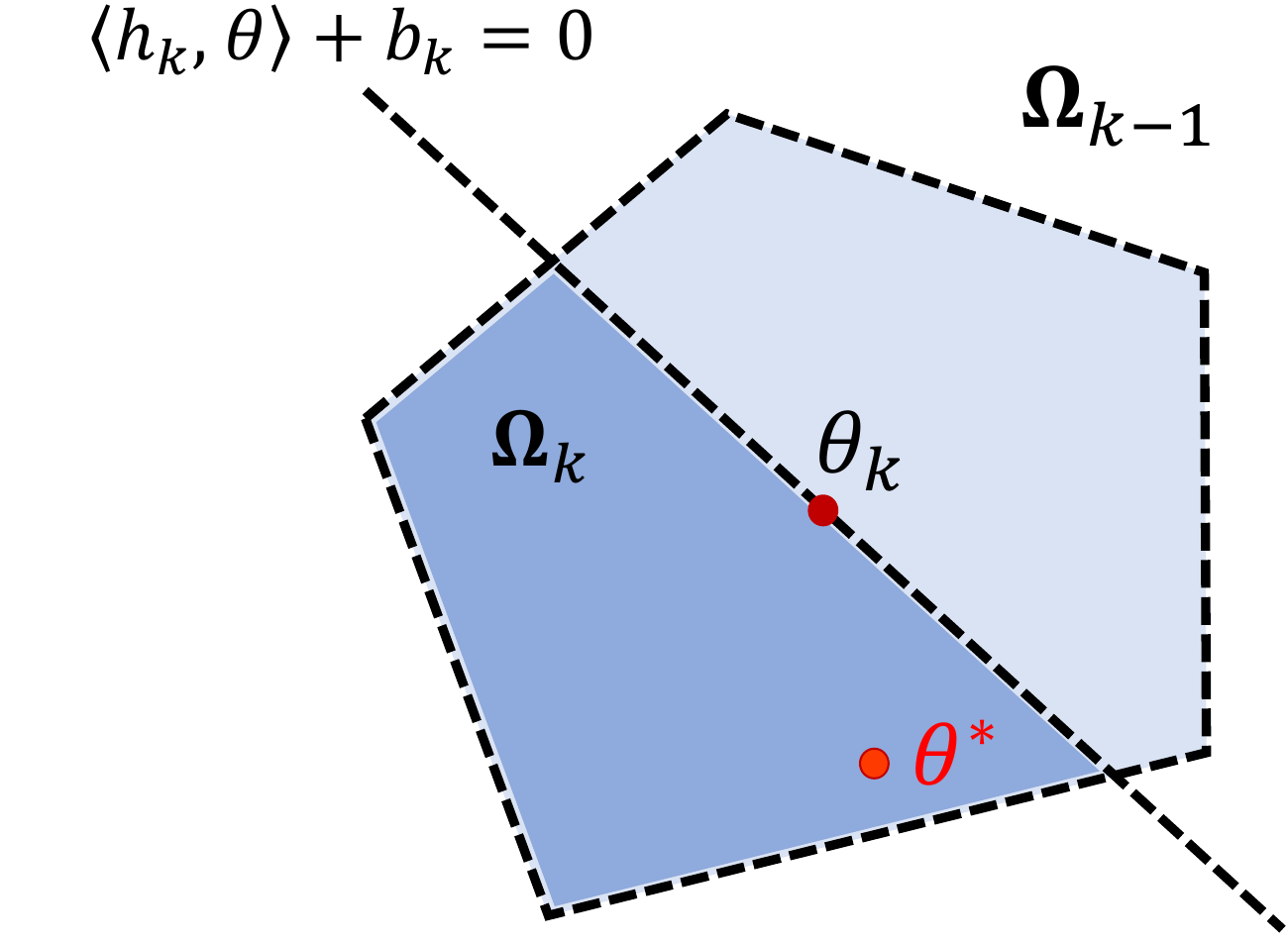}
		\caption{At $k$-th iteration}
		\label{fig_interpretation.1}
	\end{subfigure}
	\begin{subfigure}[b]{.54\linewidth}
		\includegraphics[width=\linewidth]{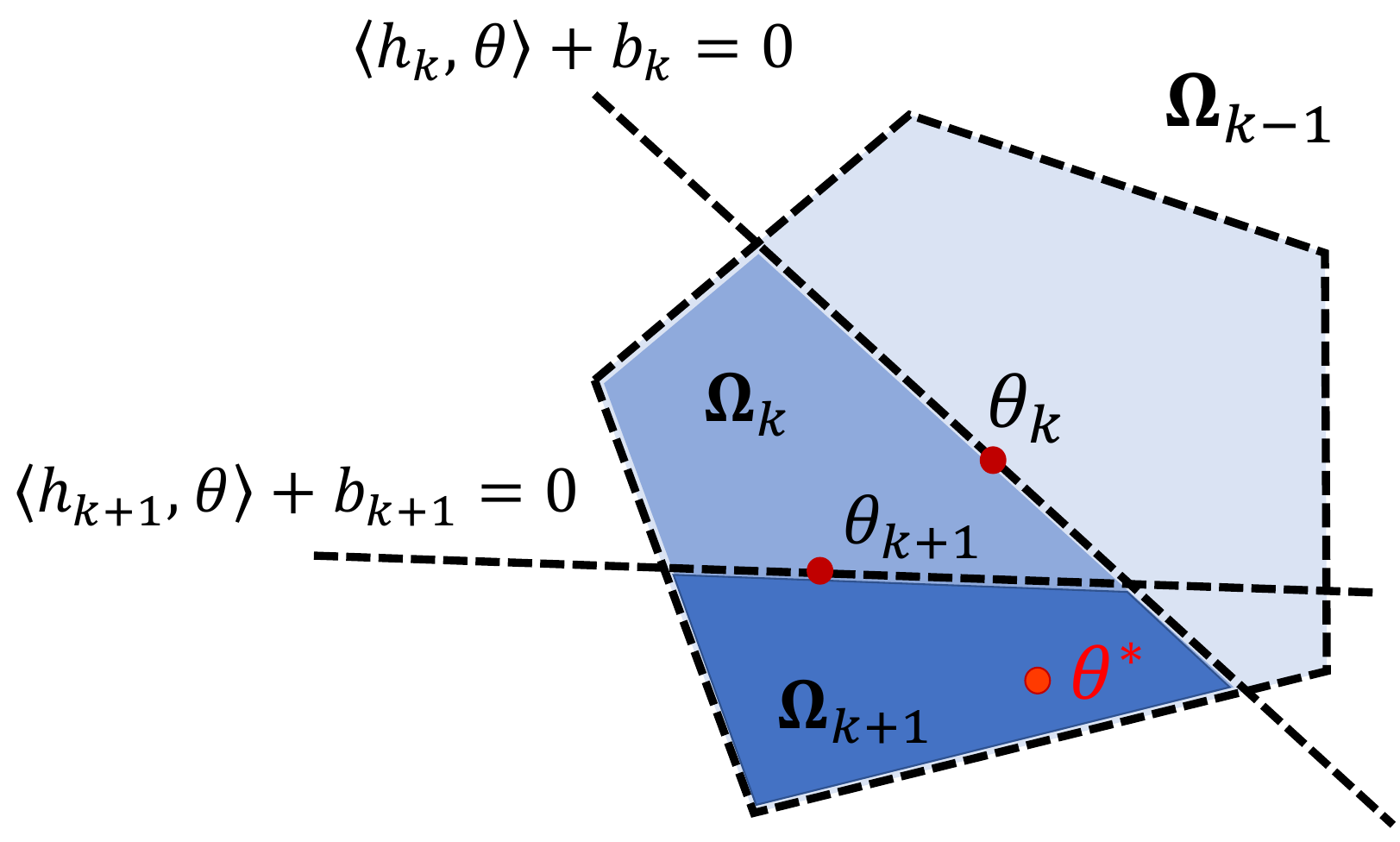}
		\caption{At $(k+1)$-th iteration}
		\label{fig_interpretation.2}
	\end{subfigure}
	\caption{Illustration of updating $\boldsymbol{\Omega}_{k}$.}
	\label{fig_interpretation}
\end{figure}

At the  $k$th iteration shown in Fig. \ref{fig_interpretation.1}, by Step 1 of the main algorithm, a  weight vector guess $\boldsymbol{\theta}_k$ (red dot)  is picked from the  current   search space $\boldsymbol{\Omega}_{k-1}$ (light blue region), i.e., $\boldsymbol{\theta}_k\in\boldsymbol{\Omega}_{k-1}$. By Step 2, we  obtain a hyperplane $\left\langle\boldsymbol{h}_k, \boldsymbol{\theta}\right\rangle+b_k=0$ (black dashed line), which cuts through the weight search space $\boldsymbol{\Omega}_{k-1}$  into two portions. By  (\ref{eq_prop1}) of Lemma \ref{lemma3},  we know that  $\boldsymbol{\theta}_k$ always lies on this hyperplane because  $\left\langle\boldsymbol{h}_k, \boldsymbol{\theta}_k\right\rangle+b_k=0$, as shown in Fig.  \ref{fig_interpretation.1}. 
By Step 3 of  the main algorithm, we  only keep one of the two cut portions,  which is the interaction  between $\boldsymbol{\Omega}_{k-1}$ and the half space $\left\langle\boldsymbol{h}_k, \boldsymbol{\theta}\right\rangle+b_k<0$, and the kept portion will be used as the new  search space  for the next  iteration, that is, $\boldsymbol{\Omega}_{k}=\boldsymbol{\Omega}_{k-1}\cap\left\{\boldsymbol{\theta}\in\mathbb{R}^r\,|\, \left\langle\boldsymbol{h}_k,\boldsymbol{\theta} \right\rangle+b_k<0 \right\}$, as shown by the  blue region in Fig. \ref{fig_interpretation.1}. The above procedure repeats  at the next iteration $k+1$ in  Fig. \ref{fig_interpretation.2}, and  finally produces a smaller  search space $\boldsymbol{\Omega}_{k+1}$, as shown by the darkest blue region in Fig. \ref{fig_interpretation.2}. From (\ref{equ_updateOmega}), one has $\boldsymbol{\Omega}_{0}\supseteq\cdots\boldsymbol{\Omega}_{k-1}\supseteq\boldsymbol{\Omega}_{k}\supseteq\boldsymbol{\Omega}_{k+1}\supseteq\cdots$. Moreover, by (\ref{eq_prop2}) in Lemma \ref{lemma3}, we note that the  true   $\boldsymbol{\theta}^*$ (red point) is always  contained in $\boldsymbol{\Omega}_{k}$ whenever $k$ is.

\begin{figure}[h]
	\hspace{-03pt}
	\begin{subfigure}[b]{0.53\linewidth}
		\includegraphics[width=\linewidth]{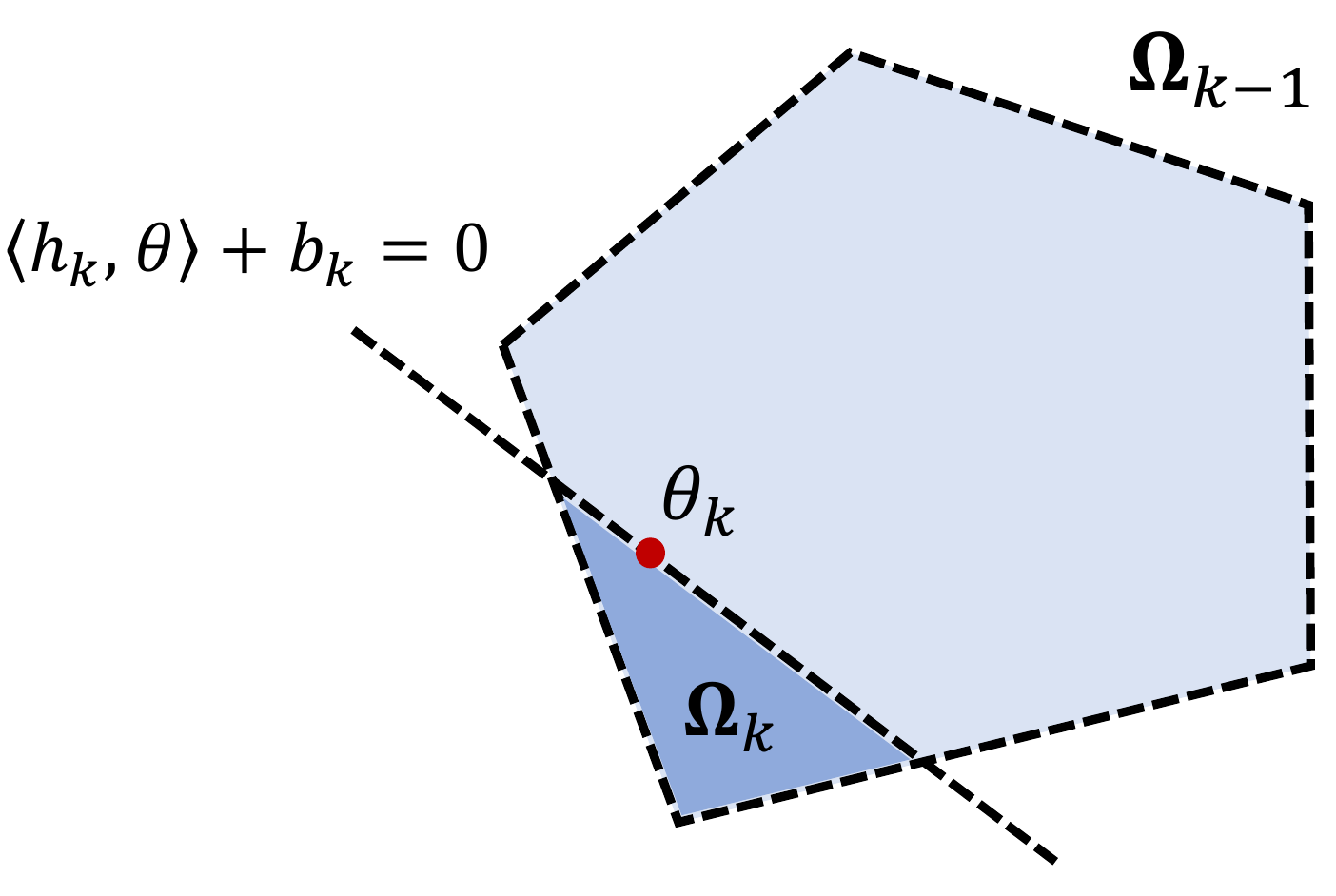}
		\caption{A large cut  from  $\boldsymbol{\Omega}_{k-1}$}
		\label{fig_interpretation2.1}
	\end{subfigure}
	\hfill
	\begin{subfigure}[b]{.40\linewidth}
		\includegraphics[width=\linewidth]{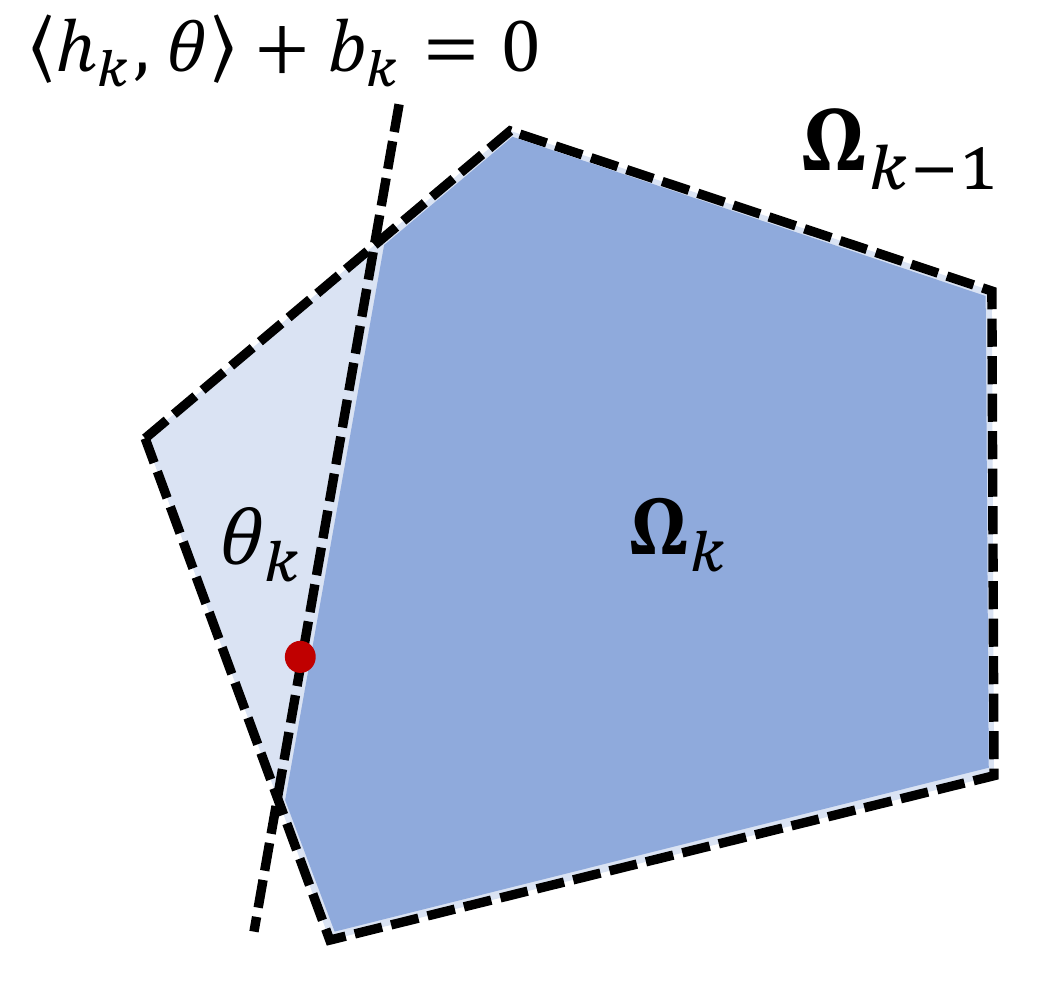}
		\caption{A small cut from  $\boldsymbol{\Omega}_{k-1}$}
		\label{fig_interpretation2.2}
	\end{subfigure}
	\caption{Illustration of how    different human directional corrections $\boldsymbol{\bar{a}}_{k}$ affect the cut of  the weight search space $\boldsymbol{\Omega}_{k-1}$.}
	\label{fig_interpretation2}
\end{figure}

Besides the above geometric illustration, we  also  have the following observations:

(1) The key idea of the proposed main algorithm is to cut and remove the weight search space $\boldsymbol{\Omega}_{k-1}$ as each human directional correction $\boldsymbol{\bar{a}}_{k}$ becomes available. Thus, we always expect that $\boldsymbol{\Omega}_{k-1}$ can quickly shrink  as $k$ increases, because thereby we can say that the     guess $\boldsymbol{\theta}_{k}$ is close  to the true weight vector $\boldsymbol{\theta}^*$. As shown in Fig. \ref{fig_interpretation},  the reduction rate of $\boldsymbol{\Omega}_{k-1}$  depends on two factors: the human directional correction  $\boldsymbol{\bar{a}}_{k}$,  and  how to choose   $\boldsymbol{\theta}_k\in\boldsymbol{\Omega}_{k-1}$, both discussed below.

(2)  From  (\ref{equ_hk}), we note that  human directional correction  $\boldsymbol{\bar{a}}_{k}$ determines $\boldsymbol{h}_{k}$, which is the normal vector of  the hyperplane $\left\langle\boldsymbol{h}_k, \boldsymbol{\theta}\right\rangle+b_k=0$. When fixing $\boldsymbol{\theta}_k$,  we can think of the hyperplane rotates around $\boldsymbol{\theta}_k$  with  different    {$\boldsymbol{\bar{a}}_{k}$},   leading to different removals of   $\boldsymbol{\Omega}_{k-1}$. This can be illustrated by comparing Fig. \ref{fig_interpretation2.1} and Fig. \ref{fig_interpretation2.2}.

(3)  The choice of the guess $\boldsymbol{\theta}_k\in\boldsymbol{\Omega}_{k-1}$ defines the specific location of the hyperplane $\left\langle\boldsymbol{h}_k, \boldsymbol{\theta}\right\rangle+b_k=0$, because the hyperplane is always passing through  $\boldsymbol{\theta}_k$  by (\ref{eq_prop1}) in Lemma \ref{lemma3}. Thus,  $\boldsymbol{\theta}_k$  also affects how $\boldsymbol{\Omega}_{k-1}$ is cut and  removed. This can be illustrated by comparing Fig. \ref{fig_interpretation.1} with Fig. \ref{fig_interpretation2.1}.

Based on the above observations, we know that the convergence of the proposed main algorithm is determined by the reduction of the weight search space $\boldsymbol{\Omega}_{k-1} $. This reduction  depends on both  the human directional correction {$\boldsymbol{\bar{a}}_k$}  and the   choice of the  weight vector guess $\boldsymbol{\theta}_k\in \boldsymbol{\Omega}_{k-1}$. In the next section, we will present a way of choosing  $\boldsymbol{\theta}_k$ to guarantee the convergence of the proposed main algorithm.

\section{Algorithm Implementation and Convergence Analysis}\label{section_results}

In this section, we  detail the choice of  $\boldsymbol{\theta}_k\in \boldsymbol{\Omega}_{k-1}$, show the convergence of the proposed algorithm, and present a detailed implementation of the algorithm with a termination criterion.

\subsection{Choice of Current Guess $\boldsymbol{\theta}_k\in \boldsymbol{\Omega}_{k-1}$}

In the proposed  main algorithm, at   iteration $k$,  the weight  search space $\boldsymbol{\Omega}_{k-1}$  is updated using (\ref{equ_updateOmega}), rewritten  below:
$$
\boldsymbol{\Omega}_{k}=\boldsymbol{\Omega}_{k-1}\cap\left\{\boldsymbol{\theta}\in\mathbb{R}^r\,|\, \left\langle\boldsymbol{h}_k,\boldsymbol{\theta} \right\rangle+b_k<0   \right\}.
$$
To show the reduction of a weight search space, it is straightforward to use the  volume measure (length or area for one- or two-dimension cases, respectively) of a (closure)   search space $\boldsymbol{\Omega}_k$, denoted as $\vol(\boldsymbol{\Omega}_k)$. Zero volume implies the convergence of the search space   \cite{boyd2007localization}. By  (\ref{equ_updateOmega}), we know that $\boldsymbol{\Omega}_{k}\subseteq\boldsymbol{\Omega}_{k-1}$ and thus $\vol(\boldsymbol{\Omega}_k)$ is non-increasing. In the following we need to further find a way such that $\vol(\boldsymbol{\Omega}_k)$ is strictly decreasing, i.e., there exists a constant $0\leq \alpha<1$ such that
\begin{equation}\label{equ_volume_reduction}
\vol(\boldsymbol{\Omega}_k)\leq \alpha \vol(\boldsymbol{\Omega}_{k-1}).
\end{equation} 

From Fig. \ref{fig_interpretation2}, 
we observe   that  for a fixed choice of  $\boldsymbol{\theta}_k  \in\boldsymbol{\Omega}_{k-1}$, different human directional corrections $\boldsymbol{\bar{a}}_k$ can lead to different reduction of $\boldsymbol{\Omega}_{k-1}$. For example,   Fig. \ref{fig_interpretation2.1} has
a large volume reduction from $\boldsymbol{\Omega}_{k-1} $ to $\boldsymbol{\Omega}_{k}$  while Fig. \ref{fig_interpretation2.2} leads to a very small  reduction.  Unfortunately, we cannot assume the specific direction of a  human correction  $\boldsymbol{\bar{a}}_k$ (it is the discretion of the  user), but we can choose the current guess $\boldsymbol{\theta}_k$ `smartly'  to avoid the very small  reduction regardless of specific human  corrections.  One intuitive choice of such  $\boldsymbol{\theta}_k$ is at some \emph{center} of the  search space $\boldsymbol{\Omega}_{k-1}$. Thus, we define the following center of the maximum volume ellipsoid (MVE) inscribed in  $\boldsymbol{\Omega}_{k-1}$.

\begin{definition}[Maximum Volume Inscribed Ellipsoid \cite{boyd2004convex}]\label{def_mve}
	Given a compact convex set $\boldsymbol{\Omega}\subset\mathbb{R}^r$, the maximum volume ellipsoid (MVE) inscribed in $\boldsymbol{\Omega}$, defined as $\boldsymbol{E}$,  is denoted by 
	\begin{equation}\label{equ_defmve}
	\boldsymbol{E}=\{ \bar{B}\boldsymbol{\theta}+\boldsymbol{\bar{d}}\,| \,\, 
	\norm{\boldsymbol{\theta}}_2\leq 1  \}. 
	\end{equation}
	Here,  $\bar{B}\in\mathbb{S}^r_{++}$ (i.e., a $r\times r$ positive definite matrix) and $\boldsymbol{\bar{d}}\in\mathbb{R}^r$ is called the center of $\boldsymbol{E}$. $\bar{B}$ and $\boldsymbol{\bar{d}}$ solve the  optimization:
	\begin{equation}\label{equ_mve_general}
	\begin{aligned}
	\max\nolimits_{\boldsymbol{{d}}, {B}\in\mathbb{S}^r_{{++}}}\,\, & \log\det{B}\\
	{s.t.}\,\, & \sup\nolimits_{\norm{\boldsymbol{\theta}}_2\leq 1} \emph{I}_{\boldsymbol{\Omega}}(B\boldsymbol{\theta}+\boldsymbol{d})\leq 0,
	\end{aligned} 
	\end{equation}
	where $\emph{I}_{\boldsymbol{\Omega}}(\boldsymbol{\theta})=0$ for $\boldsymbol{\theta}\in \boldsymbol{\Omega}$ and $\emph{I}_{\boldsymbol{\Omega}}(\boldsymbol{\theta})=\infty$ for $\boldsymbol{\theta}\notin \boldsymbol{\Omega}$.
\end{definition}

By Definition \ref{def_mve},  let $\boldsymbol{E}_k$ denote the MVE inscribed in $\boldsymbol{\Omega}_k$ and $\boldsymbol{d}_k$ denote the center of $\boldsymbol{E}_k$, $k=0,1,...$. For the choice of $\boldsymbol{\theta}_{k+1}\in\boldsymbol{\Omega}_k$, we choose
\begin{equation} \label{eq_thetak}
\boldsymbol{\theta}_{k+1}=  \boldsymbol{d}_{k},
\end{equation}
as illustrated in Fig. \ref{fig_interpretation3}. Other  choices of $\boldsymbol{\theta}_{k+1}$ using other types of centers of the search space  are discussed  in Appendix \ref{appendix_centerchoice}.

\begin{figure}[h]
	\centering
	\begin{subfigure}[b]{0.45\linewidth}
		\includegraphics[width=\linewidth]{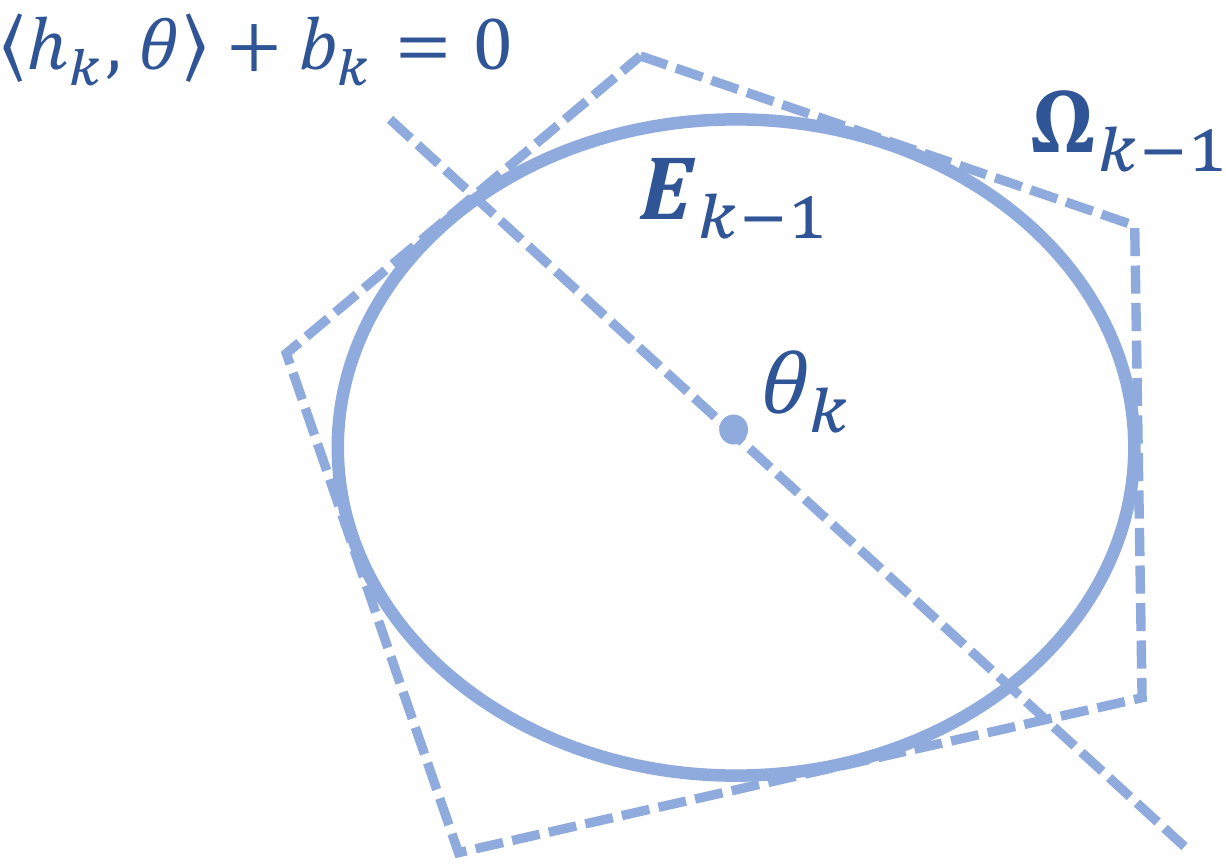}
		\caption{Center of MVE in $\boldsymbol{\Omega}_{k-1}$}
		\label{fig_interpretation3.1}
	\end{subfigure}
	\hspace{5pt}
	\begin{subfigure}[b]{.45\linewidth}
		\includegraphics[width=\linewidth]{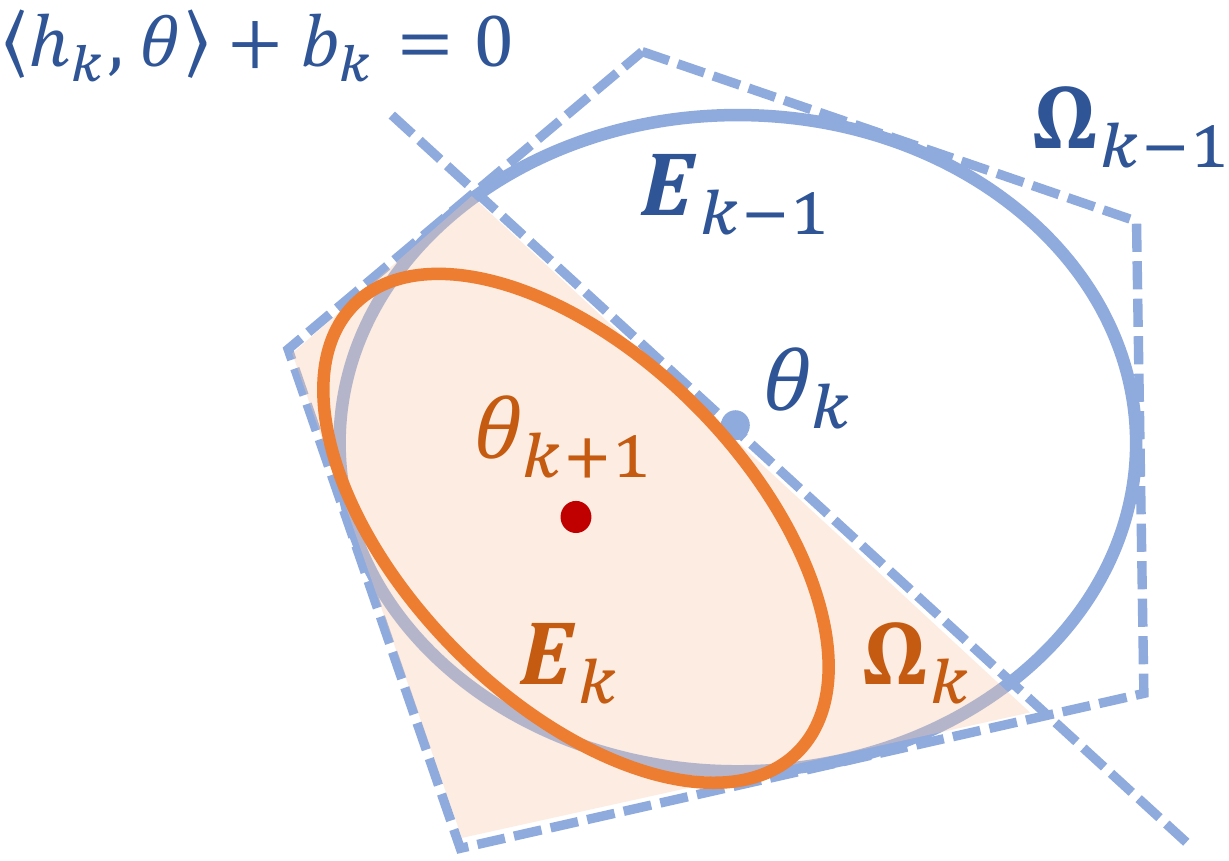}
		\caption{Center of MVE in $\boldsymbol{\Omega}_{k}$}
		\label{fig_interpretation3.2}
	\end{subfigure}
	\caption{Illustration of choosing  $\boldsymbol{\theta}_{k+1}$ as the center of MVE $\boldsymbol{E}_k$ inscribed in the  weight search space $\boldsymbol{\Omega}_k$.}
	\label{fig_interpretation3}
\end{figure}

We now give a computational method to  solve for   $\boldsymbol{d}_{k}$, i.e. the center of MVE $\boldsymbol{E}_k$ inscribed in $\boldsymbol{\Omega}_{k}$. Recall that in the  main algorithm,  the initial $	\boldsymbol{\Omega}_0$ is (\ref{equ_linftyball}), which is  equivalent to
\begin{equation}\label{equ_initialomega}
\boldsymbol{\Omega}_0=\left\{\boldsymbol{\theta}\,\,\,\middle\rvert\, \,
\begin{aligned}
\left\langle\boldsymbol{e}_i,\boldsymbol{\theta} \right\rangle-\bar{c}_i\leq0\\
-\left\langle\boldsymbol{e}_i,\boldsymbol{\theta} \right\rangle-\ubar{c}_i\leq0
\end{aligned},
\,\, \,\,\,  i=1,\cdots, r \right\},
\end{equation}
where $\boldsymbol{e}_i$ is a unit vector with the only $i$th entry as 1. Following the update  (\ref{equ_updateOmega}),  $\boldsymbol{\Omega}_k$ is also a compact polytope, which is
\begin{equation}\label{equ_omegak}
\boldsymbol{\Omega}_k=\left\{
\boldsymbol{\theta}\,\middle\rvert\,
\begin{aligned}
&\left\langle\boldsymbol{e}_i,\boldsymbol{\theta} \right\rangle-\bar{c}_i\leq0, \,\,  i=1,\cdots, r;\\
&-\left\langle\boldsymbol{e}_i,\boldsymbol{\theta}
\right\rangle-\ubar{c}_i\leq0, \,\,  i=1,\cdots, r;\\
&\left\langle\boldsymbol{h}_j,\boldsymbol{\theta} \right\rangle+b_j<0, \,\, j=1,\cdots, k
\end{aligned}
\right\}.
\end{equation}
As a result,  solving  (\ref{equ_mve_general}) for the center $\boldsymbol{d}_{k}$   becomes a convex program \cite{boyd2004convex},  stated in the lemma below.

\begin{lemma}\label{lemma5}
	For a polytope $\boldsymbol{\Omega}_{k}$ in (\ref{equ_omegak}), the center $\boldsymbol{d}_{k}$  of the MVE $\boldsymbol{E}_k$ inscribed in $\boldsymbol{\Omega}_{k}$ can be solved via the following convex optimization: 
	\begin{equation}
	\begin{aligned}\label{equ_mve}
	\min\nolimits_{\boldsymbol{d}, B\in\mathbb{S}^r_{{++}}}\,\,  &-\log\det{B}\\
	{s.t.}& \,\,  \norm{B\boldsymbol{e}_i}_2+\left\langle\boldsymbol{d}, \boldsymbol{e}_i \right\rangle\leq \bar{c}_i,\,\, i{=}1,\cdots r,\\
	& \,\,  \norm{B\boldsymbol{e}_i}_2-\left\langle\boldsymbol{d}, \boldsymbol{e}_i \right\rangle\leq \ubar{c}_i,\,\, i{=}1,\cdots r,\\
	&\,\, \norm{B\boldsymbol{{h}}_j}_2{+}\left\langle\boldsymbol{d}, \boldsymbol{h}_j\right\rangle\leq -b_j,\,\, j{=}1,\cdots k.
	\end{aligned} 
	\end{equation}
\end{lemma}
\noindent
The proof of  Lemma \ref{lemma5} can be found in Chapter 8.4.2 in  \cite[pp.414]{boyd2004convex}. (\ref{equ_mve}) can be efficiently solved by  available convex program solvers, e.g., \cite{diamond2016cvxpy}. In the implementation, since the number of linear inequalities grows as  $k$ increases, the trick of dropping some redundant inequalities in (\ref{equ_omegak}) can be adopted \cite{boyd2007localization}. Dropping redundant inequalities does not change $\boldsymbol{\Omega}_k$ and its volume reduction (convergence).  Please see how to identify the redundant inequalities in \cite{boyd2007localization}.

\subsection{Exponential Convergence  and Termination Criterion}
Now we analyze the convergence of the volume reduction of  $\boldsymbol{\Omega}_{k}$  and develop its termination criterion in implementation.  First,  the convergence of the reduction of $\vol(\boldsymbol{\Omega}_{k})$ is guaranteed by the following lemma.

\begin{lemma}\label{lemma6}
	Let $\boldsymbol{\theta}_k$ be chosen as the center $\boldsymbol{d}_{k-1}$ of the  MVE $\boldsymbol{E}_{k-1}$ inscribed in $\boldsymbol{\Omega}_{k-1}\subset \mathbb{R}^r$. Then, the update (\ref{equ_updateOmega}) has
	\begin{equation} \label{eq_volreduction}
	\frac{\vol({\boldsymbol{\Omega}_k})}{\vol(\boldsymbol{\Omega}_{k-1})}\leq (1-\frac{1}{r}).
	\end{equation}
\end{lemma}
\noindent
Lemma \ref{lemma6} is a  theorem directly from \cite{tarasov1988method}. Lemma \ref{lemma6} indicates   $$\vol{(\boldsymbol{\Omega}_k)}\leq (1-\frac{1}{r})^k\vol{(\boldsymbol{\Omega}_0)}.$$ 
Thus, the convergence rate of $\vol{(\boldsymbol{\Omega}_k)}\rightarrow 0$ is as fast as $(1-\frac{1}{r})^k\rightarrow 0$, as $k=0,1,2,...$.

To implement the proposed main algorithm, we will not only need the exponential convergence as established by Lemma \ref{lemma6} but also a termination criterion, which specifies the maximum number of iterations needed for  $\vol{(\boldsymbol{\Omega}_k)}$  below certain a given threshold. Thus, we have the following theorem.

\begin{theorem}\label{theorem_1}
	In the main algorithm, suppose $\boldsymbol{\Omega}_0$ is given by (\ref{equ_linftyball}), 
	and at iteration $k$,  $\boldsymbol{\theta}_k$ is chosen as the center $\boldsymbol{d}_{k-1}$ of the  MVE $\boldsymbol{E}_{k-1}$ inscribed in $\boldsymbol{\Omega}_{k-1}$. Given a termination condition     $$\vol{(\boldsymbol{\Omega}_k)}\leq (2\epsilon)^r$$ with $\epsilon$  a user-specified   threshold, the proposed  main algorithm runs for $k\leq K$ iterations, namely, the algorithm terminates in at most $K$ iterations, where
	\begin{equation}\label{equ_maxloop}
	K= \frac{r\log(R/\epsilon)}{-\log(1-1/r)}, 
	\end{equation}
	with $R$ given in (\ref{equ_radius_R}).
\end{theorem}
\begin{proof}
	Initially, we have $\vol{(\boldsymbol{\Omega}_0)}\leq (2R)^r$. From Lemma \ref{lemma6}, after $k$ iterations, we have 
	\begin{equation}
	\vol{(\boldsymbol{\Omega}_k)}\leq (1-\frac{1}{r})^k\vol{(\boldsymbol{\Omega}_0)}\leq (1-\frac{1}{r})^k (2R)^r,
	\end{equation} which yields to 
	\begin{equation}
	\log\vol{(\boldsymbol{\Omega}_k)}\leq k \log(1-\frac{1}{r})+\log(2R)^r.
	\end{equation}
	When $k=\frac{r\log(R/\epsilon)}{-\log(1-1/r)}$, 
	\begin{equation}
	\log\vol{(\boldsymbol{\Omega}_k)}\leq -r \log(R/\epsilon)+\log(2R)^r.
	\end{equation}
	The above equation is simplified to
	\begin{equation}
	\log \vol{(\boldsymbol{\Omega}_k)}\leq \log (2\epsilon)^r,
	\end{equation}
	which means that the termination condition $\vol{(\boldsymbol{\Omega}_k)}\leq (2\epsilon)^r$ is satisfied. This completes the proof.
\end{proof}

\noindent
We have the following comments on the above Theorem \ref{theorem_1}.

\begin{remark}\label{remark1}
	Since  Lemma \ref{lemma3} states that both $\boldsymbol{\theta}^*$ and $\boldsymbol{\theta}_k$ are    in  $\boldsymbol{\Omega}_k$ for any $k=1,2,3,...$ ,
	the user-specified threshold $\epsilon$ in the termination condition $\vol{(\boldsymbol{\Omega}_k)}\leq (2\epsilon)^r$ can also be understood as an indicator for the distance between the true  $\boldsymbol{\theta}^*$ (usually unknown in practice) and the  robot current guess $\boldsymbol{\theta}_k$.  $\epsilon$  is set based on the desired learning accuracy.  \footnote{In practice,  it is usually easy to set $\epsilon$ with a small value,  because the empirical results in Sections \ref{section_games} and  \ref{section.real_experiments} show that the robot motion trajectory can converge (to the desired motion) more quickly than the objective function itself. This means that one  usually  sees the convergence of the robot trajectory before  the termination condition is reached, as shown in Fig. \ref{fig_compare}. 
			More discussion about setting $\epsilon$ will be given in Appendix \ref{section.discussion.parameters}.}
\end{remark}

\subsection{Detailed Implementation of  Main Algorithm}

With the termination criterion stated in Theorem \ref{theorem_1} and the choice of  $\boldsymbol{\theta}_k$ in (\ref{eq_thetak}) and Lemma  \ref{lemma5}, we present the detailed implementation of the proposed main method  in Algorithm \ref{algorithm1}. The setting of the initial $\boldsymbol{\Omega}_0$ in (\ref{equ_linftyball}) and  $\epsilon$ will be given in Appendix \ref{section.discussion.parameters}.

\begin{algorithm2e}[th]
	\small 
	\SetKwInput{Initialization}{Initialization}
	\KwIn{Specify a termination threshold $\epsilon$ and use it to compute the maximum iteration  $K$ by (\ref{equ_maxloop}). }
	\Initialization{Initial weight search space $\boldsymbol{\Omega}_0$ in (\ref{equ_linftyball}).}
	\smallskip
	\For{$k=1,2,\cdots, K$}{

		\smallskip
		Choose a  weight vector guess $\boldsymbol{\theta}_k\in \boldsymbol{\Omega}_{k-1}$ via Lemma \ref{lemma5}\;
		
		\smallskip	
		Restart and plan a robot trajectory $\boldsymbol{\xi}_{\boldsymbol{\theta}_k}$ by solving a trajectory optimization  with the cost function  $J(\boldsymbol{\theta}_k)$ in (\ref{equ_objective}) and the dynamics  in (\ref{equ_dynamics})\;
		
		\smallskip
		Robot executes the planned trajectory $\boldsymbol{\xi}_{\boldsymbol{\theta}_k}$ while receving the human averaged directional correction $\boldsymbol{\bar{a}}_{k}$ in (\ref{equ_assumption_modify.abar})\;
		
		\smallskip
		Compute the  matrices $\boldsymbol{H}_1(\boldsymbol{x}^{\boldsymbol{\theta}_k}_{0:T\text{+}1},\boldsymbol{u}^{\boldsymbol{\theta}_k}_{0:T})$ and $\boldsymbol{H}_2(\boldsymbol{x}^{\boldsymbol{\theta}_k}_{0:T\text{+}1},\boldsymbol{u}^{\boldsymbol{\theta}_k}_{0:T})$, and then compute  the hyperplane and half space $\left\langle\boldsymbol{h}_k,\boldsymbol{\theta} \right\rangle+b_k<0 $ via (\ref{equ_assumption2})-(\ref{equ_hk})\;

		\smallskip
		Update the weight search space by $	\boldsymbol{\Omega}_{k}=\boldsymbol{\Omega}_{k-1}\cap\left\{\boldsymbol{\theta}\in \mathbb{R}^{r} \,|\, \left\langle\boldsymbol{h}_k,\boldsymbol{\theta} \right\rangle+b_k<0 \right\}$ via (\ref{equ_updateOmega})\;
	}
	\smallskip
	\textbf{Output:} $\boldsymbol{\theta}_K$.
	\caption{Learning from  directional corrections} \label{algorithm1}
\end{algorithm2e}

\section{Numerical Examples}\label{section.simulation}
In this section, we will test the proposed method in numerical simulations and make comparisons with the magnitude-correction methods discussed in the previous related work.

\subsection{Inverted Pendulum} \label{simulation_pendulum}
The dynamics of a pendulum is
\begin{equation}
\ddot{\alpha}=\frac{-g}{l}\sin\alpha-\frac{d}{ml^2}\dot{\alpha}+\frac{u}{ml^2},
\end{equation}
with $\alpha$ being the angle between the pendulum and the direction of gravity, $u$ is the torque applied to the pivot, $l=1$m, $m=1$kg, and $d=0.1$Ns/m are the length, mass, and damping ratio of the pendulum, respectively. We discretize the dynamics using the Euler method with a  time interval $\Delta=0.2$s.
The state  and control input of the pendulum   are  $\boldsymbol{x}=[\alpha,\dot{\alpha}]\tran$ and $\boldsymbol{u}=u$, respectively.  The initial condition is $\boldsymbol{x}_0=[0,0]\tran$. 
In the cost function  (\ref{equ_objective}), we set the feature and weight vectors as
\begin{equation}\label{pendulum_costfunction}
\begin{aligned}
\boldsymbol{\phi}&=[ \alpha^2,\alpha, \dot{\alpha}^2, u^2 ]\tran\in\mathbb{R}^4,\\
\boldsymbol{\theta}&=[\theta_1,\theta_2, \theta_3, \theta_4]\tran\in\mathbb{R}^4,
\end{aligned}
\end{equation}
respectively, and   the final cost   $h(\boldsymbol{x}_{T+1})=10(\alpha-\pi)^2+10\dot{\alpha}^2$,  as our goal is to  swing up the pendulum.
The discrete time horizon is $T=30$.

For numerical analysis, we `simulate' the human directional corrections instead of using real human corrections. The simulated corrections are generated as follows.  Suppose that the  true weight vector is known: $\boldsymbol{\theta}^*{=}[0.5,0.5,0.5,0.5]\tran$.
At  iteration $k$, a simulated     correction   $\boldsymbol{a}_{t_k}$ is generated using the \emph{sign} of  the  gradient of the true cost function $J(\boldsymbol{\theta}^*)$, i.e.,
\begin{equation}\label{correction_pendulum}
\boldsymbol{a}_{t_k}=-\text{sign}\left(\nabla J(\boldsymbol{u}^{\boldsymbol{\theta}_k}_{0:T},\boldsymbol{\theta}^*)\left[{t_k}\right]\right)
\in\mathbb{R}.
\end{equation}
Here, $\nabla J(\boldsymbol{u}^{\boldsymbol{\theta}_k}_{0:T},\boldsymbol{\theta}^*)\left[{t_k}\right]$ denotes the $t_k$th entry of $\nabla J$, and the  correction  time  ${t_k}$ is  randomly drawn  from the a uniform distribution over  horizon  $[0, T]$. Obviously, the above simulated  directional corrections satisfies  the assumption (\ref{equ_assumption}).

\begin{figure}[h]
	\centering
	\includegraphics[width=0.43\textwidth]{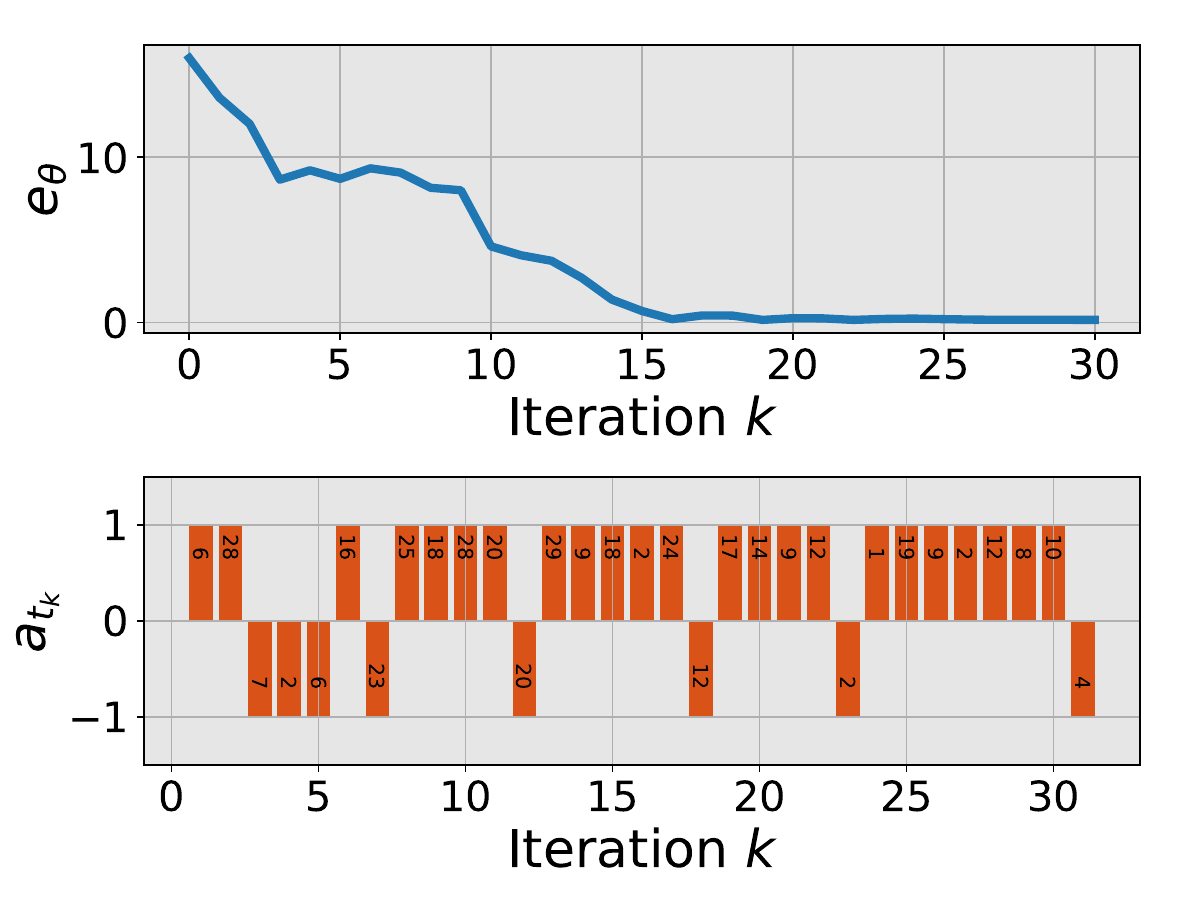}
	\caption{Learning  from simulated  directional corrections for the  pendulum. The upper panel shows the weight  error $e_{\boldsymbol{\theta}}=\norm{\boldsymbol{\theta}-\boldsymbol{\theta}^*}^2$ versus  iteration.  The bottom  shows the simulated  directional correction $\boldsymbol{a}_{t_k}$ (positive or negative sign)  at each iteration ${k}$; here, the value in each bar is  ${t}_k$, which  is randomly drawn from a uniform distribution over horizon $[0, 30]$.}
	\label{fig_pendulum_1}
\end{figure}

The initial weight search space is set as 
$\boldsymbol{\Omega}_0
=\{\boldsymbol{\theta}\,|\, 0\leq {\boldsymbol{\theta}[i]}\leq 5,\,\, i=1,2,3,4 \}$.
In Algorithm \ref{algorithm1}, we set the  termination threshold  $\epsilon=10^{-1}$, and the maximum learning iteration  solved via (\ref{equ_maxloop}) is $K=55$.  We apply  Algorithm \ref{algorithm1} to learn the  true  $\boldsymbol{\theta}^*$.  To illustrate   results, we define the weight error $e_{\boldsymbol{\theta}}=\norm{\boldsymbol{\theta}-\boldsymbol{\theta}^*}^2$ and plot  $e_{\boldsymbol{\theta}}$ versus iteration $k$ in the top panel of Fig. \ref{fig_pendulum_1}. In the bottom panel of Fig.  \ref{fig_pendulum_1}, we  plot the simulated  directional correction $\boldsymbol{a}_{t_k}$ at each iteration $k$, where $+1$ and $-1$ bar denote  the positive and negative sign (i.e., direction) of  correction $\boldsymbol{a}_{t_k}$ in (\ref{correction_pendulum}), respectively; the  number inside the bar denotes the  correction  time ${t_k}$ randomly  drawn from $[0, T]$.

From the results in Fig. \ref{fig_pendulum_1}, we can see that as the learning iteration $k$ increases, the weight vector guess $\boldsymbol{\theta}_k$ converges to the true $\boldsymbol{\theta}^*=[0.5,0.5,0.5,0.5]\tran$. This  shows the efficacy of the  method, as guaranteed by Theorem~\ref{theorem_1}.

\subsection{Two-link Robot Arm  System}\label{section_exp_robotarm_1}
Here, we test the proposed method on a two-link robot arm.  The dynamics of the robot arm  system (moving horizontally) is $M(\boldsymbol{q})\ddot{\boldsymbol{q}}+\boldsymbol{c}(\boldsymbol{q},\boldsymbol{\dot{q}})=\boldsymbol{\tau}$, where $M(\boldsymbol{q})$ is the inertia matrix, $\boldsymbol{c}(\boldsymbol{q},\boldsymbol{\dot{q}})$ is the Coriolis and centrifugal term; $\boldsymbol{q}=[q_1,q_2]\tran$ is the vector of  joint angles, and $\boldsymbol{\tau}=[\tau_1, \tau_2]\tran$ is the vector of  joint toques. The state  and control input   are  $\boldsymbol{x}=[\boldsymbol{q},\dot{\boldsymbol{q}}]\tran\in\mathbb{R}^4$ and $\boldsymbol{u}=\boldsymbol{\tau}\in\mathbb{R}^2$, respectively. The initial condition of the robot arm is set as $\boldsymbol{x}_0=[0,0,0,0]\tran$. All physical  parameters in the dynamics are  units. We discretize the   dynamics using the Euler method with a time interval $\Delta=0.2$s. In the cost function  (\ref{equ_objective}), we set the feature and weight vectors as
\begin{subequations}
	\begin{align}
	\boldsymbol{\phi}&=[ q_1^2,\, q_1,\, {q}_2^2,\,{q}_2, \, \norm{\boldsymbol{u}}^2]\tran\in\mathbb{R}^5,\\
	\boldsymbol{\theta}&=[\theta_1,\theta_2,\theta_3,\theta_4,\theta_5]\tran\in\mathbb{R}^5,
	\end{align}
\end{subequations} respectively, and   the final cost   $h(\boldsymbol{x}_{T+1})=100\big((q_1-\frac{\pi}{2})^2+q_2^2+\dot{q}_1^2+\dot{q}_2^2\big)$, as  the robot arm aims to  reach and stop at the pose of $\boldsymbol{q}=[\frac{\pi}{2}, 0]\tran$. 
The discrete-time horizon is  $T=50$.

We still `simulate'  directional corrections. Suppose that we  know the  true  $\boldsymbol{\theta}^*=[1,1,1,1,1]\tran$. At each iteration $k$, the simulation generates a directional correction $\boldsymbol{a}_{t_k}$   using   the sign of     gradient of the true  $J(\boldsymbol{\theta}^*)$:
\begin{equation}\label{correction_robotam}
\boldsymbol{a}_{t_k}=-\text{sign}\left(
\nabla J(\boldsymbol{u}^{\boldsymbol{\theta}_k}_{0:T},\boldsymbol{\theta}^*)\left[{2t_k:2t_k+1}
\right]
\right)\in\mathbb{R}^2,
\end{equation}
where $\nabla J(\boldsymbol{u}^{\boldsymbol{\theta}_k}_{0:T},\boldsymbol{\theta}^*)\left[{2t_k:2t_k+1}
\right]$ are  the  entries  at the locations from $2t_k$ to $2t_k+1$ in $\nabla J$ (because $\boldsymbol{u}\in\mathbb{R}^{2}$), and 
the correction   time  $t_k$ is  randomly  drawn from (uniform distribution)  time horizon $[0, T]$.  \texttt{sign}  is applied entry-wise.

\begin{figure}[h]
	\centering
	\includegraphics[width=0.4\textwidth]{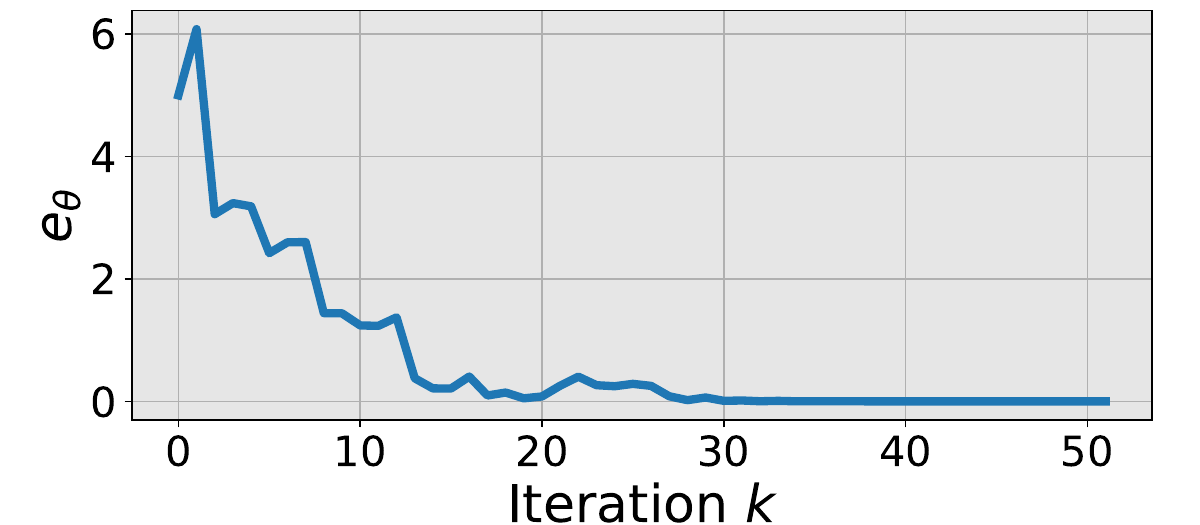}
	\caption{Learning  from simulated directional corrections for the  robot arm. The plot is $e_{\boldsymbol{\theta}}=\norm{\boldsymbol{\theta}-\boldsymbol{\theta}^*}^2$ versus  iteration. The corrections are given in Fig. \ref{fig_robotarm_2}. }
	\label{fig_robotarm_1}
\end{figure}

The initial weight search space is  $\boldsymbol{\Omega}_0
=
\{
\boldsymbol{\theta} \,\,|\,\, 
0\leq \boldsymbol{\theta}[i]\leq
4, \,\,\, i=1,2,..., 5 
\}$.
We set the  termination threshold   $\epsilon=10^{-1}$, and the maximum learning iteration is $K=83$ by (\ref{equ_maxloop}).  We apply  Algorithm~\ref{algorithm1} to learn the true  $\boldsymbol{\theta}^*$.  We define the weight error $e_{\boldsymbol{\theta}}=\norm{\boldsymbol{\theta}-\boldsymbol{\theta}^*}^2$ and plot  $e_{\boldsymbol{\theta}}$ versus  iteration  in Fig. \ref{fig_robotarm_1}. We also plot  $\boldsymbol{a}_{t_k}=[a_{t_k,1},a_{t_k,2}]\tran$   at each iteration $k$ in Fig. \ref{fig_robotarm_2}, where the value labeled on each bar marks the correction  time  $t_k$.  The results in Fig. \ref{fig_robotarm_1} show that as the  iteration increases,  $\boldsymbol{\theta}_k$ converges to the true $\boldsymbol{\theta}^*$, which  again confirms  Theorem~\ref{theorem_1}.

\begin{figure}[h]
	\centering
	\includegraphics[width=0.4\textwidth]{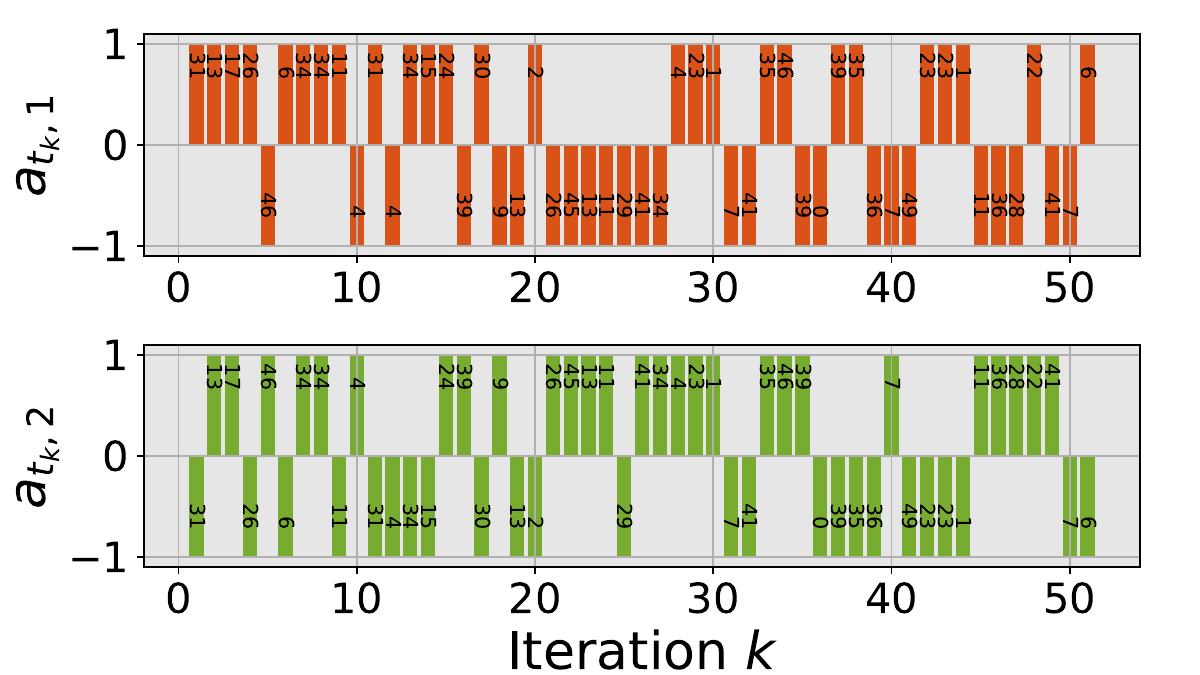}
	\caption{Simulated directional correction  $\boldsymbol{a}_{t_k}=[a_{t_k,1},a_{t_k,2}]\tran$  at each iteration $k$ for  the robot arm.  The number inside the bar is the correction time  ${t}_k$ randomly drawn from range $[0,50]$.}
	\label{fig_robotarm_2}
\end{figure}

\subsection{Comparison with  Magnitude-Correction Methods} \label{experiment_compare}

We compare the proposed method with two related works \cite{jain2015learning,bajcsy2017learning}, both of which learn an objective function from magnitude corrections.  The comparison is conducted on the inverted pendulum used in Section \ref{simulation_pendulum}. 
According to   \cite{bajcsy2017learning}, given a  magnitude correction,  the trajectory deformation  \cite{zhang2019learning} is  used to  obtain  a    \emph{human intended trajectory}.  Specifically, suppose the robot  trajectory is   $\boldsymbol{{\xi}}_{\boldsymbol{\theta}_k}=\{\boldsymbol{{x}}^{\boldsymbol{\theta}_k}_{0:T+1},\boldsymbol{{u}}^{\boldsymbol{\theta}_k}_{0:T}\}$. 
Given   a  correction $\boldsymbol{\bar{a}}_k$,   the human intended trajectory, denoted as  $\boldsymbol{\bar{\xi}}_{\boldsymbol{\theta}_k}=\{\boldsymbol{\bar{x}}^{\boldsymbol{\theta}_k}_{0:T+1},\boldsymbol{\bar{u}}^{\boldsymbol{\theta}_k}_{0:T} \}$,  is  solved as
\begin{equation}\label{equ_compare}
\boldsymbol{\bar{u}}^{\boldsymbol{\theta}_k}_{0:T}=\boldsymbol{{u}}^{\boldsymbol{\theta}_k}_{0:T}+M^{-1}\boldsymbol{\bar{a}}_k,
\end{equation}
where  $M$ is a matrix to  propagate a single-time-step correction $\boldsymbol{\bar{a}}_k$ (\ref{equ_aug_correction_vec}) along the rest of the robot trajectory \cite{dragan2015movement}, and $\boldsymbol{\bar{x}}^{\boldsymbol{\theta}_k}_{0:T+1}$ in $\boldsymbol{\bar{\xi}}_{\boldsymbol{\theta}_k}$  is obtained via rolling out the robot dynamics  (\ref{equ_dynamics}) given $\boldsymbol{\bar{u}}^{\boldsymbol{\theta}_k}_{0:T}$.
The learning  update used in both \cite{jain2015learning} and \cite{bajcsy2017learning}   is
\begin{equation}\label{equ_compare_update}
\boldsymbol{\theta}_{k+1}=\boldsymbol{\theta}_{k}+\eta\Big(\boldsymbol{\phi}(\boldsymbol{\bar{\xi}}_{\boldsymbol{\theta}_k})-\boldsymbol{\phi}(\boldsymbol{{\xi}}_{\boldsymbol{\theta}_k})\Big),
\end{equation}
where $\boldsymbol{\phi}(\boldsymbol{\bar{\xi}}_{\boldsymbol{\theta}_k})$ and $\boldsymbol{\phi}(\boldsymbol{{\xi}}_{\boldsymbol{\theta}_k})$ are  the  feature vectors of the human intended trajectory $\boldsymbol{\bar{\xi}}_{\boldsymbol{\theta}_k}$ and  the   uncorrected robot trajectory $\boldsymbol{{\xi}}_{\boldsymbol{\theta}_k}$, respectively.
Here, we set $M$ as  the finite differencing matrix \cite{zhang2019learning} and $\eta=0.0002$ (for  best performance).

In  comparison, we  use the simulated  corrections $\boldsymbol{a}_{t_k}$, which are generated from the true  $\boldsymbol{\theta}^*$, as  in (\ref{correction_pendulum}). We set the magnitude $\norm{\boldsymbol{a}_{t_k}}$ to three levels: $\norm{\boldsymbol{a}_{t_k}}=0.00125$, $\norm{\boldsymbol{a}_{t_k}}=0.001$, and $\norm{\boldsymbol{a}_{t_k}}=0.0008$,  because we want to see how sensitive the magnitude-correction methods \cite{jain2015learning,bajcsy2017learning} are to the  magnitudes of   corrections. Here, the correction time $t_k$ for different magnitude levels is different random draws. To illustrate    results, we measure the following three aspects for all methods.

\begin{itemize}
	\item \emph{Weight Error},  defined as $\norm{\boldsymbol{\theta}_k-\boldsymbol{\theta}^*}^2$. This is to measure the  error between the  learned $\boldsymbol{\theta}_k$ and the true  $\boldsymbol{\theta}^*$.
	
	\item \emph{Cost Regret}\cite{jain2015learning},  defined as $(\boldsymbol{\theta}^*)\tran\big(\boldsymbol{\phi}(\boldsymbol{{\xi}}_{\boldsymbol{\theta}_k})-\boldsymbol{\phi}(\boldsymbol{\xi}_{\boldsymbol{\theta}^*})\big)$. This is to measure  the  misalignment of the costs between  robot  current trajectory $\boldsymbol{{\xi}}_{\boldsymbol{\theta}_k}$ and the desired trajectory $\boldsymbol{\xi}_{\boldsymbol{\theta}^*}$,  evaluated under the  true cost function  $J(\boldsymbol{\theta}^*)=(\boldsymbol{\theta}^*)\tran \boldsymbol{\phi}(\boldsymbol{\xi})$.

	\item \emph{Trajectory error}, defined as $\norm{\boldsymbol{\xi}_{\boldsymbol{\theta}_k}-\boldsymbol{\xi}_{\boldsymbol{\theta}^*}}^2$, which measures the distance between the robot current trajectory $\boldsymbol{\xi}_{\boldsymbol{\theta}_k}$  and the desired trajectory $\boldsymbol{\xi}_{\boldsymbol{\theta}^*}$.
\end{itemize}

Given the simulated corrections $\boldsymbol{a}_{t_k}$, we run  Algorithm \ref{algorithm1} and the  method \cite{jain2015learning,bajcsy2017learning} (i.e.,  updating the weight vector via (\ref{equ_compare_update})). We show the results in Fig. \ref{fig_compare}, where the results of our method are in orange and the results of the magnitude-correction method \cite{jain2015learning,bajcsy2017learning} are in blue. Each column in Fig. \ref{fig_compare} corresponds to one level of correction magnitudes  (recall that we want to see how robust a method is against different magnitudes of corrections). In each magnitude level (each column), the first row shows the weight error versus iteration,   the second row shows the cost regret versus iteration, and the third row shows the trajectory error versus iteration. Based on the results in Fig.\ref{fig_compare}, we make the following comments.

\begin{figure}[h]
	\begin{subfigure}{.155\textwidth}
		\centering
		\includegraphics[width=\linewidth]{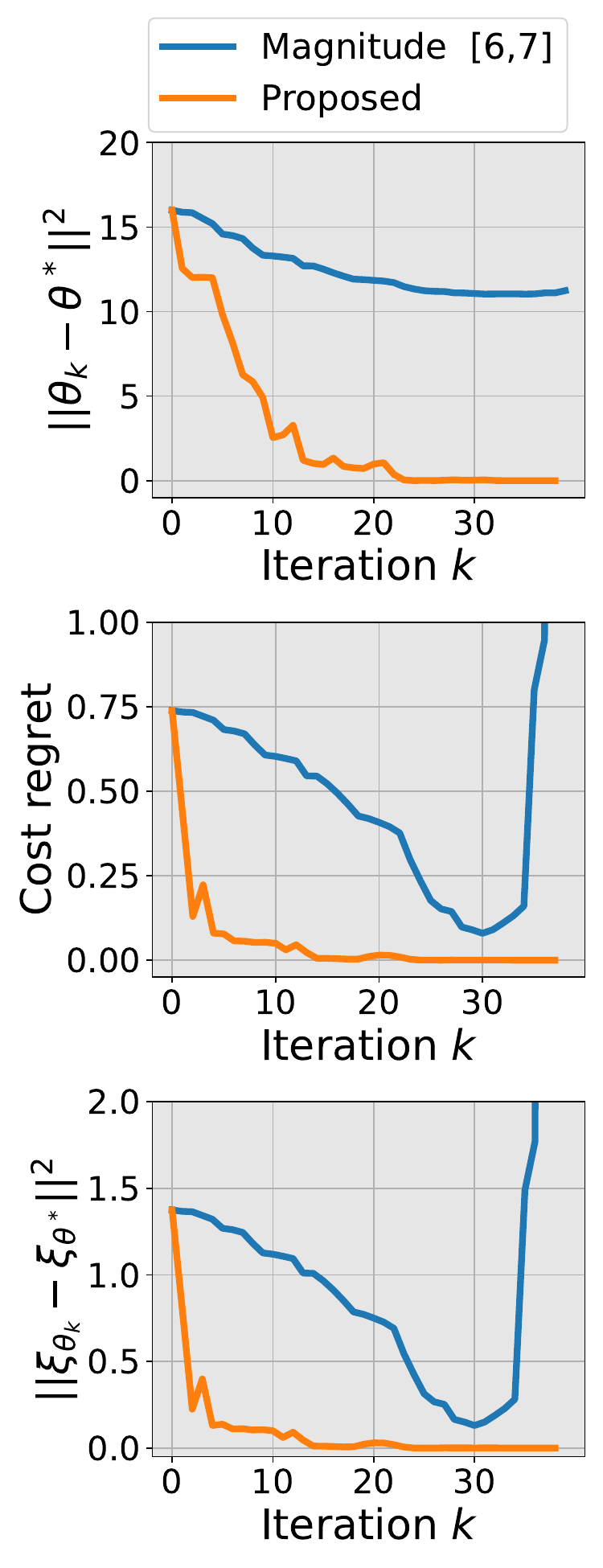}
		\caption{$\norm{\boldsymbol{a}_{t_k}}=0.00125$}
		\label{fig_compare.1}
	\end{subfigure}
	\hfill
	\begin{subfigure}{.155\textwidth}
		\centering
		\includegraphics[width=\linewidth]{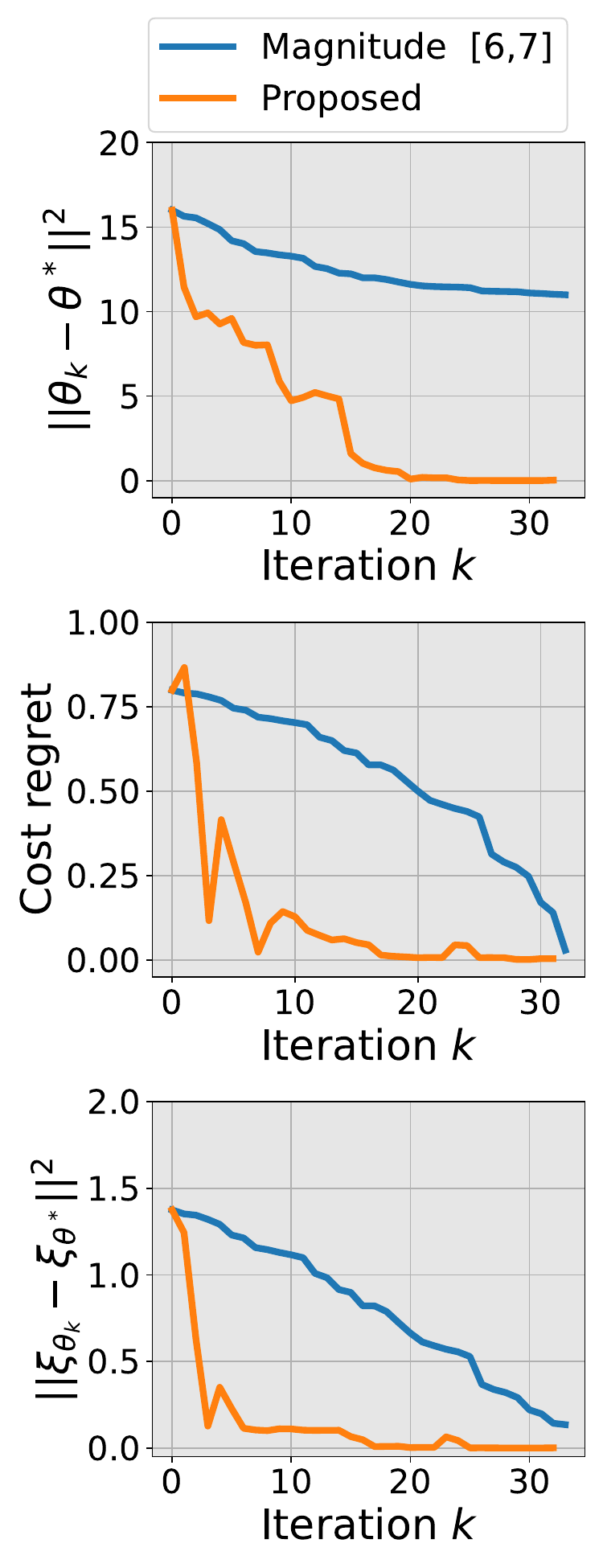}
		\caption{$\norm{\boldsymbol{a}_{t_k}}=0.001$}
		\label{fig_compare.2}
	\end{subfigure}
	\hfill
	\begin{subfigure}{.155\textwidth}
		\centering
		\includegraphics[width=\linewidth]{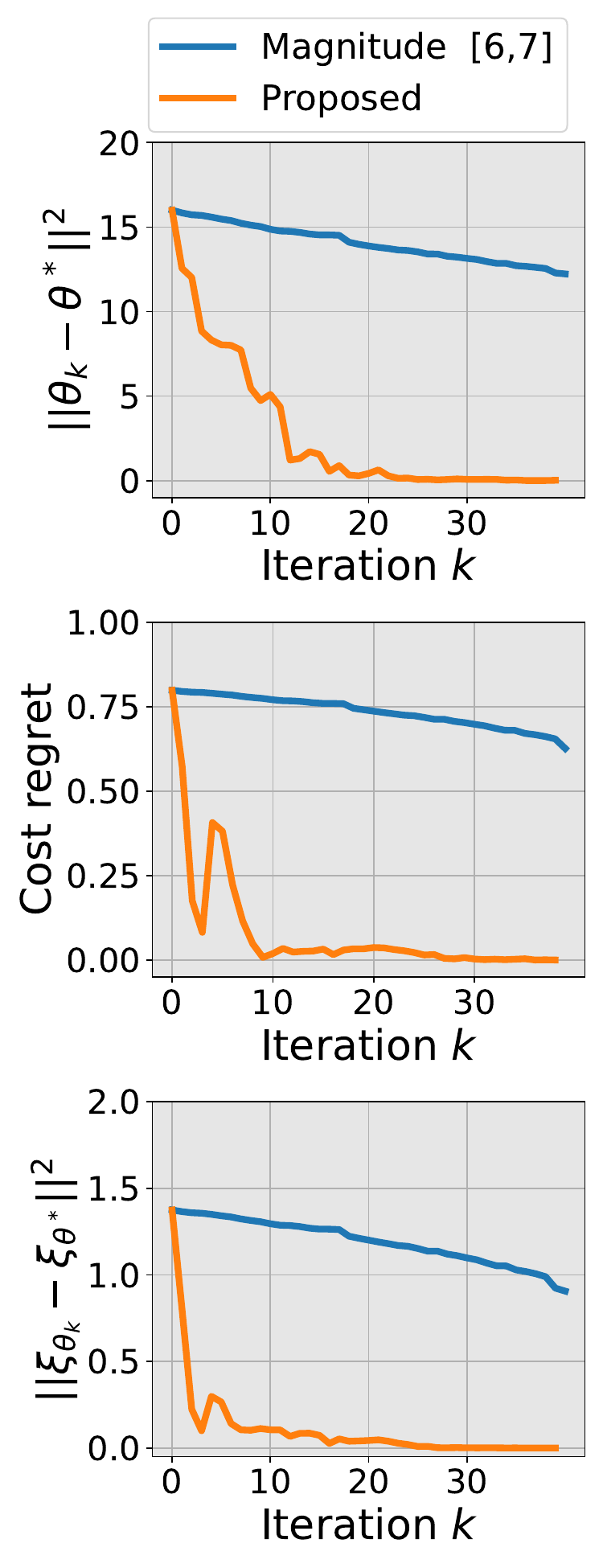}
		\caption{$\norm{\boldsymbol{a}_{t_k}}=0.0008$}
		\label{fig_compare.3}
	\end{subfigure}
	\caption{Comparison between the proposed method (in orange lines) and the magnitude-correction methods   \cite{jain2015learning,bajcsy2017learning} (in blue lines). Each column corresponds to a different  level of  correction magnitude, and  correction time $t_k$ in different magnitude levels are  different random draws. For each magnitude level (each column), the first row shows the weight error versus iteration; the second row is the cost regret versus iteration; and the third row is the trajectory error  versus iteration. Detailed analysis is given in the text.}
	\label{fig_compare}
\end{figure}

For the magnitude-correction methods \cite{jain2015learning,bajcsy2017learning} (in blue lines) in Fig. \ref{fig_compare}, comparing different columns, we note that larger magnitude  leads to faster convergence. However, as shown by the left and middle panels in the first row, the magnitude-correction method converges to a  local weight vector but not $\boldsymbol{\theta}^*$.   This is not a surprise,  because the magnitude-correction method can only guarantee the convergence of the cost regret, as shown by \cite{jain2015learning}, not necessarily the convergence of the cost function itself. One conjecture to the above results could be that the local weight vector (different from the true $\boldsymbol{\theta}^*$)  results in the same robot trajectory and the same optimal cost as $\boldsymbol{\theta}^*$ does. This conjecture has been evidenced in the second row and third row in Fig. \ref{fig_compare}, where the left and middle panels show that the magnitude-correction methods, although having learned a different weight vector, have both the cost regret and trajectory error converging to zero.

Also for the magnitude-correction methods \cite{jain2015learning,bajcsy2017learning} (in blue lines),
when $\norm{\boldsymbol{a}_{t_k}}=0.00125$ (first column),   we  observe that   there is a sharp increase of the cost regret and trajectory error near   convergence after 30 iterations. This could be because   the  magnitude-correction methods \cite{jain2015learning,bajcsy2017learning} are not robust against   the large magnitude of corrections near the desired trajectory. Specifically,  after $30$ iterations near   convergence, a relatively larger  magnitude, say $\norm{\boldsymbol{a}_{t_k}}=0.00125$, can be an over-correction  and thus violating  $J(\boldsymbol{\bar{u}}^{\boldsymbol{\theta}_k}_{0:T},\boldsymbol{\theta}^*)< J(\boldsymbol{{u}}^{\boldsymbol{\theta}_k}_{0:T},\boldsymbol{\theta}^*)$. This explanation has been  supported by  the second column of Fig. \ref{fig_compare}, where with a smaller magnitude $\norm{\boldsymbol{a}_{t_k}}=0.001$, there is no such over-correction phenomenon.

Combining all three columns in Fig. \ref{fig_compare}, one can conclude that  for  the magnitude-correction methods \cite{jain2015learning,bajcsy2017learning},  large correction magnitude, such as $\norm{\boldsymbol{a}_{t_k}}=0.00125$, will have faster convergence, but it can lead to  over-corrections, as shown in Fig. \ref{fig_compare.1}, making the learning process unstable. Smaller correction magnitude, such as $\norm{\boldsymbol{a}_{t_k}}=0.0008$, can avoid the  over-correction issue, but it leads to  slow convergence, as shown in Fig. \ref{fig_compare.3}. Notice that the magnitude change from $\norm{\boldsymbol{a}_{t_k}}=0.00125$ (unstable) in Fig. \ref{fig_compare.1} to  $\norm{\boldsymbol{a}_{t_k}}=0.0008$ (stable) in Fig. \ref{fig_compare.3} is very small, which could suggest that the magnitude-correction methods are sensitive to the selection of correction magnitudes. Thus, it could be difficult in practice to provide a   valid correction magnitude. We will further show and analyze this in our following user study in Section \ref{section_games}.

In contrast, Fig. \ref{fig_compare} has shown that the proposed method consistently learns the true  $\boldsymbol{\theta}^*$ and achieves a zero cost regret and zero trajectory error regardless of the correction magnitude levels.   Also, compared to the  magnitude-correction methods \cite{jain2015learning,bajcsy2017learning}, the proposed method requires less learning iterations (corrections) for convergence. 
Since the proposed method only leverages the direction of  $\boldsymbol{a}_{t_k}$  regardless of $\norm{\boldsymbol{a}_{t_k}}$, there is no over-correction issues near  convergence. In practice, the choice of directional corrections is more flexible than choosing  magnitude corrections, as  analyzed in Fig. \ref{figure_correctionvs} in  Section \ref{section_problem}. We  will continue to show this advantage in the following user study in Section \ref{section_games}.

\section{User Study}\label{section_games}

To show the effectiveness of the proposed method for learning from real human directional corrections, we have developed two human-robot simulation games based on which a user study is conducted. One game is  robot arm reaching  (Fig. \ref{fig_robot_iteration0}) and the other is 6-DoF  quadrotor maneuvering  (Fig. \ref{fig_uav_iteration0}). In each game, a human participant observes the visualization of robot motion, while applying directional corrections through a keyboard. The goal of each game is to teach a robot to learn an objective function, such that it can successfully navigate through an environment without the knowledge about the environment obstacles. 
We have released the accompanying codes of those two games for the readers to try themselves: \url{https://github.com/wanxinjin/Learning-from-Directional-Corrections}.

In what follows next, Sections \ref{game.arm} and \ref{game.quadrotor} describe  the designs of the two games. Section \ref{game.participant} provides the details of human participants and the procedure of the user study. The outcomes of the user study are presented in Section \ref{game.measurement} and detailed analysis is given in  Section \ref{game.resuts}.

\subsection{Robot Arm Reaching Game}\label{game.arm}

\begin{figure}[h]
	\centering
	\includegraphics[width=0.25\textwidth]{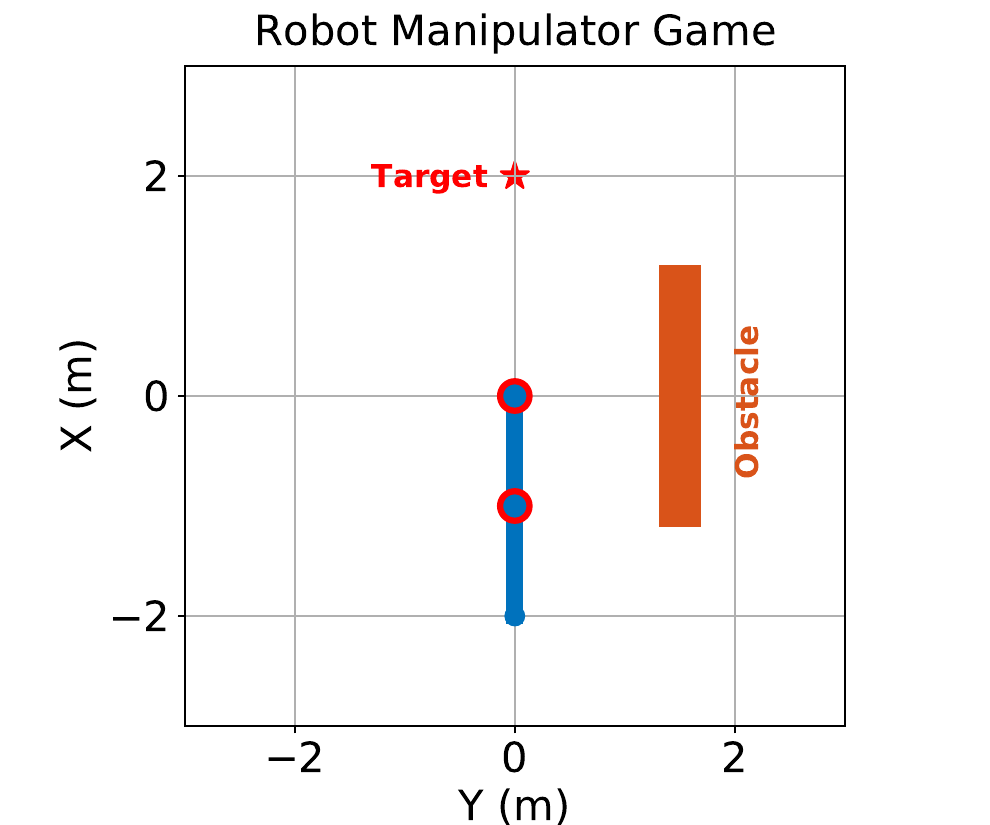}
	\caption{The robot arm reaching game. A human participant teaches the robot arm to learn a cost function by applying directional corrections via a keyboard (see the instructions in Table \ref{table_keycustomize}). The goal is to let the robot arm successfully move from the initial pose (the pose in the figure)  to reach the target (the star) while avoiding the obstacle.  Note that we have no knowledge about the location and size of the obstacle.} 
	\label{fig_robot_iteration0}
\end{figure}

\subsubsection{Robot Setup}\label{exp_armsimenv} 
The dynamics of a robot arm and  its physical parameters  follow Section \ref{section_exp_robotarm_1}.  The robot initial state is $\boldsymbol{x}_0{=}[-\frac{\pi}{2}, 0, 0, 0]\tran$, as  in Fig. \ref{fig_robot_iteration0}. The parameterized cost function $J(\boldsymbol{\theta})$ in (\ref{equ_objective}) is
\begin{subequations}\label{robotarm_features}
	\begin{align}
	\boldsymbol{\phi}&=[ q_1^2,\, q_1,\, {q}_2^2,\,{q}_2, \, \norm{\boldsymbol{u}}^2]\tran\in\mathbb{R}^5,\label{robotarm_features.1}\\
	\boldsymbol{\theta}&=[\theta_1,\theta_2,\theta_3,\theta_4,\theta_5]\tran\in\mathbb{R}^5.
	\end{align}
\end{subequations} 
Here, $\boldsymbol{\theta}\tran\boldsymbol{\phi}$  is  a general second-order polynomial. It is worth noting that in practice if prior knowledge about good features is not available,  general polynomial features are a good option.
The final cost  $h(\boldsymbol{x}_{T+1})$ in (\ref{equ_objective}) is 
\begin{equation}\label{equ_exp_pendulum_finalcost}
h(\boldsymbol{x}_{T+1})=100\big((q_1-\frac{\pi}{2})^2+q_2^2+\dot{q}_1^2+\dot{q}_2^2\big),
\end{equation} 
as   the robot arm aims to finally reach and stop at  the target pose: ${q}^{\text{target}}_1=\frac{\pi}{2}$ and ${q}^{\text{target}}_2=0$,  marked in Fig. \ref{fig_robot_iteration0}. The time  horizon of the game  is set as $T=50$ (that is, $50\Delta=10$s). Since we   choose the polynomial features  (\ref{robotarm_features.1}),      $\boldsymbol{\theta}$ will dictate  how the robot  reaches  the target (i.e., the specific trajectory to the goal).
By default, the initial  search space $\boldsymbol{\Omega}_0$ is set 
\begin{equation}\label{exp_robotarm_initialcondition}
\boldsymbol{\Omega}_0{=}\{\boldsymbol{{\theta}}\,|\,
\theta_1, \theta_3\in[0,1], \theta_2,\theta_4\in[-3,3], \theta_5\in[0, 0.5]
\}.
\end{equation}

Without human corrections, the robot with a random $\boldsymbol{\theta}_0$  will move and crash into the obstacle in Fig. \ref{fig_robot_iteration0}. The goal of the game is to let a human participant teach the robot arm to learn a  cost function, such that it can  move from the initial condition  and reach the target pose while avoiding the obstacle. Note that we have no knowledge about the  obstacle.

\begin{table}[h]
	\centering
	\caption	{Keyboard interface for the robot arm game.}
	\begin{tabular}{lll}
		\toprule
		Keys    & Directional correction & Interpretation of correction \\
		\midrule
		
		\emph{up} &   $\boldsymbol{a}=[1,0]$ &  counter-close-wise torque to Joint 1 \\
		
		\emph{down} &   $\boldsymbol{a}=[-1,0]$ &  close-wise torque to Joint 1 \\
		
		\emph{left} &   $\boldsymbol{a}=[0,1]$ &  counter-close-wise torque to Joint 2  \\
		
		\emph{right} &   $\boldsymbol{a}=[0,-1]$ &  close-wise torque to Joint 2 \\
		
		\bottomrule
	\end{tabular}
	
	\begin{tablenotes}
		\item[1] \textsuperscript{\textcolor{red}{*}} 	When implementing  \cite{jain2015learning,bajcsy2017learning},
		to allow   magnitude corrections via the  same interface, we additionally  detect the key pressing duration  $\delta t$  and interpret the  correction magnitude $u$ to be proportional to  $\delta t$ with saturation $u_{\max}$; i.e., $u=\min\{\beta \delta t, u_{\max}\}$ with $\beta=2$ and $u_{\max}=1$.  The magnitude correction using the above interface is $\boldsymbol{u}=u\boldsymbol{a}$. The correction time $t_k$ is the beginning time of the key press. 
	\end{tablenotes}
	
	\label{table_keycustomize}
\end{table}

\subsubsection{Interface} \label{exp_armsiminterface}

In the  robot arm  game, we use  a keyboard as the interface for a  human participant to apply directional corrections. We  customize  (\emph{up, down, left, right}) keys and associate them  with corresponding  directional corrections as listed in Table \ref{table_keycustomize}. During the game, at each learning iteration,  a human participant is allowed to press one or multiple keys from (\emph{up, down, left, right}), and  the  interface is listening to which key(s) the human player hits and  recording the timing  of the keystroke(s). The recorded information is translated into the directional correction $\boldsymbol{a}_{t_k}$ as per Table \ref{table_keycustomize}. For example, at  iteration $k$, while the robot  is  executing the motion trajectory $
\boldsymbol{\xi}_{\boldsymbol{\theta}_k}$, a human player hits the  \emph{up} and \emph{left} keys simultaneously at the time step $10$; then the corresponding correction information is translated into $\boldsymbol{a}_{t_k}=[1,1]\tran$ with $t_k=10$ according to Table~\ref{table_keycustomize}. We obtain averaged $\boldsymbol{\bar{a}}_{k}$ from $\boldsymbol{a}_{t_k}$ via (\ref{equ_aug_correction_vec}) and (\ref{equ_assumption_modify.abar}).

\subsection{6-DoF Quadrotor  Maneuvering Game}\label{game.quadrotor}

\begin{figure}[h]
	\centering
	\includegraphics[width=0.35\textwidth]{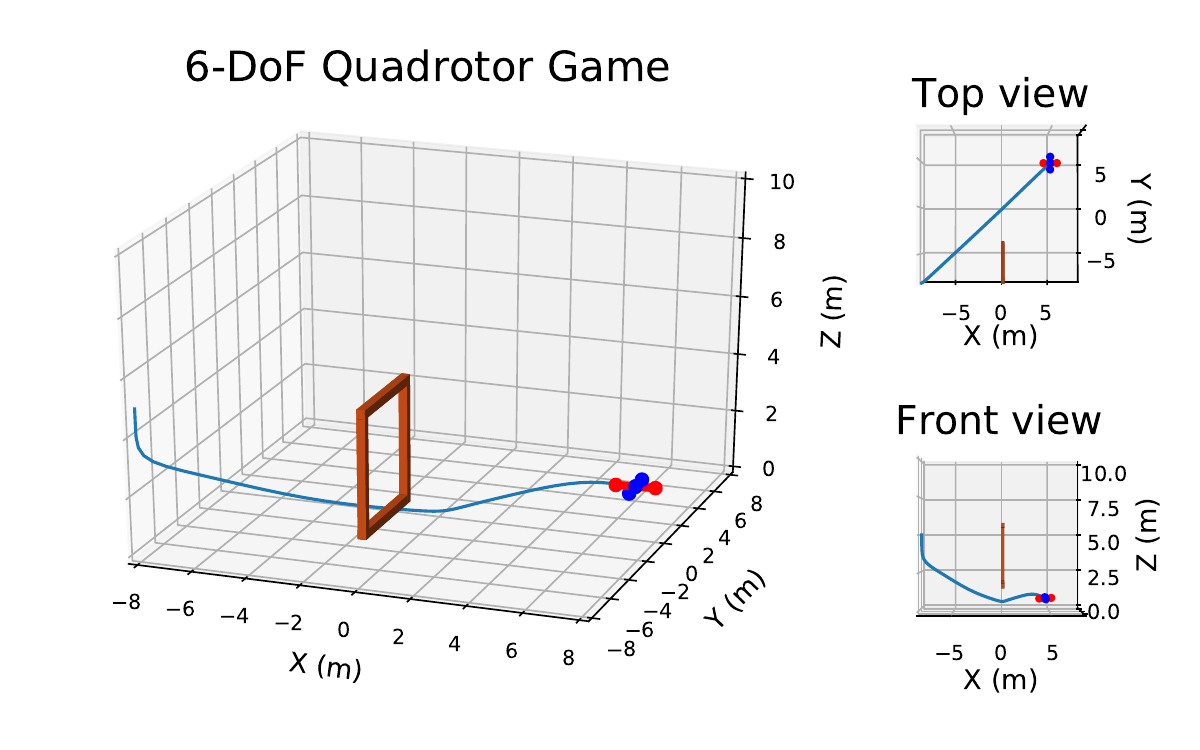}
	\caption{The 6-DoF quadrotor game. A human participant teaches the quadrotor to learn a   cost function by applying directional corrections through a keyboard (see the instruction in Table \ref{table_keycustomize_uav}). The goal is that  the quadrotor can  fly from the initial position (in the bottom left),  go through a  gate (colored in brown), and land on the target  (in the upper right).  Note that we have no knowledge about the gate.}
	\label{fig_uav_iteration0}
\end{figure}
\subsubsection{Robot Setup}
The dynamics of a quadrotor  is
\begin{subequations}\label{uav_dynamics}
	\begin{align}
	\dot{\boldsymbol{r }}_{I}&=\dot{\boldsymbol{v}}_{I},\\
	m\dot{\boldsymbol{v}}_{I}&=m\boldsymbol{g}_{I}+\boldsymbol{f}_{I},\\
	\dot{\boldsymbol{q}}_{B/I}&=\frac{1}{2}\Omega(\boldsymbol{\omega}_{B}){\boldsymbol{q}}_{B/I},\label{quaterdiff}\\
	J_{B}\dot{\boldsymbol{\omega}}_{B}&=\boldsymbol{\tau}_{B}-\boldsymbol{\omega}_{{B}}\times J_{B}\boldsymbol{\omega}_{B}.
	\end{align}
\end{subequations}
Here, the subscripts $_{B}$ and $_{I}$ denote the quantities  expressed in the quadrotor's body frame and world frame, respectively;  $m$ is the mass of the quadrotor; $\boldsymbol{r}_I\in\mathbb{R}^{3}$ and $\boldsymbol{v}_I\in\mathbb{R}^{3}$ are the position and velocity, respectively; $J_{B}\in\mathbb{R}^{3\times3}$ is the moment of inertia; $\boldsymbol{\omega}_{B}\in\mathbb{R}^{3}$ is the angular velocity; $\boldsymbol{q}_{B/I}\in\mathbb{R}^{4}$ is the unit quaternion \cite{kuipers1999quaternions} describing the attitude of the quadrotor's   body frame with respect to the world frame; (\ref{quaterdiff}) is the  time derivative of  the   quaternion  with $\Omega(\boldsymbol{\omega}_B)$ being the matrix form of $\boldsymbol{\omega}_{B}$  for quaternion multiplication \cite{kuipers1999quaternions};
$\boldsymbol{\tau}_{B}\in \mathbb{R}^{3}$ is the torque vector applied to the quadrotor; and $\boldsymbol{f}_{I}\in \mathbb{R}^{3}$ is the total force vector applied to the center of mass (COM). The total force magnitude $ \norm{\boldsymbol{f}_{I}}={ {f} }$ (along the z-axis of the quadrotor's  body frame) and the torque $\boldsymbol{\tau}_{B}=[\tau_x, \tau_y, \tau_z]\tran$ are generated by four rotor thrusts $[T_1, T_2, T_3, T_4]\tran$ as follows:

\begin{small}
	\begin{equation}
	\begin{bmatrix}
	{ {f}}\\
	\tau_x\\
	\tau_y\\
	\tau_z
	\end{bmatrix}=
	\begin{bmatrix}
	1&1 &1 &1\\
	0&-l_\text{w}/2& 0 & l_\text{w}/2 \\
	-l_\text{w}/2& 0 & l_\text{w}/2 & 0 \\
	\kappa&-\kappa&\kappa&-\kappa
	\end{bmatrix}
	\begin{bmatrix}
	T_1\\
	T_2\\
	T_3\\
	T_4
	\end{bmatrix},
	\end{equation}
\end{small}%
with $l_\text{w}$ denoting  the wing length of the quadrotor and $\kappa$ a fixed constant, here $\kappa=0.1$. In  dynamics   (\ref{uav_dynamics}), the gravity constant $\norm{\boldsymbol{g}_I}$ is set as $10$m/s$^2$ and the other  physical parameters  are units. We define the state  vector of the quadrotor as 
\begin{equation}
\boldsymbol{x}=\begin{bmatrix}
\boldsymbol{r}_I & \boldsymbol{v}_I  &\boldsymbol{q}_{B/I}  & \boldsymbol{\omega}_B 
\end{bmatrix} \in\mathbb{R}^{13},
\end{equation}
and  the control input vector as 
\begin{equation}
\boldsymbol{u}=\begin{bmatrix}
T_1 & T_2  &T_3 & T_4
\end{bmatrix}\tran \in\mathbb{R}^{4}.
\end{equation}
We discretize   (\ref{uav_dynamics})  by the Euler method with   interval $\Delta=0.1$s. To achieve  {$SE$(3)} maneuvering, we  define the attitude error between the quadrotor  attitude $\boldsymbol{q}$ and   target  $\boldsymbol{q}^{\text{target}}$ as \cite{lee2010geometric}
\begin{equation}
e(\boldsymbol{q},\boldsymbol{q}_{\text{target}})=\small{\frac{1}{2}} \text{trace}(I-R\tran(\boldsymbol{q}^{\text{target}})R(\boldsymbol{q})),
\end{equation}
where $R(\boldsymbol{q})\in\mathbb{R}^{3\times3}$ is the  rotation matrix  corresponding to~$\boldsymbol{q}$. 

\smallskip
In the cost function    (\ref{equ_objective}),  we  set the   final cost  term as
\begin{align}
h({\boldsymbol{x}_{T+1}})=\,&\norm{\boldsymbol{r}_I-\boldsymbol{r}_I^{\text{target}}}^2+10\norm{\boldsymbol{v}_I}^2+\nonumber\\
&100e(\boldsymbol{q}_{B/I},\boldsymbol{q}_{B/I}^{\text{target}})+10\norm{\boldsymbol{w}_B}^2, \label{uav_objective.2}
\end{align}
because the quadrotor aims to finally land on a target position  $\boldsymbol{r}_{I}^\text{target}$ with a target attitude $\boldsymbol{q}_{B/I}^{\text{target}}$. Here, $\boldsymbol{r}_{I}=[r_x,r_y,r_z]\tran$ is the position of the quadrotor expressed in the  world  frame. We set the weight-feature  cost term as
\begin{subequations}\label{uav_objective}
	\begin{align}
	\boldsymbol{\phi}&=\begin{bmatrix}
	r_{x}^2 & r_{x} & r_{y}^2& r_{y}&  
	r_{z}^2& r_{z}& \norm{\boldsymbol{u}}^2
	\end{bmatrix}\tran\in\mathbb{R}^7,\label{uav_objective.1}\\
	\boldsymbol{\theta}&=\begin{bmatrix}
	\theta_1 & \theta_2 &\theta_3 & \theta_4 &\theta_5 & \theta_6  &\theta_7
	\end{bmatrix}\tran\in\mathbb{R}^7.
	\end{align}
\end{subequations}
Here,   feature vector $\boldsymbol{\phi}$ consists of  general
polynomial features and     $\boldsymbol{\theta}$   will  determine the specific  trajectory of the quadrotor to the target.   By default,  the initial weight search space is
\begin{multline}\label{exp_uav_initialcondition}
\boldsymbol{\Omega}_0{=}\{\boldsymbol{{\theta}}\,|\,
\theta_1, \theta_3, \theta _5\in[0,1], \\ \theta_2,\theta_4, \theta_6\in[-8,8], \theta_7\in[0, 0.5]
\}.
\end{multline}

As  in Fig. \ref{fig_uav_iteration0}, the goal of this  quadrotor  game  is  to let a human participant  teach the quadrotor to learn a   cost function, such that the quadrotor can successfully fly from the initial position $\boldsymbol{r}_{I}(0)=[-8, -8, 5]\tran$ (in the bottom left),  go through a  gate (in brown color), and finally land on a  target position $\boldsymbol{r}_{I}^{\text{target}}=[8, 8, 0]\tran$ (in the upper right) with the target attitude $\boldsymbol{q}_{B/I}^{\text{target}}=[1, 0, 0, 0]\tran$. The initial    attitude of the quadrotor is $\boldsymbol{q}_{B/I}(0)=[1, 0, 0, 0]\tran$ and initial velocities  zeros. The time horizon for this game is  $T=50$, that is, $T\Delta=5$s. Note that we have no knowledge about the gate.

\begin{table}[h]
	\centering
	\caption {Keyboard interface for   6-DoF quadrotor game.}
	\begin{tabular}{lll}
		\toprule
		Keys & Directional correction &  Interpretation of correction  \\
		\midrule
		
		\emph{`up'} &   \begin{tabular}{@{}c@{}}$T_1{=}1, T_2{=}1,$\\
			$T_3{=}1, T_4{=}1$	
		\end{tabular}  & Upward force  applied at COM  \\[8pt]
		
		\emph{`down'} &   \begin{tabular}{@{}c@{}}$T_1{=}{-}1, T_2{=}{-}1,$\\
			$T_3{=}{-}1, T_4{=}{-}1$	
		\end{tabular} &   Downward force applied at COM  \\[8pt]
		
		\emph{`w'} &  
		\begin{tabular}{@{}l@{}} $T_1{=}0, T_2{=}1$\\
			$ T_3{=}0,  T_4{=}{-}1$	
		\end{tabular}
		& Negative torque along body-axis x \\[8pt]
		
		\emph{`s'} &   		\begin{tabular}{@{}l@{}} $ T_1{=}0, T_2{=}{-}1$\\
			$ T_3{=}0,  T_4{=}1$	
		\end{tabular} & Positive torque along body-axis x \\[8pt]
		
		\emph{`a'} & \begin{tabular}{@{}l@{}} $ T_1{=}1,  T_2{=}0 $\\
			$T_3{=}{-}1, T_4{=}0$	
		\end{tabular}   & Negative torque along body-axis y \\[8pt]
		
		\emph{`d'} &  \begin{tabular}{@{}l@{}}  $ T_1{=}{-}1,  T_2{=}0$ \\
			$T_3{=}1, T_4{=}0$	
		\end{tabular}  & Positive torque along body-axis y \\[4pt]

		\bottomrule
	\end{tabular}
	\begin{tablenotes}
		\item[1] \textsuperscript{\textcolor{red}{*}}  When implementing  \cite{jain2015learning,bajcsy2017learning},
		to allow   magnitude corrections via the  same interface, we  additionally  detect  key pressing duration  $\delta t$  and interpret the  correction magnitude $u$ to be  $u=\min\{\beta \delta t, u_{\max}\}$ with $\beta{=}2$ and $u_{\max}{=}1$ in our implementation.  Thus, the magnitude correction using the above interface is $\boldsymbol{u}=u[T_1, T_2, T_3, T_4]\tran$. The correction time $t_k$ is the beginning time of the key press.

	\end{tablenotes}
	\label{table_keycustomize_uav}
\end{table}

\subsubsection{Interface} \label{section.game.quadrotor}
The  keyboard interface for applying directional corrections is in Table~\ref{table_keycustomize_uav}. Specifically, we customize the (`\emph{up}', `\emph{down}', `\emph{w}', `\emph{s}', `\emph{a}', `\emph{d}') keys and map them to specific directional correction signals. During a learning iteration, a human participant is allowed to press one or multiple  keys  in Table \ref{table_keycustomize_uav}. The interface listens to  the keystrokes and translates the keystrokes into the directional corrections based on Table \ref{table_keycustomize_uav}. The time  step at which  a key is hit is  the correction time  $t_k$.  For example,  if a human participant presses \emph{`s'} key at  time step $5$;  then, the   directional correction will be $\boldsymbol{a}_{t_k}=[0, -1, 0, 1]\tran$  with  $t_k=5$.  Based on  (\ref{equ_aug_correction_vec}) and (\ref{equ_assumption_modify.abar}), we  obtain $\boldsymbol{\bar{a}}_{k}$ from $\boldsymbol{a}_{t_k}$.

\subsection{Participants and Procedure} \label{game.participant}

A total of 17 volunteers from Purdue College of Engineering have participated in our user study. Among these 17 participants,  1 was female and 16  male, with ages from 20 to 37 years old ($25.812\pm3.936$). 5 of those  17 participants had no robotics or related background, and all participants were novices to the two games. This user study had been reviewed and approved by the  Institutional Review Board (IRB) of  Purdue University, and all participants had signed the consent forms.

Each participant was instructed to play the above two games. In each game, a participant played 5 successive rounds with the proposed method and 5  successive rounds with the magnitude-correction method \cite{jain2015learning,bajcsy2017learning}. \emph{A game round} is defined as a complete run of a learning method until the success or failure of the round. A failure of a game round is identified if the robot fails to accomplish the task after the total number of human corrections has exceeded the maximum correction count. In our user study, the maximum correction count is 20 for the robot arm game and 15 for the quadrotor game.  We set the maximum correction count because the participants are expected to try their best to teach the robot with as few corrections as possible.

We  used a within-subjects design: the order of the two games (i.e., which game goes first) and learning methods used (i.e., which method is used first) were counterbalanced across all participants in our user study. Also, the participants were not informed which methods were being used for the current game. Before starting a game, each participant was instructed how to play the game and was given one hands-on game round to get familiar with the interface before the experiment recording starts. After the games, each participant was asked to finish a  post-game survey about her/his  opinions of the game experience (which will be reported in  Section \ref{userstudy.analysis.others}).

\begin{figure*}[h]
	\centering
	\begin{subfigure}{.31\textwidth}
		\centering
		\includegraphics[width=\linewidth]{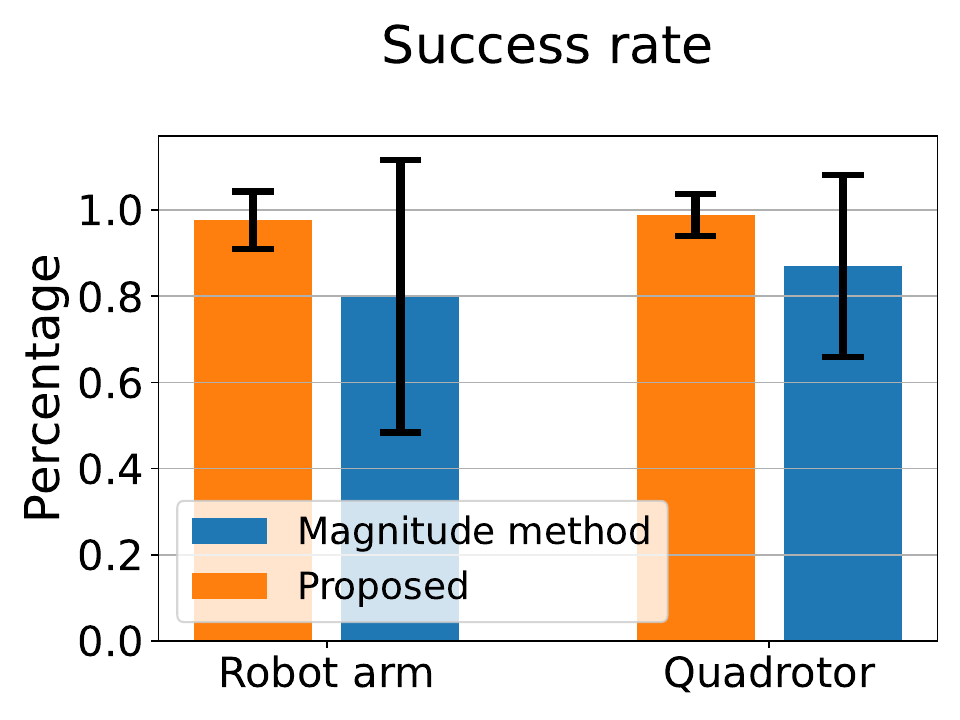}
		\caption{Success rate}
		\label{userstudy.result.1}
	\end{subfigure}
	\hfill
	\begin{subfigure}{.31\textwidth}
		\centering
		\includegraphics[width=\linewidth]{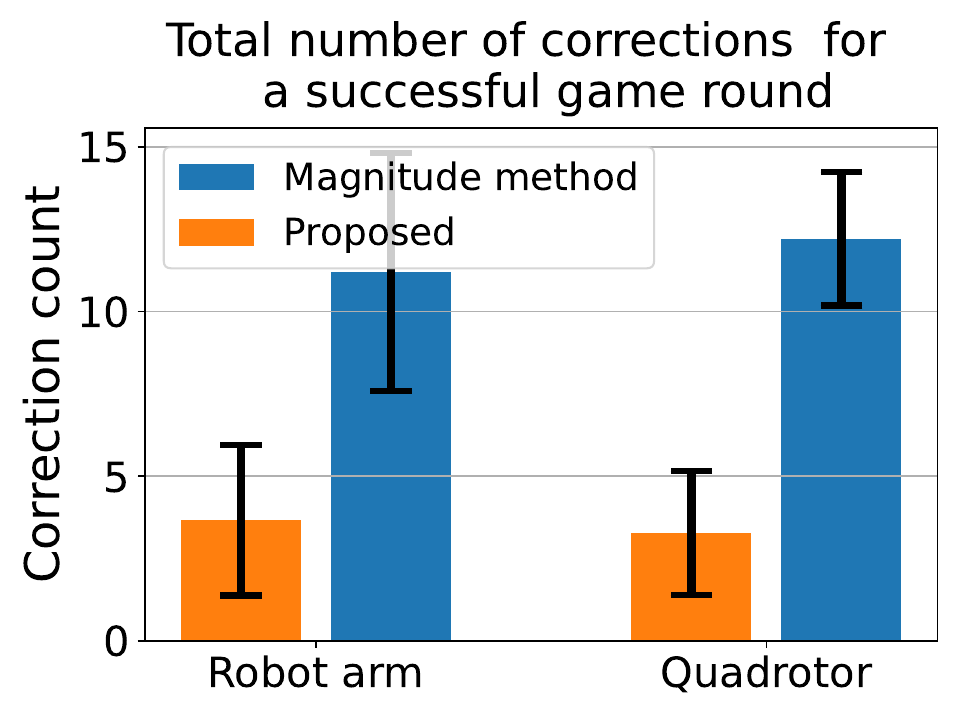}
		\caption{Total  correction count in a success game}
		\label{userstudy.result.2}
	\end{subfigure}
	\hfill
	\begin{subfigure}{.31\textwidth}
		\centering
		\includegraphics[width=\linewidth]{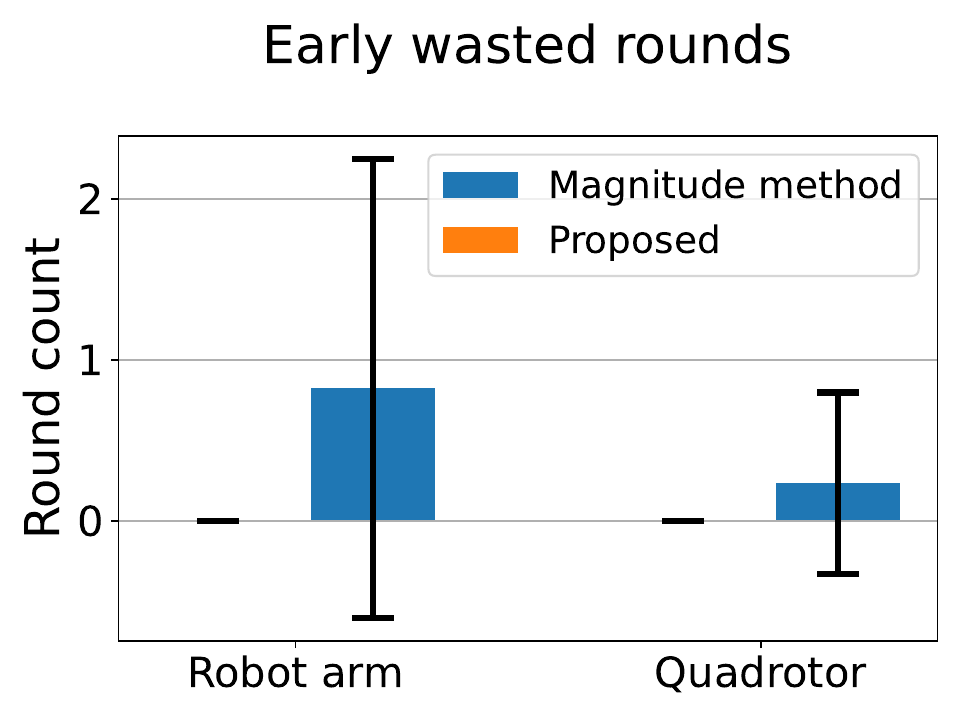}
		\caption{Early wasted rounds}
		\label{userstudy.result.3}
	\end{subfigure}
	\caption{The outcomes of all 17 participants for
playing different robot games (robot-arm or quadrotor)		
		 using different learning methods (the proposed method and magnitude-correction  method \cite{jain2015learning,bajcsy2017learning}). (a) shows the success rate, i.e., the ratio of successful game rounds over total 5  rounds attempted; (b) shows the total number of human corrections needed for one successful game round; and (c) shows the  number of the early wasted (failed) game rounds before a participant constantly play successfully. The bars denote the mean values over all participants, and the top line segments are the
		standard deviations.  Analysis  based on  statistical tests (with significance level $\alpha=0.05$) is   in  Section \ref{game.resuts}.  One illustrative  round for the  robot arm  game from one participant is in Fig. \ref{exp_fig_arm_correction} and Table \ref{table_robotarm_results}, and one illustrative quadrotor game round from a participant is in Fig. \ref{exp_fig_uav_correction} and Table~\ref{table_uav_results}.
	} 
	\label{userstudy.result}
\end{figure*}

\subsection{Outcome Measurements} \label{game.measurement}

For each game (robot arm or quadrotor) with each learning method (the proposed method or magnitude-correction method \cite{jain2015learning,bajcsy2017learning}), we measure the following outcomes for one participant in his/her 5-successive game rounds.

\begin{itemize}
	\item \emph{Success rate}. This is the ratio of successful game rounds over all rounds attempted  (which is 5 here).  This outcome indicates
	the \textbf{efficacy} of a learning method. 	
	
	\smallskip

	\item \emph{Total number of corrections for a successful game round.} This measures how many human corrections are needed for a successful game round. This measure can show the \textbf{efficiency}/\textbf{effortlessness}  of a learning method---fewer corrections mean less human effort in teaching a robot. 
	
	\smallskip

	\item \emph{Early wasted rounds.} This is to measure how many early rounds of a game are wasted (failed) before a participant begins to constantly play successful game rounds. This can show the \textbf{accessibility}  of a learning method---fewer early wasted rounds indicate a  more accessible experience for a participant to successfully teach a robot.
\end{itemize}

\noindent
The outcomes of all participants are organized according to the method type  (the proposed method and magnitude-correction method \cite{jain2015learning,bajcsy2017learning}) and the game type (robot-arm and quadrotor). We report all outcomes in Fig. \ref{userstudy.result} and provide a detailed analysis of the results based on statistical tests.

\subsection{Results and Analysis} \label{game.resuts}

The statistics of three outcomes over all 17 participants are shown in Fig. \ref{userstudy.result}. Each outcome is presented with respect to each learning method (the proposed method or magnitude-correction method \cite{jain2015learning,bajcsy2017learning}) on each game (robot arm or quadrotor). All three outcomes here are averaged over all participants with the error bar denoting the standard deviation over all participants.  We have the following analysis of the  results based on  statistical tests with significance level $\alpha=0.05$,

\begin{figure*}[h]
	\centering
	\begin{subfigure}{.22\textwidth}
		\centering
		\includegraphics[width=\linewidth]{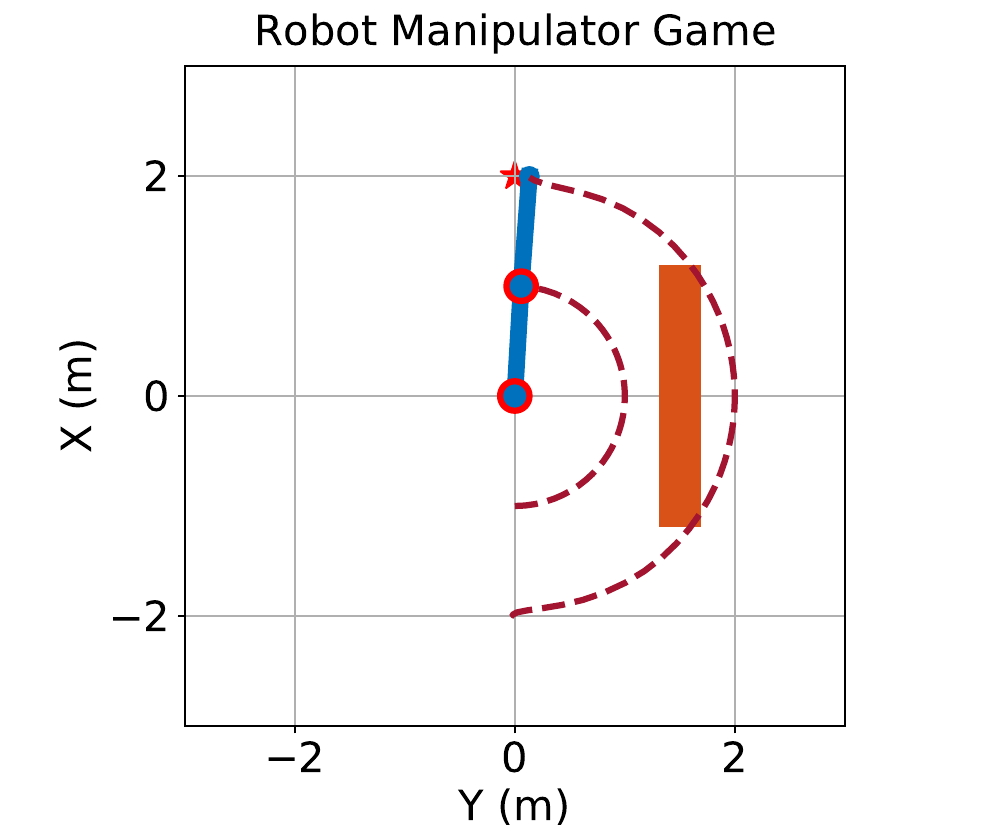}
		\caption{Iteration $k=1$}
		\label{exp_fig_arm_correction.1}
	\end{subfigure} 
	\begin{subfigure}{.22\textwidth}
		\centering
		\includegraphics[width=\linewidth]{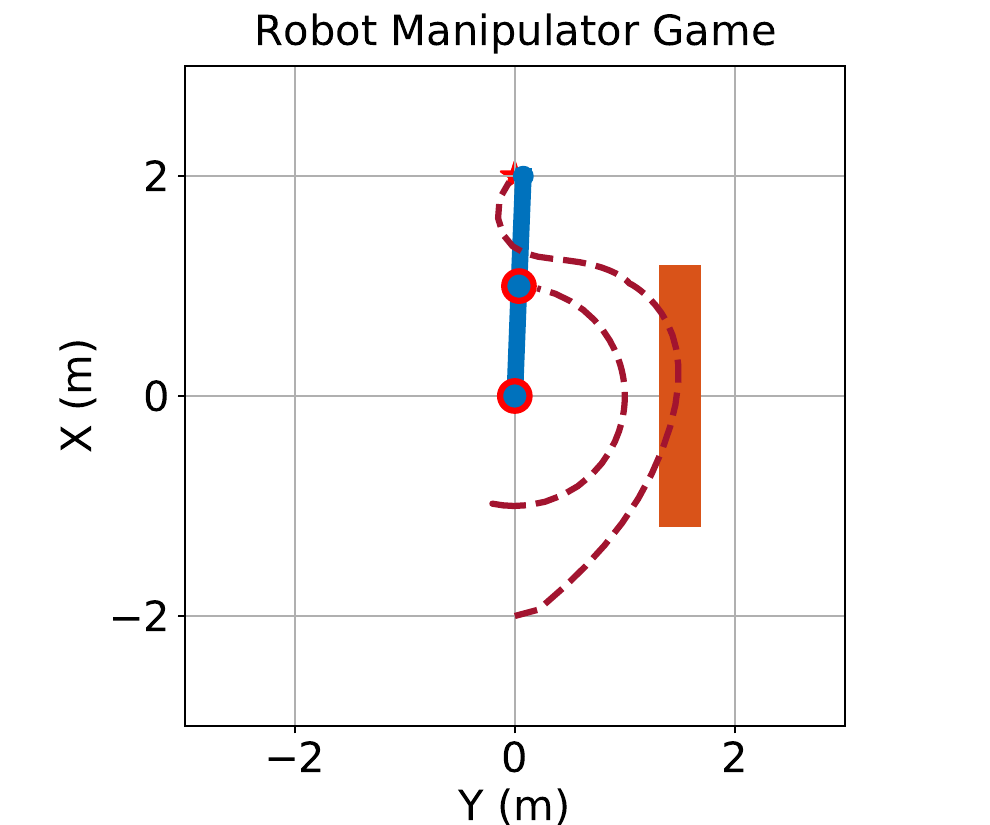}
		\caption{Iteration $k=2$}
		\label{exp_fig_arm_correction.2}
	\end{subfigure}
	\begin{subfigure}{.22\textwidth}
		\centering
		\includegraphics[width=\linewidth]{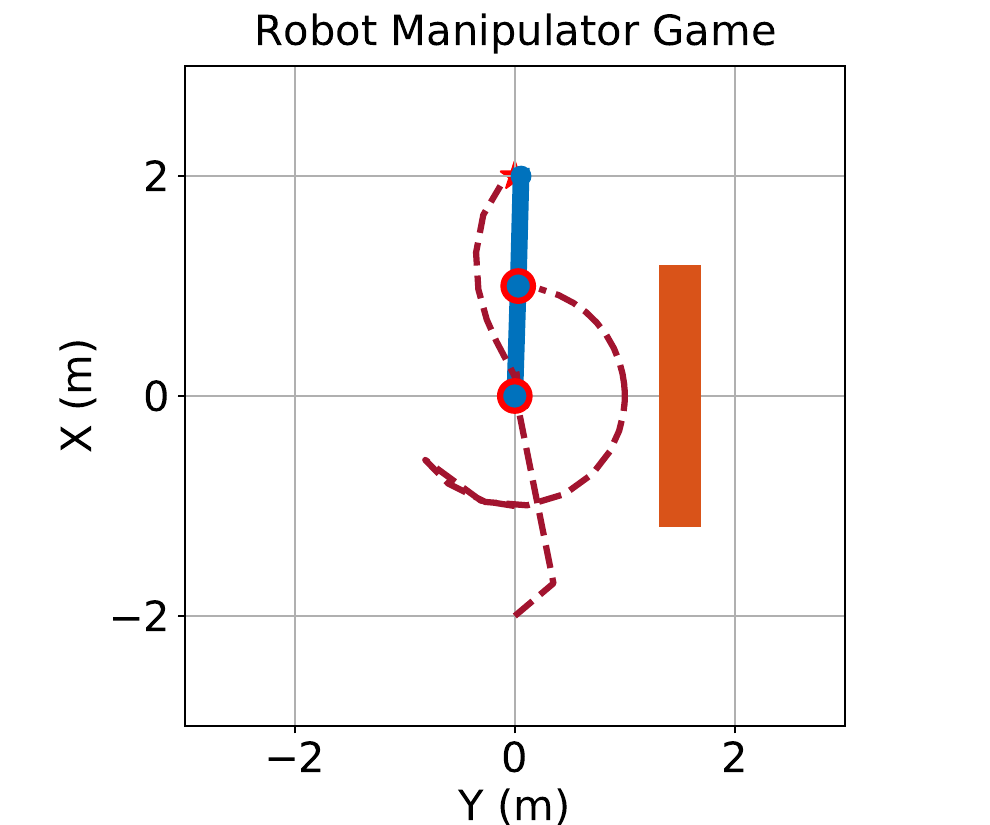}
		\caption{Iteration $k=3$}
		\label{exp_fig_arm_correction.3}
	\end{subfigure}
	\begin{subfigure}{.22\textwidth}
		\centering
		\includegraphics[width=\linewidth]{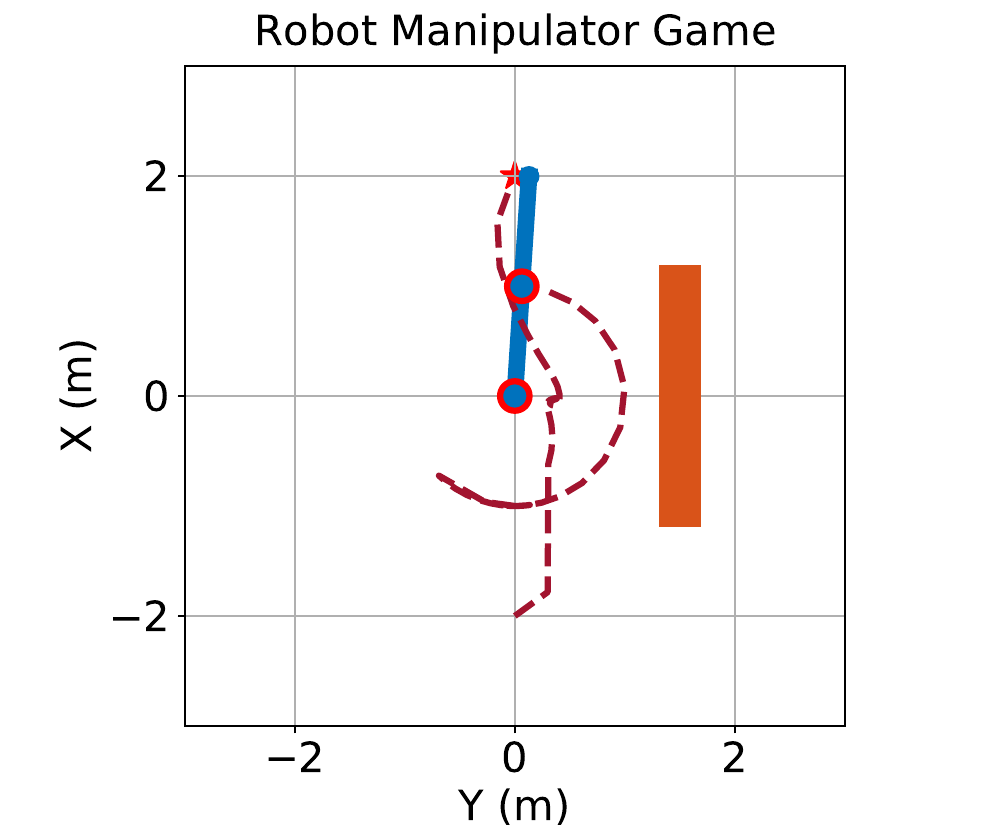}
		\caption{Iteration $k=4$}
		\label{exp_fig_arm_correction.4}
	\end{subfigure}
	\caption{One illustrative    round for the robot arm game from one participant.  At each   iteration~$k$, the robot  weight vector guess $\boldsymbol{{\theta}}_k$ and the participant directional correction $\boldsymbol{a}_{t_k}$  are reported in Table \ref{table_robotarm_results}. 
	} 
	\label{exp_fig_arm_correction}
\end{figure*}

\smallskip
\begin{table*}[h]
	\centering
	\caption	{An illustrative  round for the robot arm game from one participant.}
	\begin{tabular}{ccc}
		\toprule
		Iteration $k$   & Current weight vector guess $\boldsymbol{\theta}_k$ & A participant's directional correction $\boldsymbol{a}_{t_k}$ and correction time ${t_k}$ \\[3pt]
		\midrule
		$k=1$ & $\boldsymbol{\theta}_k=[0.50,\, 0.00, \,0.50, -0.00,\, 0.25]\tran$  & $\boldsymbol{a}_{t_k}=[0, 1]$ (i.e., \emph{left} key pressed) \quad and \quad ${t_k}=11$
		\\[3pt]
		$k=2$ & $\boldsymbol{\theta}_k=[0.50,\, 0.00, \,0.50, -1.50, \,0.25]\tran$  & $\boldsymbol{a}_{t_k}=[0, 1]$ (i.e., \emph{left} key pressed) \quad and  \quad${t_k}=16$\\[3pt]
		$k=3$ & $\boldsymbol{\theta}_k=[0.50,\, 0.00, \,0.34, -2.03, \,0.25]\tran$  & $\boldsymbol{a}_{t_k}=[-1, 0]$ (i.e., \emph{down} key pressed) \quad and\quad ${t_k}=34$\\[3pt]
		$k=4$ & $\boldsymbol{\theta}_k=[0.50,\, 1.48, \,0.36, -2.00, \,0.25]\tran$  & \emph{ Game success! }\\
		\bottomrule
	\end{tabular}
	\label{table_robotarm_results}
\end{table*}

\begin{figure*}
	\centering
	\begin{subfigure}{.248\textwidth}
		\centering
		\includegraphics[width=\linewidth]{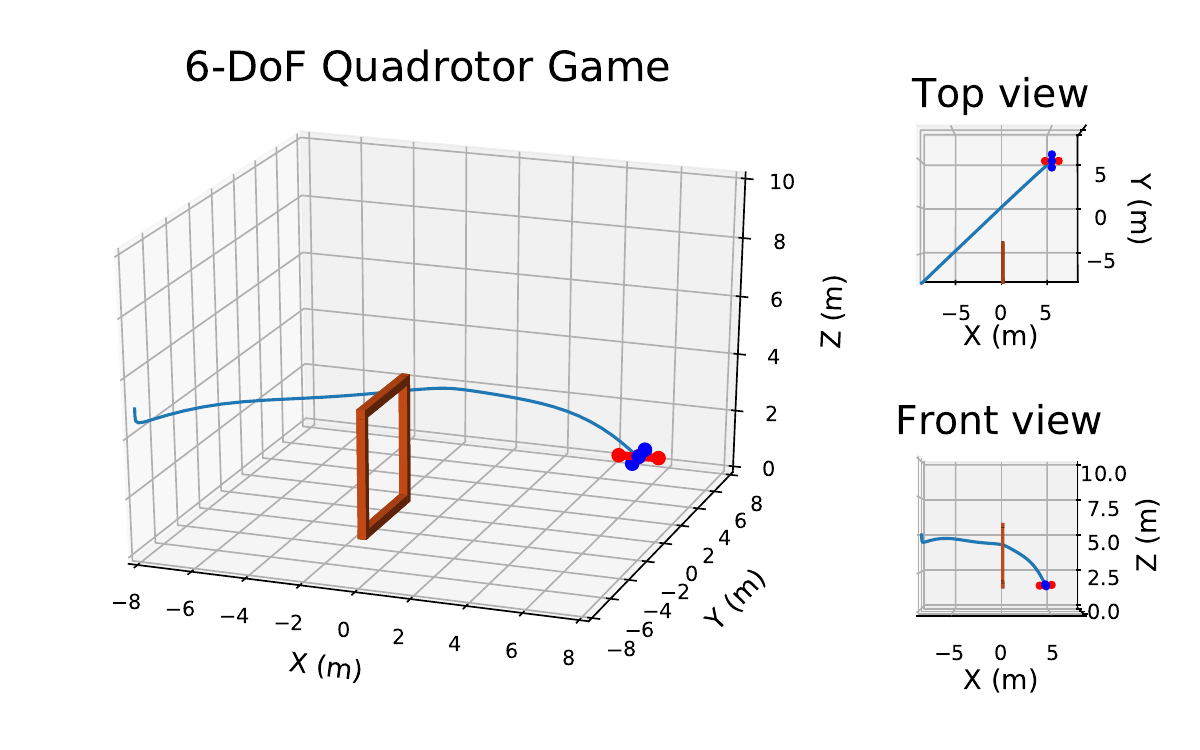}
		\caption{Iteration $k=2$}
		\label{exp_fig_uav_correction.1}
	\end{subfigure}
	\hfill
	\begin{subfigure}{.248\textwidth}
		\centering
		\includegraphics[width=\linewidth]{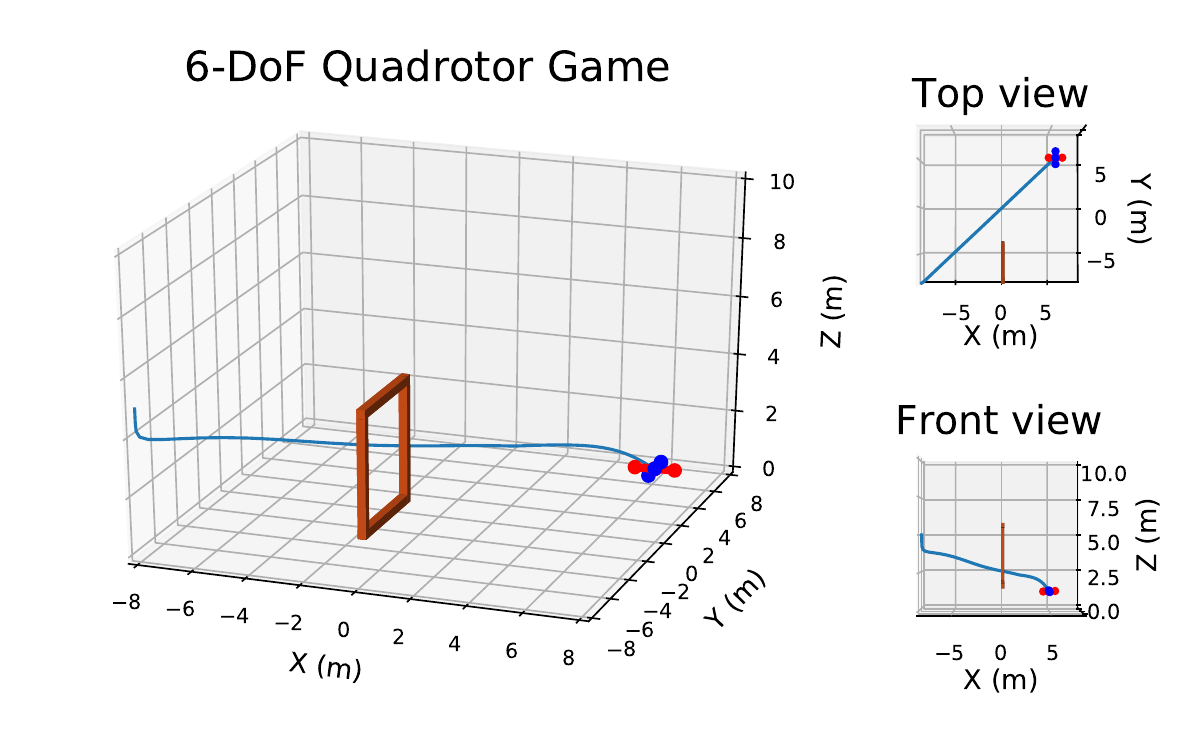}
		\caption{Iteration $k=3$}
		\label{exp_fig_uav_correction.2}
	\end{subfigure}\hfill
	\begin{subfigure}{.248\textwidth}
		\centering
		\includegraphics[width=\linewidth]{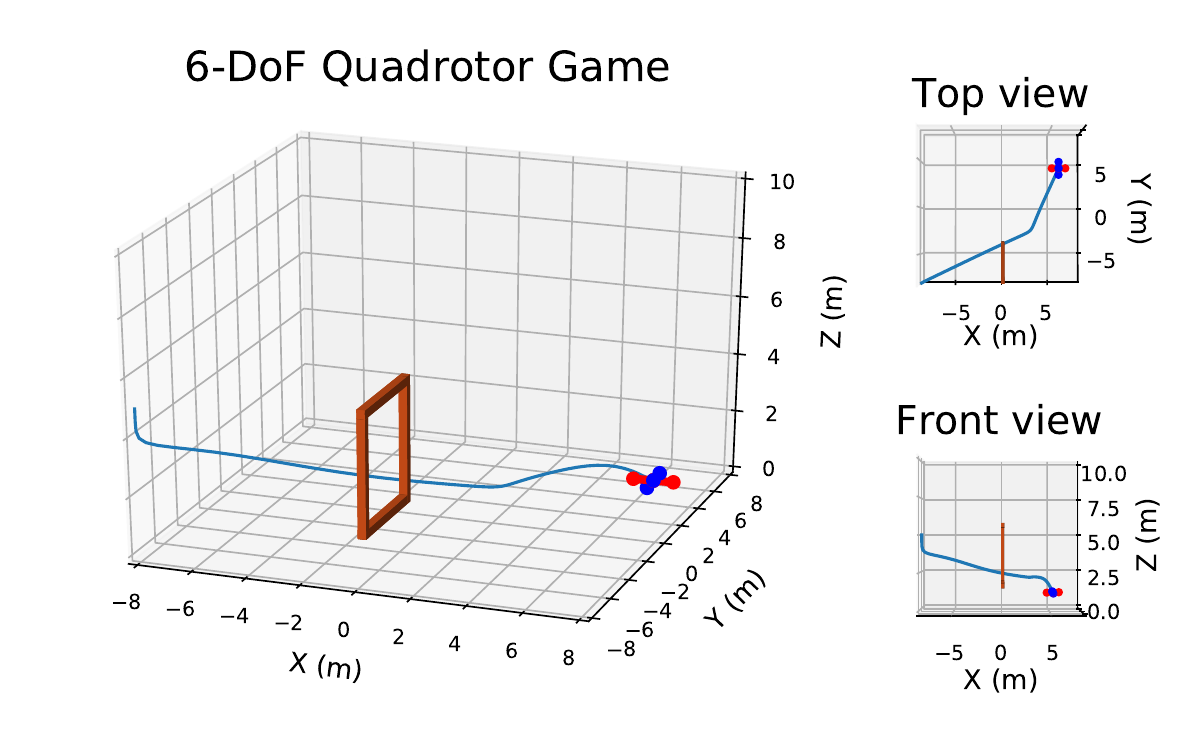}
		\caption{Iteration $k=4$}
		\label{exp_fig_uav_correction.3}
	\end{subfigure}\hfill
	\begin{subfigure}{.248\textwidth}
		\includegraphics[width=\linewidth]{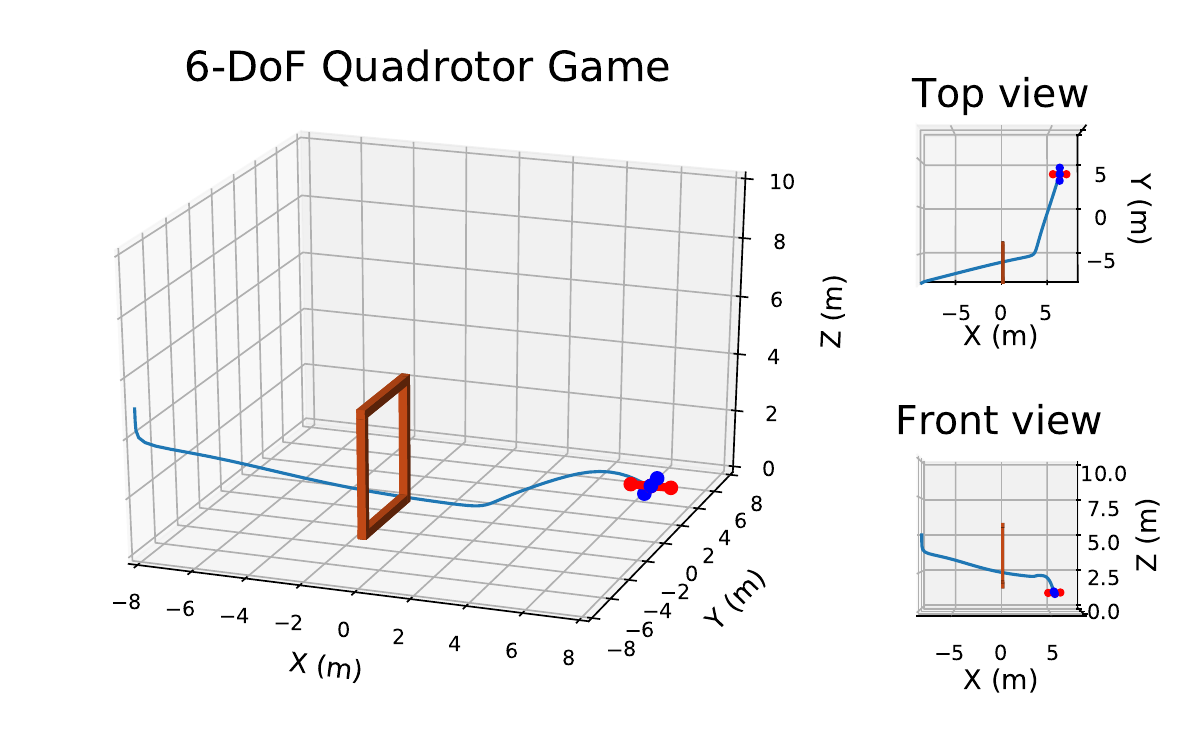}
		\caption{Iteration $k=5$}
		\label{exp_fig_uav_correction.5}
	\end{subfigure}
	\caption{One illustrative  round for the quadrotor game from one participant. At each iteration~$k$, the quadrotor weight vector  guess $\boldsymbol{{\theta}}_k$ and the participant directional correction $\boldsymbol{a}_{t_k}$   are  in Table  \ref{table_uav_results}. Iteration $k=1$ is shown in Fig. \ref{fig_uav_iteration0}.
	} 
	\label{exp_fig_uav_correction}
\end{figure*}

\begin{table*}[h]
	\centering
	\caption	{An illustrative  round for the quadrotor game from one participant (Iteration $k=1$ is shown in Fig. \ref{fig_uav_iteration0}).}
	\begin{tabular}{ccc}
		\toprule
		Iteration $k$   & Current weight vector guess $\boldsymbol{\theta}_k$ &  A participant's directional corrections $\boldsymbol{a}_{t_k}$ and  correction time ${t_k}$ \\
		\midrule
		$k=1$ & $\boldsymbol{\theta}_k=[0.50,\, 0.00, \,0.50, -0.00, \,0.50, \, -0.00,\, 0.25]\tran$  & 
		\begin{tabular}{@{}c@{}} $\boldsymbol{a}_{t_k}=[1, 1, 1, 1]$ (i.e., \emph{up} key pressed) \quad and \quad \,\,\,${t_k}=8$\\
			$\boldsymbol{a}_{t_k}=[1, 1, 1, 1]$ (i.e., \emph{up} key pressed) \quad and \quad ${t_k}=20$
		\end{tabular}
		\\[6pt]

		$k=2$ & $\boldsymbol{\theta}_k=[0.50,\, -0.00, \,0.50, -0.00, \,0.50, \, -3.99,\, 0.25]\tran$  & $\boldsymbol{a}_{t_k}=[-1, -1, -1, -1]$ (i.e., \emph{down} key pressed) \quad and \quad ${t_k}=14$\\[4pt]

		$k=3$ & $\boldsymbol{\theta}_k=[0.50,\, -1.70, \,0.50, -1.70, \,0.52, \, -1.89,\, 0.25]\tran$   & $\boldsymbol{a}_{t_k}=[ 0, -1,  0,  1]$ (i.e., \emph{'s'} key pressed) \quad and ${t_k}=13$\\[4pt]

		$k=4$ & $\boldsymbol{\theta}_k=[0.50,\, -2.76, \,0.50, -2.45, \,0.60, \, -2.22,\, 0.25]\tran$  & $\boldsymbol{a}_{t_k}=[ 0, -1,  0,  1]$ (i.e., \emph{'s'} key pressed) \quad and ${t_k}=19$\\[4pt]
		
		$k=5$ & $\boldsymbol{\theta}_k=[0.50,\, -3.11, \,0.50,\, 4.89, \,0.65 \, -2.67,\, 0.25]\tran$  & \emph{ Game success!  \quad }\\[4pt]
		
		\bottomrule
	\end{tabular}
	\label{table_uav_results}
\end{table*}

\subsubsection{Success Rate}
In Fig. \ref{userstudy.result.1}, for the robot arm game,  the proposed method obtains a success rate of $97.65\%\pm6.64\%$  compared to the  magnitude-correction method \cite{jain2015learning,bajcsy2017learning}'s $80.00\%\pm31.62\%$. The proposed method yields a significantly higher success rate than the magnitude-correction method with  $t$-score  $t=2.25$ and $p$-value $p=0.031<\alpha$.  For the quadrotor game in Fig. \ref{userstudy.result.1}, the proposed method attains  $94\%\pm12.8\%$ versus the  magnitude-correction method's $82\%\pm20.9\%$. It also shows a significantly higher success rate than the magnitude-correction method \cite{jain2015learning,bajcsy2017learning}, with  $t$-score $t=2.24$ and $p$-value $p=0.032\leq\alpha$. 

We provide the following comments for the  above results. Recall that the  magnitude-correction method \cite{jain2015learning,bajcsy2017learning}  assumes $\small J(\boldsymbol{u}_{0:T}^{\boldsymbol{{\theta}}_k}+\boldsymbol{\bar{a}}_k,\boldsymbol{\theta}^*){<} J(\boldsymbol{u}_{0:T}^{\boldsymbol{{\theta}}_k},\boldsymbol{\theta}^*)$, while the our method   assumes $\langle-\nabla J(\boldsymbol{u}_{0:T}^{\boldsymbol{\theta}_k},\boldsymbol{\theta}^*), \boldsymbol{\bar{a}}_k \rangle{>}0$, regardless of  magnitude $\norm{\boldsymbol{\bar{a}_k}}$. As shown in Fig.  \ref{figure_correctionvs}, \emph{the allowable region of  directional corrections  is much larger than that of   magnitude corrections}. Since all participants  are   novices to both games,  giving valid magnitude correction could be difficult for them.   By contrast,  novice participants  are more likely to give  valid directional corrections. This leads to  higher success rate and lower variance for the proposed  method than the magnitude-correction method.

The difficulty of giving valid magnitude corrections can also be seen from the number of early wasted/failed rounds in Fig. \ref{userstudy.result.3}. For the magnitude method, before a participant masters a game,  $0.82\pm 1.42$ early game rounds  are wasted/failed for the robot arm game, and $0.24\pm 0.56$ early rounds are failed for the quadrotor game. In contrast, the proposed method has $0\pm0$ early wasted rounds in both games.  This means that  a novice participant needs more practice  to master the skill of giving valid magnitude corrections.

\subsubsection{Total Number of Corrections  for  Successful Game}
In  Fig. \ref{userstudy.result.2}, for a successful robot arm game,  the proposed  method needs a total of  $3.66\pm2.29$  corrections, while the magnitude-correction method   \cite{jain2015learning,bajcsy2017learning} requires $11.21\pm3.62$   corrections. Thus, the proposed method requires significantly fewer  corrections than the magnitude-correction method, with $t$-score  $t=7.27$ and $p$-value $p=2.94\times10^{-8}\leq \alpha$. In the quadrotor game, the proposed  method takes a total of $3.27\pm1.88$   corrections for a successful  round, while  the magnitude-correction method  needs  $12.20\pm2.02$ corrections. This again shows that the proposed method is significantly more efficient (fewer corrections) than the magnitude-correction method, with $t$-score  $t=13.32$ and $p$-value $p=1.35\times10^{-14}<\alpha$.

We have the following interpretations for the above results. Recall that the  assumption for a valid magnitude correction $\boldsymbol{\bar{a}}_k$  is  $ J(\boldsymbol{u}_{0:T}^{\boldsymbol{{\theta}}_k}+\boldsymbol{\bar{a}}_k,\boldsymbol{\theta}^*){<} J(\boldsymbol{u}_{0:T}^{\boldsymbol{{\theta}}_k},\boldsymbol{\theta}^*)$. Since all participants are novice to the games, the magnitude corrections made by a participant were not always valid in a game round; in other words, some magnitude corrections can be over-correction. Because of those  invalid corrections, the magnitude-correction method takes longer (more corrections) to succeed. In contrast,  the proposed method   only leverages   directional corrections, which are  more likely to  satisfy $\langle-\nabla J(\boldsymbol{u}_{0:T}^{\boldsymbol{\theta}_k},\boldsymbol{\theta}^*), \boldsymbol{\bar{a}}_k \rangle{>}0$; it requires   fewer human corrections (effort).

Since the total number of corrections for a successful game indicates how much effort a participant needs to  successfully teach a  robot. Thus, the above results in Fig. \ref{userstudy.result.2} show that the proposed method is significantly more effortless than the magnitude-correction method  \cite{jain2015learning,bajcsy2017learning}.

\subsubsection{Early Wasted Rounds}
In  Fig. \ref{userstudy.result.3}, in the robot arm game (5  rounds in total), the early wasted (failed) rounds  is  $0\pm 0$ for the proposed method and  $0.82\pm 1.42$ for the magnitude-correction method \cite{jain2015learning,bajcsy2017learning};  thus, the proposed method requires  significantly fewer rounds of practice (thus is more accessible) than the magnitude-correction method, with  $t$-score  $t=2.38$ and $p$-value $p=0.02<\alpha$. For the quadrotor game (5 rounds in total), a participant wasted $0\pm 0$ early rounds using the proposed method, while $0.24\pm 0.56$ early rounds using the magnitude-correction method.  The early wasted rounds of the proposed method are not significantly lower than that of the magnitude-correction method, with     $t$-score  $t=1.73$ and $p$-value $p=0.09>\alpha$.

We have the following explanations for the above results. 
For the magnitude-correction method,  since not all magnitude corrections are valid, a participant needs more early trials to realize the effects of correction magnitude on the robot motion and how to apply valid  magnitude. Instead, for the proposed  method, a participant   only gives  directional commands, which can be more intuitive, as shown in Fig. \ref{exp_fig_arm_correction} and Fig. \ref{exp_fig_uav_correction}. 
Fig. \ref{userstudy.result.3}  shows that the proposed  method is potentially more accessible than the magnitude-correction method \cite{jain2015learning,bajcsy2017learning}

\subsubsection{Other Observations and Subjective Survey} \label{userstudy.analysis.others}
We also observe  that the number of the
directional corrections in the user study is generally smaller than that in the simulation cases in Section \ref{section.simulation}. This is because the convergence of the robot motion trajectory (to the desired motion in the human's perspective) is empirically faster than the convergence of the cost function weight vector (to the true $\boldsymbol{\theta}^*$). This has been shown in Fig. \ref{fig_compare}, where   the convergence of  $\norm{\boldsymbol{\theta}_k-\boldsymbol{\theta}^*}^2$  requires around 20 iterations, while the convergence of  $\norm{\boldsymbol{\xi}_{\boldsymbol{\theta}_k}-\boldsymbol{\xi}_{\boldsymbol{\theta}^*}}^2$ requires only 10 iterations. Since  a game is deemed successful as long as the robot motion can successfully navigate through the environment, thus a   participant usually takes fewer corrections. 
We also find that there is no unique solution for the choice of directional corrections and that different participants applied different corrections to manage the games.

We have also performed a  post-game survey about how participants perceived the two   learning methods. 10 out of 17 participants said that they felt \emph{the robot with the proposed method is much `smarter'  than the robot with the magnitude-correction method  \cite{jain2015learning,bajcsy2017learning}.} Those  opinions are also consistent with the  results in  Fig. \ref{userstudy.result}, as we have analyzed above.

\smallskip

\subsubsection{Summary} In summary of all  statistical  analysis above, we conclude that compared to the state-of-the-art magnitude-correction method \cite{jain2015learning,bajcsy2017learning},    the proposed method is
\begin{itemize}
	\item  \emph{significantly more effective} ($p< 0.033$) --- it attains  higher success rate for teaching robots,
	\item  \emph{significantly  more efficient} ($p< 10^{-7}$) --- fewer human  corrections  are needed  for  successful robot learning,
	\item  \emph{potentially more accessible} ($p\leq 0.09$) for  human users --- a novice human user  takes fewer trials before constantly providing   successful corrections to a robot. 
\end{itemize}
The post-game survey reports that a majority of the user-study participants indicated that a robot with the proposed method is much `smarter' than with the magnitude-correction method.

\section{Real-world Experiment}\label{section.real_experiments}
In this section, we test the proposed method in a real-world experiment using a Parrot Mambo quadrotor.

\begin{figure}[h]
	\centering
	\begin{subfigure}{.245\textwidth}
		\centering
		\includegraphics[width=\linewidth]{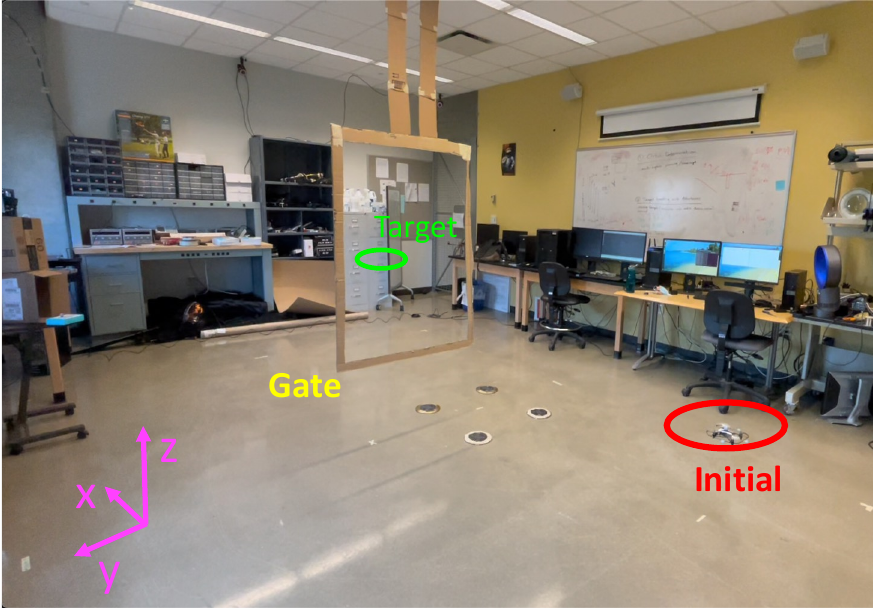}
		\caption{Main camera view}
		\label{drone.setup.1}
	\end{subfigure}
	\begin{subfigure}{.23\textwidth}
		\centering
		\includegraphics[width=\linewidth]{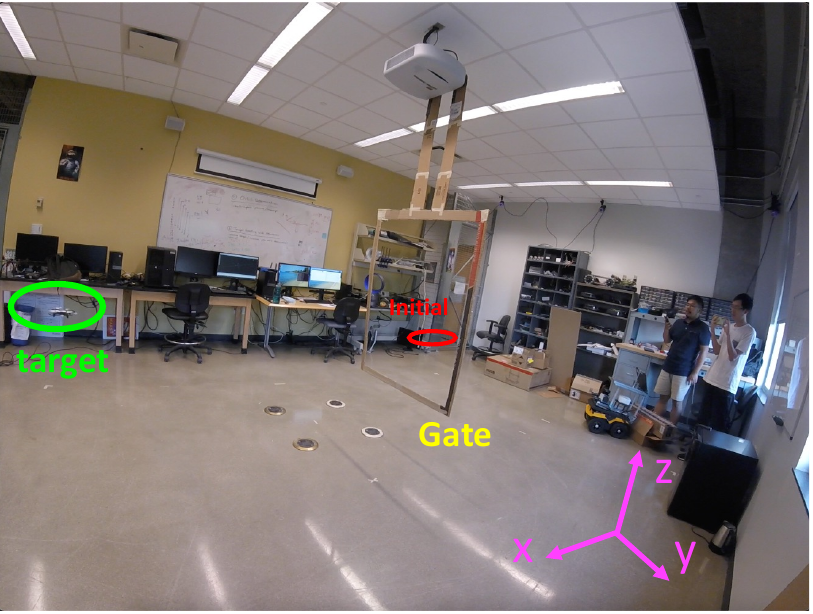}
		\caption{Back camera view}
		\label{drone.setup.2}
	\end{subfigure}
	\caption{
The setup of the real-world  experiment. A human user teaches the quadrotor to fly from an initial position (red circle), go through the gate (with yellow label), and finally, land on a target position (green circle). Without human  corrections, the quadrotor is initialized with a random cost function and fails to accomplish the task.   Note that we have no knowledge about  the gate.
	} 
	\label{realexperiment.setup}
\end{figure}

\begin{figure*}[h]
	\begin{subfigure}{.245\textwidth}
		\centering
		\includegraphics[width=\linewidth]{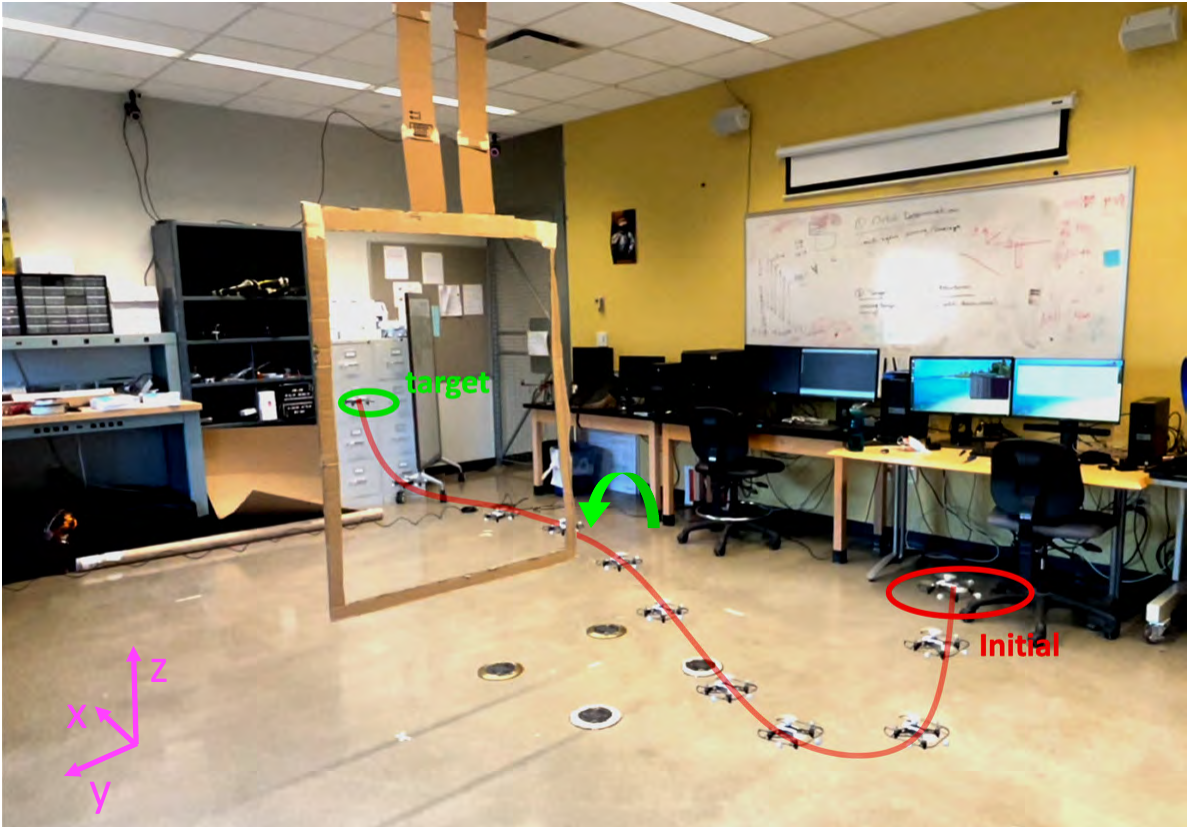}
		\caption{Iteration $k=1$}
		\label{drone.res.1}
	\end{subfigure}
	\begin{subfigure}{.245\textwidth}
		\centering
		\includegraphics[width=\linewidth]{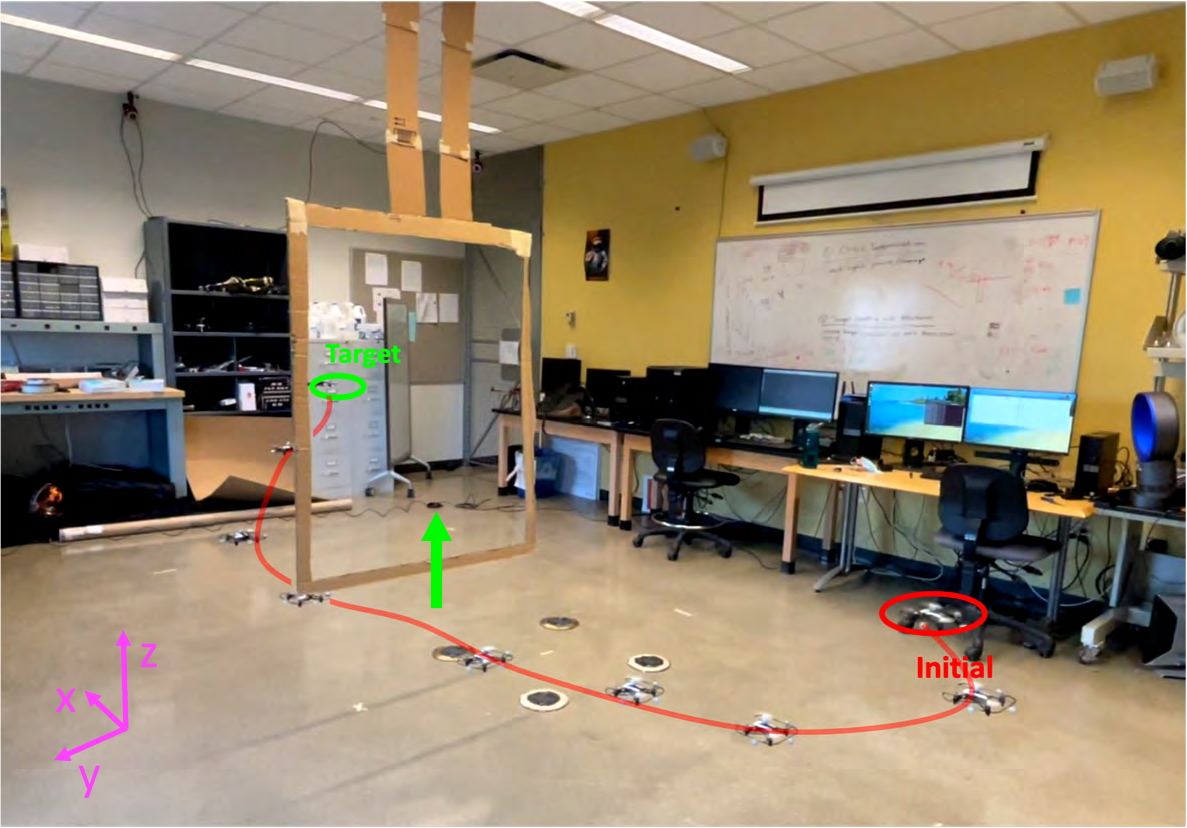}
		\caption{Iteration $k=2$}
		\label{drone.res.2}
	\end{subfigure}
	\begin{subfigure}{.245\textwidth}
		\centering
		\includegraphics[width=\linewidth]{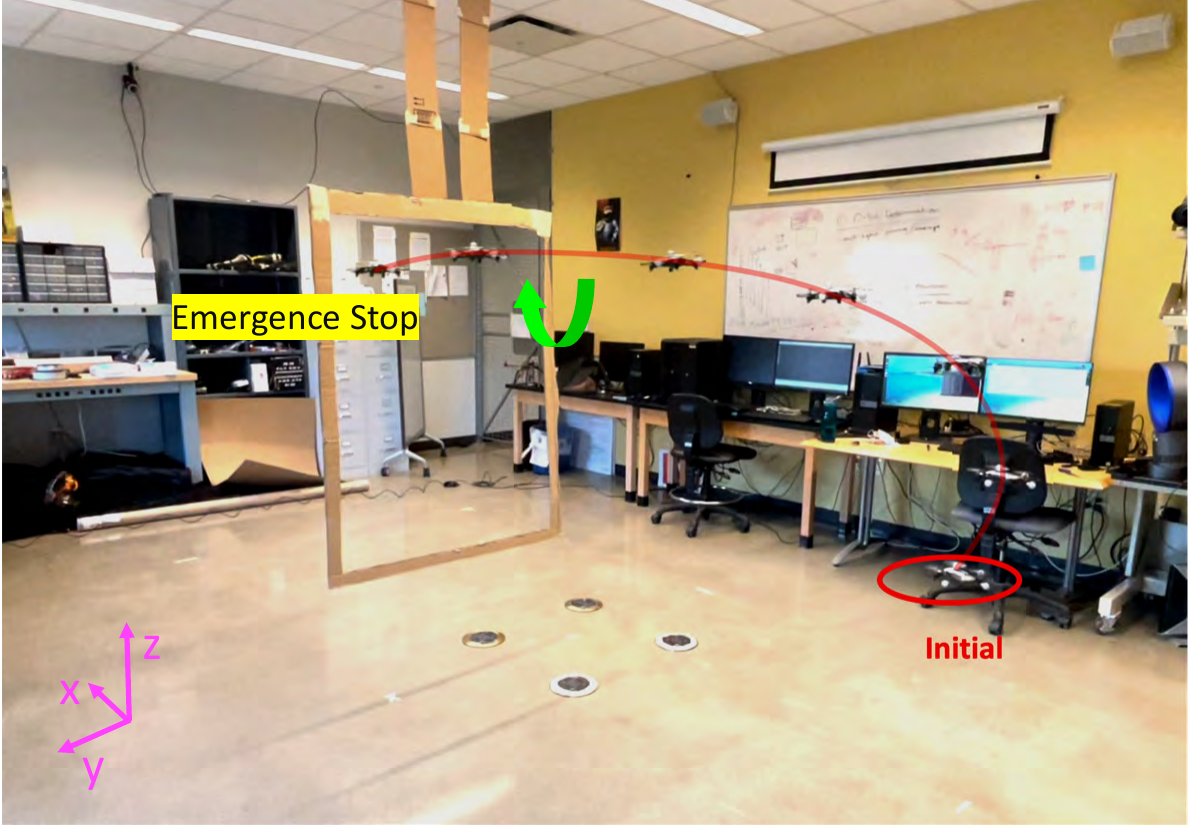}
		\caption{Iteration $k=3$}
		\label{drone.res.3}
	\end{subfigure}
	\begin{subfigure}{.245\textwidth}
		\centering
		\includegraphics[width=\linewidth]{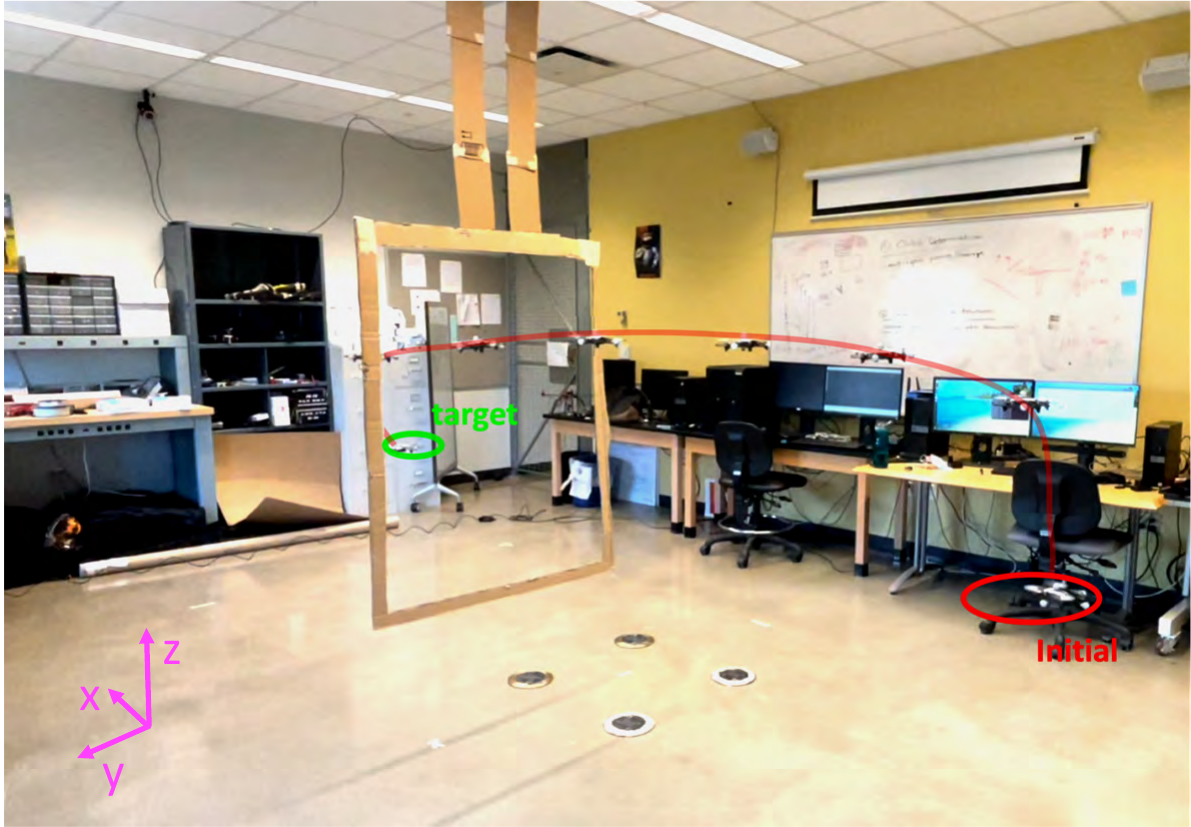}
		\caption{Iteration $k=4$}
		\label{drone.res.4}
	\end{subfigure}
	\caption{Results of a real-world quadrotor  learning  from  human  directional corrections. Here, the quadrotor  trajectory is marked in red lines. The initial position (red circle) and target position (green circle)  are also marked. In (a), during the quadrotor flying, the human user applied a torque correction in the opposite of the x-axis,  shown by the green arrow. In (b), during the quadrotor flying, the human user applied a upward thrust correction (in green arrow) in z-axis. In (c), the quadrotor detected a potential collision with the left pillar of the gate and thus triggered an emergence stop. But before the emergence stop, the human user had already applied a positive x-axis (in green arrow) torque  to correct the quadrotor. In (d), the quadrotor successfully learned a  cost function  to fly from the initial position, go through the gate, and finally reach the target position. 
	} 
	\label{drone.res}
\end{figure*}

\subsection{Experiment Setup and  Procedure}

Our real-world experiment is a Parrot Mambo quadrotor task shown in Fig. \ref{realexperiment.setup}. The goal of the experiment is to let a human user teach the quadrotor to accomplish the following task: starting from the initial position (red circle), flying through the gate (labeled in yellow), and landing on the target position (green circle). Without human correction, the quadrotor with a random initial cost function fails to accomplish the task.  The procedure of learning from  human  directional corrections is as follows: 
\begin{itemize}
	\item First, perform the quadrotor motion planning in a remote computer by solving a trajectory optimization  with the current  cost function;
	\item Second, the quadrotor executes the plan (tracking control), at the same time a  human user watches the quadrotor's motion and applies directional corrections via the keyboard of the remote computer. 
	\item Third,  update the quadrotor's cost function in the remote computer according to Algorithm \ref{algorithm1}; then repeat the above steps until the success of the task.
\end{itemize}

In the above procedure, the interface for a human user to apply directional corrections is still a  keyboard, following  Table \ref{table_keycustomize_uav}.  The communication frequency between the quadrotor and the remote computer is around 10Hz. Note that because the cost function in some learning iterations may lead the quadrotor to collide with the gate during execution,  we have included a collision  stop mechanism:  if the quadrotor detects that it is too close to the gate, an emergence stop is triggered. In the iteration where an emergence stop is triggered,   human corrections, if applied before an emergence stop,  can still  count  for updating the cost function for the next iteration.  We will discuss  the damage protection of robot  learning  in details  in Appendix \ref{section.discussion.damage}.

We have invited a  human user (not the co-authors themselves), who is a novice to our work, to perform the above experiment. The human user first had a successful warm-up training in our previous quadrotor game in order to get familiar with the interface (Table \ref{table_keycustomize_uav}), and then successfully conducted this real-world quadrotor experiment in just one run.

\subsection{Results and Analysis}

The results in all learning iterations are in Fig. \ref{drone.res}. At each iteration, the quadrotor's trajectory is highlighted in the red lines, and the human  directional corrections are labeled in the green arrows. The initial and target positions are also marked.  We provide the following analysis for  the results in Fig. \ref{drone.res}. 

\smallskip

In the first iteration in Fig. \ref{drone.res.1}, since the quadrotor was initialized with a random cost function, the planned and executed  trajectory (labeled in the red line) are off the gate---the quadrotor missed the gate and flew from its right side. Thus, when observing the quadrotor's motion, the human user applied a directional toque to correct the quadrotor, which is in the direction of the negative x-axis, as shown by the green arrow in Fig. \ref{drone.res.1}.  After receiving the correction, the quadrotor updated its cost function and planned a new trajectory for the next iteration in Fig. \ref{drone.res.2}. In the second iteration in  Fig. \ref{drone.res.2},  while the quadrotor was executing the new motion, the human user observed that its motion was too below the gate and thus applied a directional force in the positive z-axis. Using this directional correction, the quadrotor updated its cost function for the third iteration in Fig. \ref{drone.res.3}. 

In the third iteration in Fig. \ref{drone.res.3}, while executing the planned motion, the quadrotor detected a potential collision with the left pillar of the gate and thus immediately triggered an emergence stop. However, because  humans are able to predict the  potential collision, the human  user had already applied a directional torque  in the positive x-axis (the green arrow) \emph{before the emergence stop}. Thus, although the quadrotor had stopped in emergence, the  directional correction still counted in the update of the cost function. In the fourth iteration in Fig. \ref{drone.res.4},  the quadrotor successfully flew through the gate and reached the target position.

\smallskip
Thus, it took only three human directional corrections for the quadrotor to learn a cost function to accomplish the task in an unknown environment.
Based on the above results and analysis, we can conclude that the proposed method is effective and efficient for learning an objective function for desired robot motion from human directional corrections.

\section{Conclusion}
This paper has proposed a new method to enable a robot to learn an objective function from human directional corrections. A human directional correction can be any input change to the robot as long as it is in a direction that improves the robot's current motion relative to an implicit objective function.   The proposed learning method only uses the direction of a correction to update the estimate of the objective function. The learning process is based on the cutting plane method and  has straightforward geometric interpretations.  We have established the theoretical results to show the convergence of the estimate of the objective function  towards the true one induced by all human corrections. 
The proposed approach has been validated by numerical examples,  a user study on two human-robot games, and a real-world quadrotor experiment. The results confirm the convergence and efficacy of the proposed method and show that the proposed method is significantly more effective (higher success rate), efficient/effortless (less correction needed), and potentially more accessible (fewer early wasted trials) than the state-of-the-art human-robot learning methods.

\section*{Acknowledgments}
We thank Tianyu Zhou for his effort in helping organize and collect the data in the user study.

\appendices
\section{Proof of Lemma \ref{lemma3}}\label{appendix_2}

First, we prove (\ref{eq_prop1}). From Step 2 in the main algorithm, we know that the robot's current trajectory $\boldsymbol{\xi}_{\boldsymbol{{\theta}}_k}=\{\boldsymbol{x}^{\boldsymbol{{\theta}}_k}_{0:T\text{+}1},\boldsymbol{u}_{0:T}^{\boldsymbol{{\theta}}_k}\}$ is a result of minimizing the cost function $J(\boldsymbol{\theta}_{k})$ subject to dynamics constraint in   (\ref{equ_dynamics}). This means that $\boldsymbol{\xi}_{\boldsymbol{{\theta}}_k}$ must satisfy the optimality condition (i.e., first order condition). Following a similar derivation from  (\ref{equ_linear_dynamics_mat}) to  (\ref{equ_gradient}) in the proof of Lemma \ref{lemma_correctionEquivalent}, we can obtain the optimality condition
\begin{multline}\label{equ_appendix_optimality}
\boldsymbol{0}={\nabla J(\boldsymbol{u}_{0:T}^{\boldsymbol{{\theta}}_k}, \boldsymbol{\theta}_k)}
=\boldsymbol{H}_1(\boldsymbol{x}^{\boldsymbol{{\theta}}_k}_{0:T\text{+}1},\boldsymbol{u}_{0:T}^{\boldsymbol{{\theta}}_k})\boldsymbol{\theta}_k+\\\boldsymbol{H}_2(\boldsymbol{x}^{\boldsymbol{{\theta}}_k}_{0:T\text{+}1},\boldsymbol{u}_{0:T}^{\boldsymbol{{\theta}}_k})\nabla h(\boldsymbol{x}_{T+1}^{\boldsymbol{\theta}_k}).
\end{multline}
It is worth mentioning that the above optimality condition (\ref{equ_appendix_optimality}) is also derived in \cite{jin2018inverse}.
As a result,
\begin{align}
0&=\left\langle{\nabla J(\boldsymbol{u}_{0:T}^{\boldsymbol{{\theta}}_k}, \boldsymbol{\theta}_k)}, \boldsymbol{\bar{a}}_k\right\rangle\\
&=\left\langle
\boldsymbol{H}_1(\boldsymbol{x}^{\boldsymbol{{\theta}}_k}_{0:T\text{+}1},\boldsymbol{u}_{0:T}^{\boldsymbol{{\theta}}_k})\boldsymbol{\theta}_k, \boldsymbol{\bar{a}}_k
\right\rangle\nonumber\\
&\qquad\qquad+
\left\langle
\boldsymbol{H}_2(\boldsymbol{x}^{\boldsymbol{{\theta}}_k}_{0:T\text{+}1},\boldsymbol{u}_{0:T}^{\boldsymbol{{\theta}}_k})\nabla h(\boldsymbol{x}_{T+1}^{\boldsymbol{\theta}_k}), \boldsymbol{\bar{a}}_k
\right\rangle\\
&=\left\langle \boldsymbol{h}_k,\boldsymbol{\theta}_k \right\rangle+b_k,
\end{align}
where  the third equality is due to the definition of hyperplane in (\ref{equ_hk}). This completes the proof of (\ref{eq_prop1}).

\medskip

Next, we prove (\ref{eq_prop2})  by induction. In the main algorithm, we know $\boldsymbol{\theta}^*\in \boldsymbol{\Omega}_0$ for $k=0$. Assume that  $\boldsymbol{\theta}^*\in\boldsymbol{\Omega}_{k-1}$ holds at the $(k{-}1)$-th iteration. By Step 3 in the main algorithm, we have 
\begin{equation}
\boldsymbol{\Omega}_{k}=\boldsymbol{\Omega}_{k-1}\cap\left\{\boldsymbol{\theta}\in\mathbb{R}^r\,|\, \left\langle\boldsymbol{h}_k,\boldsymbol{\theta} \right\rangle+b_k<0\right\}.
\end{equation}
In order to prove $\boldsymbol{\theta}^*\in\boldsymbol{\Omega}_{k}$ we only need to show that 
\begin{equation}\label{equ_appendix_assumption}
\left\langle\boldsymbol{h}_k,\boldsymbol{\theta}^* \right\rangle+b_k<0,
\end{equation}
which is true according to (\ref{equ_assumption2}) in Lemma \ref{lemma_correctionEquivalent}.
Thus, $\boldsymbol{\theta}^*\in\boldsymbol{\Omega}_{k}$ also holds at the $k$th iteration. Thus, we conclude that (\ref{eq_prop2}) holds. This completes the proof of Lemma \ref{lemma3}. \qed

\section{Discussion}
\subsection{Learning Non-convex Cost Functions}
In our problem formulation and method development, we have not imposed any restriction on the convexity of the cost function  (\ref{equ_objective}). However, the proposed method can sufficiently handle both convex and non-convex cost functions. But for non-convex cost functions, we have the following comments.

For a general non-convex cost function (\ref{equ_objective}), the robot desired motion can be a \emph{local}  minimum of this non-convex cost function. As long as all human corrections aim to drive the robot towards the same local minimum (i.e., the same desired motion), all theories and methods developed in the paper still apply.  In the otherwise case,   human corrections may not have a consistent goal. For example, in one iteration the human intends to drive the robot towards one local minimum, while in the other iteration, the human aims to drive the robot to a \emph{different} local minimum---a different desired motion. In those inconsistent cases, the assumption  (\ref{equ_assumption}) would be violated, thus potentially leading to the failure of the proposed method.

In sum,  if the cost function (\ref{equ_objective}) is non-convex,  the desired robot motion $\boldsymbol{\xi}_{\boldsymbol{\theta}^*}$ can be a local minimum of the non-convex  cost function $J(\boldsymbol{\theta}^*)$. The proposed method in this paper  still applies if all human corrections aim to drive the  robot towards the same local minimum  $\boldsymbol{\xi}_{\boldsymbol{\theta}^*}$.

\subsection{On-the-fly  Implementation of the Proposed Method}

In the previous experiments, each learning iteration requires a robot to start over the task. However, this is not necessary, and one can readily implement the proposed method in an on-the-fly manner. An on-the-fly implementation is given in Algorithm \ref{algorithm2}, where the changes compared to  Algorithm~\ref{algorithm1} are highlighted in red.  Specifically, in Algorithm \ref{algorithm2}, after the robot receives a human directional correction $\boldsymbol{a}_{t_k}$
at time step $t_k$, instead of starting over the task (i.e.,
returning to the very early  initial condition $\boldsymbol{x}_0$), we can
update the robot cost function immediately and plan the 
motion by starting from the robot \emph{current}  state $\boldsymbol{x}_{t_k}$ (on
which the correction $\boldsymbol{a}_{t_k}$ is made). Thus, the robot will
continue to execute the task from the latest state without
starting over. For this on-the-fly  implementation,
all theories and other properties of our method  remain
unchanged. 

\begin{algorithm2e}[th]
	\small 
	\SetKwInput{Initialization}{Initialization}
	\KwIn{Specify a termination threshold $\epsilon$ and use it to compute the maximum iteration  $K$ by (\ref{equ_maxloop}). }
	\Initialization{Initial weight search space $\boldsymbol{\Omega}_0$ in (\ref{equ_linftyball}), \textcolor{red}{and initial robot  state $\boldsymbol{x}_{t_0}$ with $t_0=0$}}
	\smallskip
	\For{$k=1,2,\cdots, K$}{

		\smallskip
		Choose a  weight vector guess $\boldsymbol{\theta}_k\in \boldsymbol{\Omega}_{k-1}$ by Lemma 3\;
		
		\smallskip
		\textcolor{red}{Starting from the current robot state $\boldsymbol{x}_{t_{k-1}}$,}
		plan a new robot motion $\boldsymbol{\xi}_{\boldsymbol{\theta}_k}$ by solving a trajectory optimization  with the cost function  $J(\boldsymbol{\theta}_k)$  and  dynamics   (\ref{equ_dynamics})\;
		
		\smallskip
		\textcolor{red}{Starting from $\boldsymbol{x}_{t_{k-1}}$,} the robot  executes the planned motion trajectory $\boldsymbol{\xi}_{\boldsymbol{\theta}_k}$ while receving  human directional corrections $\boldsymbol{\bar{a}}_{t_k}$\;
		
		\smallskip
		Compute the   matrices $\boldsymbol{H}_1(\boldsymbol{x}^{\boldsymbol{\theta}_k}_{0:T\text{+}1},\boldsymbol{u}^{\boldsymbol{\theta}_k}_{0:T})$ and $\boldsymbol{H}_2(\boldsymbol{x}^{\boldsymbol{\theta}_k}_{0:T\text{+}1},\boldsymbol{u}^{\boldsymbol{\theta}_k}_{0:T})$, and then compute  the hyperplane and half space $\left\langle\boldsymbol{h}_k,\boldsymbol{\theta} \right\rangle+b_k<0 $ by (\ref{equ_assumption2})-(\ref{equ_hk})\;

		\smallskip
		Update the weight search space by $	\boldsymbol{\Omega}_{k}=\boldsymbol{\Omega}_{k-1}\cap\left\{\boldsymbol{\theta}\in\mathbb{R}^r \,|\, \left\langle\boldsymbol{h}_k,\boldsymbol{\theta} \right\rangle+b_k<0 \right\}$ by (\ref{equ_updateOmega})\;
	}
	\smallskip
	\textbf{Output:} $\boldsymbol{\theta}_K$.
	\caption{Learning from  directional corrections (an on-the-fly implementation version)} \label{algorithm2}
\end{algorithm2e}

\subsection{Damage Prevention During Learning}\label{section.discussion.damage}

Since an intermediate cost function during learning iterations can lead a robot to collision or damage, we briefly discuss how to avoid such scenarios. 

One effective way to prevent damage and collision during robot learning is to add emergence stop mechanisms in the robot's lower-level control, as we have adopted for our real-world quadrotor experiment in Section \ref{section.real_experiments}. Fortunately, as long as a human user applies directional corrections before the emergence stop, those corrections still count for the update of the cost function. Based on the previous user study and real-world experiment, we observe that humans usually have a good ability to predict potential collisions and are able to make preemptive corrections \emph{before the robot triggers emergence stop}. Thus, it is a usual case that a robot has already received human directional corrections before triggering an emergence stop, and hence, the update of the cost function continues. This has been shown in Fig. \ref{drone.res.3} in our real-world experiment in Section \ref{section.real_experiments}.

\subsection{Choice of  $\boldsymbol{\Omega}_0$ and $\epsilon$ in Algorithm \ref{algorithm1}}\label{section.discussion.parameters}

The initial weight search space $\boldsymbol{\Omega}_0$  in (\ref{equ_linftyball}) should include any possible true   $\boldsymbol{\theta}^*$, i.e., $\boldsymbol{\theta}^*\in\boldsymbol{\Omega}_0$. 
Although this is hard to verify as $\boldsymbol{\theta}^*$ is usually unknown in practice, a good practice is  to choose   as large  $\boldsymbol{\Omega}_0$  as possible (note that $\boldsymbol{\Omega}_0$ also needs to ensure the existence of  solution $
\boldsymbol{\xi}_{\boldsymbol{\theta}}$; this is the reason why in our experiments, the range of  weights of the second-order terms in a polynomial is always positive).  Alternatively, one can also use a trial-and-error procedure to set the initial $\boldsymbol{\Omega}_0$: first, try a small $\boldsymbol{\Omega}_0$; under this small $\boldsymbol{\Omega}_0$ if the final (converged) robot motion is not desired, increase the size of  $\boldsymbol{\Omega}_0$  and repeat this process until satisfied. Our previous experiment experience has shown that choosing a good $\boldsymbol{\Omega}_0$ is not difficult. This is because even a small size of $\boldsymbol{\Omega}_0$ has enough expressiveness power to represent a good variety of robot motions. This empiricism is consistent with the recent results   of learning implicit models \cite{amos2017optnet,jin2021safePDP, jin2021learning} in the machine learning community, where it shows that simple objective or energy functions can have enough representation power.

In Algorithm \ref{algorithm1},  $\epsilon$  determines the accuracy of weight vector convergence, as stated in Theorem \ref{theorem_1}. The choice of $\epsilon$  depends on specific accuracy requirements, and a large $\epsilon$ will terminate the algorithm early.  In fact,   it is always easy to set $\epsilon$ using a reasonably small value,  because our previous experiment, such as in Fig. \ref{fig_compare},     has shown that the convergence of the robot motion trajectory is much faster than the convergence of the objective function itself. This means that it is a usual case that one observes a good convergence of robot motion trajectory before the convergence of cost function reaches termination condition. This empiricism has also been reported in the literature on inverse optimal control such as \cite{jin2020learning}.

\subsection{Free from  Noisy Robot Execution}

Finally, we would point out that in implementation,  the noise in robot states and controls cannot directly enter into the proposed algorithm. Specifically, as stated in Algorithm \ref{algorithm1} and   our real-world experiment, the proposed algorithm decouples   motion planning (i.e., plan $\boldsymbol{\xi}_{\boldsymbol{\theta}_k}$ by solving a trajectory optimization) and the robot  execution of $\boldsymbol{\xi}_{\boldsymbol{\theta}_k}$. The noise of the robot states and inputs can only enter into the robot execution stage, not the motion planning stage. While a human  observes the robot's  noisy execution of $\boldsymbol{\xi}_{\boldsymbol{\theta}_k}$ and gives  directional corrections $\boldsymbol{\bar{a}}_k$, the proposed algorithm updates the cost function using the \emph{noise-free}   $\boldsymbol{\xi}_{\boldsymbol{\theta}_k}$ \emph{taken from the motion planner} instead of from the robot's  execution.  Thus, during the entire learning process,   the noise in the robot's actual execution will never enter into the algorithm, thus will not influence the learning results.

\section{Other Choices of $ \boldsymbol{\theta}_k$ }\label{appendix_centerchoice}

For the weight  search space $\boldsymbol{\Omega}_{k-1} \subset \mathbb{R}^r$, we  choose $ \boldsymbol{\theta}_{k}$ as the center of Maximum Volume Ellipsoid (MVE) inscribe $\boldsymbol{\Omega}_{k-1}$. Other choices for $ \boldsymbol{\theta}_k$ could be  the center of gravity \cite{newman1965location}, the Chebyshev center \cite{elzinga1975central}, the analytic center \cite{goffin1993computation}, etc.

\subsubsection{Center of Gravity} The  center of gravity for a polytope $\boldsymbol{\Omega}$  is defined as 
\begin{equation}\label{equ_centerofgravity}
\boldsymbol{\theta}_{\text{cg}}=\frac{\int_{\boldsymbol{\Omega}}^{}\boldsymbol{\theta}d\boldsymbol{\theta}}{\int_{\boldsymbol{\Omega}}^{}d\boldsymbol{\theta}}.
\end{equation}
Following \cite{grunbaum1960partitions}, the volume reduction rate using the center of gravity is 
\begin{equation}
\frac{\vol(\boldsymbol\Omega_{k+1}) }{\vol(\boldsymbol\Omega_{k})}\leq 1-\frac{1}{e}\approx 0.63,
\end{equation}
which may lead to faster convergence than the rate $(1-1/r)$ using the center of MVE. However, for a polytope described by a set of linear inequalities, it is more   expensive  to compute  the center of gravity in (\ref{equ_centerofgravity})  than to compute the center of MVE  \cite{boyd2007localization}.

\subsubsection{Chebyshev Center} The Chebyshev center is defined as  the center of the largest Euclidean ball that lies inside the polytope  $\boldsymbol{\Omega}$.  Chebyshev center for a polytope can be efficiently computed   by solving a linear program \cite{boyd2004convex}. But  Chebyshev center is not  affinely invariant to the transformations of coordinates \cite{boyd2007localization}. Thus, a linear mapping of features may lead to an inconsistent weight vector estimation.

\subsubsection{Analytic Center} Given a polytope $\boldsymbol{\Omega}=\{\boldsymbol{\theta} \,|\, \left\langle\boldsymbol{h}_i,\boldsymbol{\theta} \right\rangle+b_i<0, \,\, i=1,\cdots, m \}$, the analytic center is defined as 
\begin{equation}\label{equ_analytic_center}
\boldsymbol{\theta}_{\text{ac}}=\min_{  \boldsymbol{\theta}}-\sum\nolimits_{i=1}^{m}\log(b_i-\boldsymbol{h}_i\tran\boldsymbol{\theta}).
\end{equation}
As shown by  \cite{atkinson1995cutting,nesterov1995cutting},  the analytic center achieves a good trade-off in terms of simplicity and practical performance. However, it might not be friendly for analyzing the volume reduction compared to using the center of MVE.

\bibliographystyle{IEEEtran}
\bibliography{trobib}

\vspace{0pt}

\begin{IEEEbiography}[{\includegraphics[width=1.2in,height=1.2in,keepaspectratio]{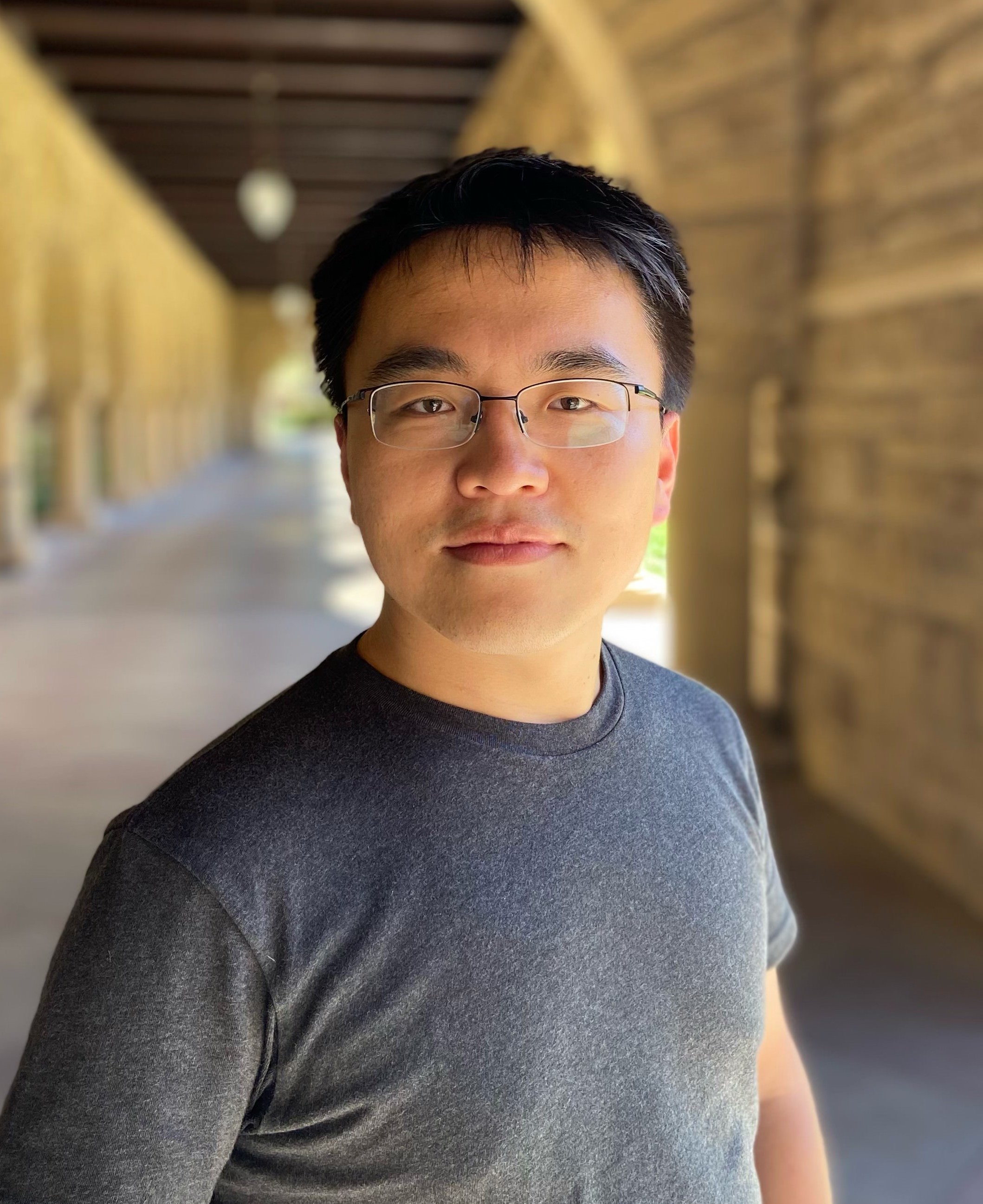}}]
{Wanxin Jin}  is a postdoctoral researcher in the GRASP Laboratory at the University of Pennsylvania. He received the Ph.D. degree in Autonomy and Control at  Purdue University in 2021. From 2016 to 2017, he was a Research Assistant at Technical University Munich, Germany.  Wanxin's research interests include robotics, control, machine learning, and optimization, with emphasis on learning,  planning, and control of robots as they interact with the world and humans.
\end{IEEEbiography}

\vspace{8pt}

\begin{IEEEbiography}[{\includegraphics[width=1.2in,height=1.25in,clip,keepaspectratio]{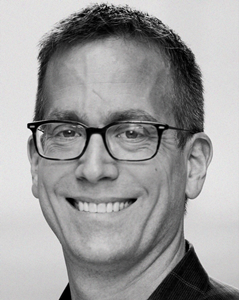}}]
	{Todd D. Murphey}
received his B.S. degree in mathematics from the University of Arizona and the Ph.D. degree in Control and Dynamical Systems from the California Institute of Technology. He is a Professor of Mechanical Engineering at Northwestern University. His laboratory is part of the Center for Robotics and Biosystems, and his research interests include robotics, control, machine learning in physical systems, and computational neuroscience.
\end{IEEEbiography}

\vspace{8pt}

\begin{IEEEbiography}[{\includegraphics[width=1.2in,height=1.25in,clip,keepaspectratio]{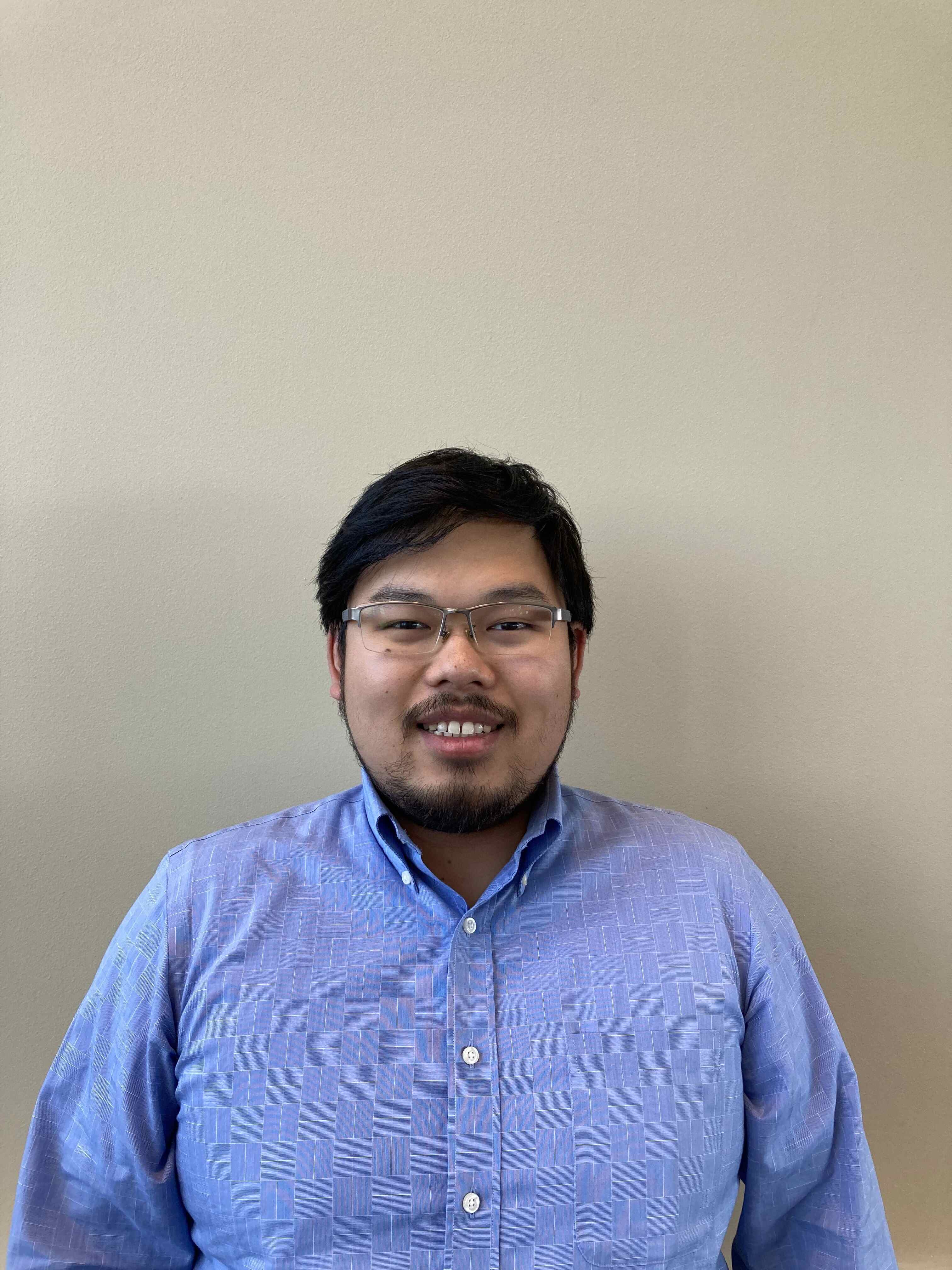}}]
	{Zehui Lu}
	received the M.S. degree in mechanical engineering from the University of Michigan, Ann Arbor, MI, USA, in 2019. Between January 2020 and July 2020, he was a research engineer at UM Ford Center for Autonomous Vehicles (FCAV) and ROAHM LAB, all with the University of Michigan. Currently, he is working toward the Ph.D. degree in aeronautics and astronautics engineering at Purdue University, West Lafayette, IN, USA.
	His current research interests include autonomy, control, and optimization, and their applications in robotics and multi-agent systems.
\end{IEEEbiography}

\vspace{8pt}

\begin{IEEEbiography}[{\includegraphics[width=1.0in,height=1.25in,clip,keepaspectratio]{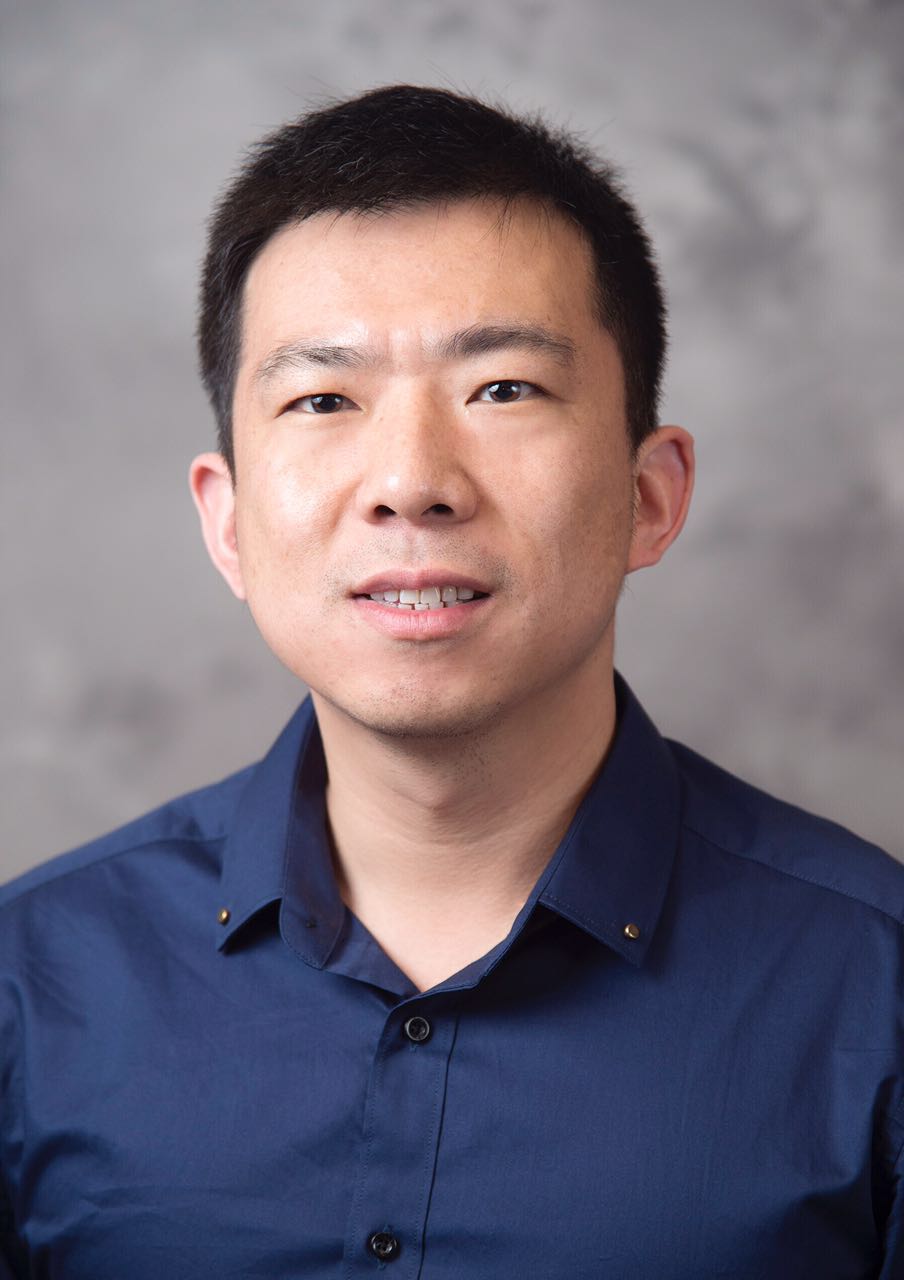}}]
	{Shaoshuai Mou}
is an Associate Professor in the School of Aeronautics and Astronautics at Purdue University. Before joining Purdue, he received a Ph.D. in Electrical Engineering at Yale University in 2014 and worked as a postdoc researcher at MIT for a year after that. His research interests include multi-agent system, control and learning, robotics control, human-robot teaming, resilient autonomy, and also experimental research involving autonomous air and ground vehicles. Dr. Mou co-directs Purdue University’s Center for Innovation in Control, Optimization and Networks (ICON), which aims to integrate classical theories in control/optimization/networks with recent advances
in machine learning/AI/data science to address fundamental challenges in autonomous and connected systems. 
\end{IEEEbiography}


\vfill

\end{document}